\documentclass{article}

\PassOptionsToPackage{numbers}{natbib}
\usepackage[final]{neurips_2024}

\usepackage[pdftex]{graphicx}
\usepackage[utf8]{inputenc} % allow utf-8 input
\usepackage[T1]{fontenc}    % use 8-bit T1 fonts
\usepackage{hyperref}       % hyperlinks
\usepackage{url}            % simple URL typesetting
\usepackage{booktabs}       % professional-quality tables
\usepackage{amsfonts}       % blackboard math symbols
\usepackage{nicefrac}       % compact symbols for 1/2, etc.
\usepackage{microtype}      % microtypography
\usepackage{xcolor}         % colors
\usepackage{multirow}
\usepackage{wrapfig}
\usepackage{makecell}
\usepackage{graphicx}
\usepackage{tikz}
\usepackage{multicol}
\usepackage{graphicx}
\usepackage{caption}
\usepackage{comment}
\usepackage{subcaption}
\usepackage{tikz}
\usepackage{xspace}
\usepackage{amsmath}
\usepackage{amsthm}
\usepackage{amsfonts}
\usepackage{cleveref}
\usepackage{pifont}
\usepackage{stfloats}

\theoremstyle{definition}
\newtheorem{definition}{Definition}

\newcommand{\jiasi}[1]{\textcolor{red}{JC: #1}}

\newcommand{\squishlist}{
   \begin{list}{$\bullet$}
    { \setlength{\itemsep}{0pt}      \setlength{\parsep}{3pt}
      \setlength{\topsep}{3pt}       \setlength{\partopsep}{0pt}
      \setlength{\leftmargin}{1.0em} \setlength{\labelwidth}{1em}
      \setlength{\labelsep}{0.5em} } }
\newcommand{\squishend}{
    \end{list}  }

\definecolor{darkred}{RGB}{150,0,0}
\definecolor{darkgreen}{RGB}{0,150,0}
\definecolor{darkblue}{RGB}{0,0,200}

\newcommand{\eg}{\textit{e}.\textit{g}.,~}
\newcommand{\ie}{\textit{i}.\textit{e}.,~}

\newtheorem{lemma}{Lemma}
\newtheorem{proposition}{Proposition}

\newcommand{\beq}{\begin{equation}}
\newcommand{\ba}{\begin{align}}
\newcommand{\ea}{\end{align}}

\newcommand{\eeq}{\end{equation}}

\newcommand{\cmark}{\ding{51}}%
\newcommand{\xmark}{\ding{55}}%

\newcommand{\order}[1]{{\cal{O}}(#1)}

\newcommand{\bgl}{{~\big |~}}

\definecolor{emmanuel}{RGB}{255,127,0}

\newcommand{\qb}{{\vct{q}}}
\newcommand{\Pro}{\mathbb{P}}

\newcommand{\vct}[1]{\boldsymbol{#1}}%{{\bf{\emph{#1}}}}

\newcommand{\LM}{M\xspace}
\newcommand{\LMPP}{(M+P)\xspace}
\newcommand{\alg}{{\sf \small TREACLE}\xspace}
\newcommand{\algU}{the Calibrated Cascade Algorithm\xspace}
\newcommand{\cost}{\texttt{cost}\xspace}
\newcommand{\reward}{\texttt{reward}\xspace}

\newcommand{\BMajority}{{Majority Voting}\xspace}
\newcommand{\OnlineKnapsack}{{Online Knapsack}\xspace}
\newcommand{\OfflineKnapsack}{{Offline Knapsack}\xspace}

\makeatletter
\newtheorem*{rep@theorem}{\rep@title}
\newcommand{\newreptheorem}[2]{%
\newenvironment{rep#1}[1]{%
 \def\rep@title{#2 \ref{##1}}%
 \begin{rep@theorem}}%
 {\end{rep@theorem}}}
\makeatother

\newreptheorem{theorem}{Theorem}
\newreptheorem{lemma}{Lemma}
\newreptheorem{proposition}{Proposition}
\newreptheorem{definition}{Definition}
\usepackage{amsmath,amsfonts,bm}

\def\eqref#1{equation~\ref{#1}}
\def\1{\bm{1}}

\def\eps{{\epsilon}}

\DeclareMathAlphabet{\mathsfit}{\encodingdefault}{\sfdefault}{m}{sl}
\SetMathAlphabet{\mathsfit}{bold}{\encodingdefault}{\sfdefault}{bx}{n}

\newcommand{\E}{\mathbb{E}}

\title{TREACLE: Thrifty Reasoning via Context-Aware LLM and Prompt Selection}

\title{Efficient Contextual LLM Cascades through Budget-Constrained Policy Learning}

\author{%
  Xuechen Zhang \\
  University of Michigan \\
  Ann Arbor, MI\\
  \texttt{zxuechen@umich.edu} \\
  \And
  Zijian Huang \\
  University of Michigan \\
  Ann Arbor, MI\\
  \texttt{zijianh@umich.edu} \\
  \And
  Ege Onur Taga \\
 University of Michigan \\
   Ann Arbor, MI\\
  \texttt{egetaga@umich.edu} \\
  \And
  Carlee Joe-Wong \\
  Carnegie Mellon University \\
  Pittsburgh, PA\\
  \texttt{cjoewong@andrew.cmu.edu} \\
    \And
  Samet Oymak \\
 University of Michigan \\
   Ann Arbor, MI\\
  \texttt{oymak@umich.edu} \\
    \And
  Jiasi Chen \\
 University of Michigan \\
   Ann Arbor, MI\\
  \texttt{jiasi@umich.edu} \\
}

\begin{document}

\maketitle

\begin{abstract}
Recent successes in natural language processing have led to the proliferation of large language models (LLMs) by multiple providers.
Each LLM offering has different inference accuracy, monetary cost, and latency, and their accuracy further depends on the exact wording of the question (\ie the specific prompt).
At the same time, users often have a limit on monetary budget and latency 
to answer all their questions, and they do not know which LLMs to choose for each question to meet their accuracy and long term budget requirements. 
        To navigate this rich design space, we propose \alg (\underline{T}hrifty \underline{Rea}soning via \underline{C}ontext-Aware \underline{L}LM and Prompt S\underline{e}lection),
a reinforcement learning policy that jointly selects the model and prompting scheme while respecting the user's monetary cost and latency constraints. TREACLE uses the problem context, including question text embeddings (reflecting the type or difficulty of a query) and the response history (reflecting the consistency of previous responses) to make smart decisions. 
Our evaluations on standard reasoning datasets (GSM8K, CSQA, and LLC) with various LLMs and prompts show that TREACLE enables cost savings of up to 85\% compared to baselines, while maintaining high accuracy. Importantly, it provides the user with the ability to gracefully trade off accuracy for cost. 
\end{abstract}

\section{Introduction}
\label{sec:intro}
\begin{wrapfigure}{r}{6cm}
 \vspace{-15pt}
\centering
\includegraphics[scale=0.23]{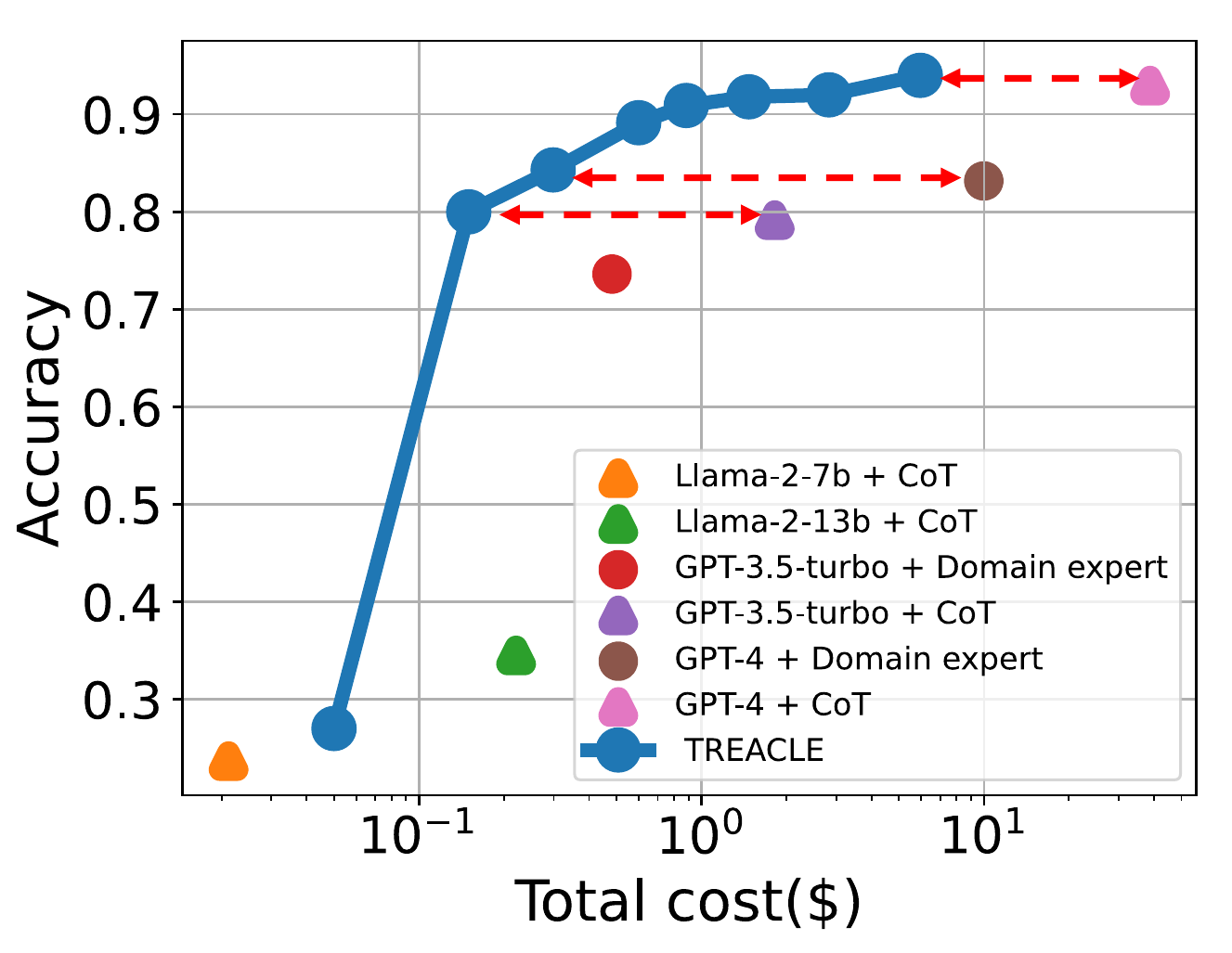}
% \includegraphics[scale=0.45]{img_new/fig1_result_ablation.pdf}
% \end{center}
\vspace{-2mm}
\caption{$\alg$ chooses LLMs to achieve high accuracy and $\sim$85\% cost reduction, compared to individual LLMs.}
%GPT prices are from OpenAI's published API prices; Llama is open-source and thus synthetic prices are used.
%\xuechen{only model, only prompt and only temperature,personally, I like the previous one}
%\jiasi{Change ``DQN'' to ``CURE'' in all figs. Total cost rather than budget would be more useful for the intro discussion. Units for x-axis.}\xuechen{fixed}
%\jiasi{Note: LLaMA is fake price, what to do?}
\label{fig:clean-full}
\vspace{-15pt}
\end{wrapfigure}
%\jiasi{terminology consistency: ``costs'' include ``monetary price'' and ``latency''. ``\underline{responses}'' vs ``answers''. ``prompting scheme''. ``LLaMA'' or ``Llama''?}\\
%\jiasi{remove small font size from captions so that ``Figure'' and caption text are the same size}
%\xuechen{\underline{question}, query, \underline{re-query}}\\
%\xuechen{query times, time is confusing}
%LLMs are popular but heterogeneous capabilities and costs, how to pick?
% \xuechen{remove ICML footer for arxiv submission}
\vspace{-5pt}
The success of large language models (LLMs) in recent years has led to a explosion of heterogeneous models and providers, including as Meta's Llama, OpenAI's ChatGPT, and Google's Gemini.
As LLMs continue to proliferate in the near future, we envisage a generative AI marketplace with a large variety of providers, LLMs, and deployments. Notably, LLMs have widely varying capabilities and costs: capabilities in terms of accuracy in responding to different types of queries, and cost in terms of monetary price and query latency.
As an illustration, the accuracy versus cost tradeoffs of various Llama and GPT LLMs are shown in \Cref{fig:clean-full} on grade school math word problems~\cite{cobbe2021training}. 
%\jiasi{dataset correct for Figure 1}?
As can be seen, GPT-3.5 tends to have lower accuracy than GPT-4 (79\% vs 92\% respectively), but costs about 20 times less.
This heterogeneous array of LLMs can bewilder users who must choose between them.
%, as shown in Table \ref{tab:GSM8K model performance}).

% \begin{figure}[t]
% %\vspace{-5mm}
% \begin{center}
% \includegraphics[width=0.4\textwidth]{img_new/pricing_fig1-1.pdf}
% % \includegraphics[scale=0.45]{img_new/fig1_result_ablation.pdf}
% \end{center}
% \vspace{-2mm}
% \caption{$\alg$ chooses between LLMs to achieve high accuracy and $\sim$85\% cost reduction compared to individual LLMs.
% %GPT prices are from OpenAI's published API prices; Llama is open-source and thus synthetic prices are used.
% }
% %\xuechen{only model, only prompt and only temperature,personally, I like the previous one}
% %\jiasi{Change ``DQN'' to ``CURE'' in all figs. Total cost rather than budget would be more useful for the intro discussion. Units for x-axis.}\xuechen{fixed}
% %\jiasi{Note: LLaMA is fake price, what to do?}
% \label{fig:clean-full}
% \vspace{-5pt}
% \end{figure}

Another challenge is that the \emph{specific prompt} included in the question plays a critical role in eliciting accurate responses. This is especially true for reasoning problems where
%the model benefits from generating a chain of reasoning. This has spurred research on \emph{prompting schemes} \jiasi{cite some prompting schemes} that 
prompting a model to explain its reasoning can produce more accurate, but often more costly, answers.
Chain-of-thought (CoT)~\cite{wei2022chain} is an example of such a prompting scheme, in which the question includes a few examples of worked out problems, which cost more (due to the additional words included in the question) but also produce more accurate responses.
For example, in \Cref{fig:clean-full}, GPT-4 with CoT (pink triangle) achieves a 92\% accuracy, compared to GPT-4 with a domain expert prompt (brown dot, reminding the LLM that it is a ``math solver'') that achieves 83\%.
However, using the CoT prompt costs 3.9$\times$ more due to the extra words included in the query.
A final challenge is that the optimal choice of LLM and prompt depends on the \emph{specific question} being asked; the accuracy of a particular LLM and prompt combination for a particular question is unknown in advance, requiring learning or prediction.
%\carlee{should highlight here that accuracy for each prompting strategy and LLM depends on the query: so we need to learn/predict the accuracy for each LLM/prompt combination for each query}
%\jiasi{example of a CoT query being more accurate than standard prompt, but costs more in terms of tokens/dollars, refer to figure 1}.
%\jiasi{get some numbers from Table 1}.
%$1.38\times10^{-3}$ vs $2.94\times10^{-2}$

%our approach
Thus, the heterogeneity of the LLM landscape and the tradeoffs between accuracy and cost make it challenging to determine the optimal strategy of: \textbf{\emph{Which LLM to select and how to prompt it, in order to answer all questions while respecting cost constraints?}} 
%\xuechen{hightlight per question?}%, given complex accuracy and cost tradeoffs.
To address this, we propose a \underline{T}hrifty \underline{Rea}soning via \underline{C}ontext-Aware \underline{L}LM and Prompt S\underline{e}lection ($\alg$) framework. $\alg$ is a learning-based approach that solves reasoning questions by automatically selecting which LLM model %$(\LM_i)_{i=1}^K$ 
and prompt to use for each question.
%\xuechen{unclear same model promt for each problem or same models for all the problems}
Given a cost budget, including a total monetary price across all questions and an average per-query latency, its goal is to maximize the average accuracy of the responses.
%\carlee{why do we have a latency budget for all queries? It might sound better to say this is an average per-query latency constraint; intuitively, one would care more about each query's latency rather than when all queries are completed} \xuechen{using per-query budget might naturally solve the issue that why we know how many questions in total}
As shown in \Cref{fig:clean-full}, $\alg$ achieves the Pareto front of individual LLMs by combining them intelligently. % to meet a user-defined budget.

%It leverages context about the current query in order to make intelligent decisions at the present time, while considering the future remaining budget.
%Further, it has the ability to re-query the current LLM and utilize the consistency of the responses as a ``verifier'' of its own correctness.
%Uncertainty aspect of \alg~refers the model's ability to output calibrated likelihood of accurate answer.
%Thus the decision space is not only what LLM to use, but also the appropriate paired prompting scheme in an integrated approach.
%and utilizing the previous completions, subject to cost constraints.
%An additional wrinkle is that for \emph{reasoning problems}, which are the focus of this work, the solution to a problem involves a chain of reasoning to support it.

%This goal jointly tackles model selection and LLM prompting. So far, these topics have been considered isolated and, unfortunately, existing solutions are relatively ad hoc. Here, we aim to develop novel and integrative approaches. 

%prior approaches
%\xuechen{move to related work?} 
Several recent works utilize multiple LLMs during inference with a cascade design, where queries propagate through a cascade of LLMs, considering the LLMs' accuracy-cost tradeoffs. 
%\carlee{To save space, this paragraph could be integrated into the contributions list, by contrasting our contributions to prior work}
Most aim to maximize accuracy and lack an explicit way to control long-term costs, as \alg has. By posing the problem of LLM and prompt selection as a budget-constrained policy optimization, \alg provides a unified approach to efficient LLM cascades (see Table \ref{tab:compare}). 
\alg makes informed decisions based on the full context of the LLM cascade, including the query embedding, answer statistics, and remaining budget.
Overall, this paper makes the following contributions:

%In other words, they measure cost as a byproduct of LLM mixing rather than as a core constraint.
%Compared to FrugalGPT's~\cite{chen2023frugalgpt} threshold-based cascade strategy, we provide a learning-based method that incorporates question context and re-querying capabilities. %\xuechen{ref to the lemma somewhere?}

%Our evaluations on standard reasoning benchmarks show that our method outperforms FrugalGPT as well as the majority voting approach of \cite{yue2023large}. \xuechen{cite \cref{tab:compare} here?}%, and it is often competitive with the oracle baseline of Offline Knapsack.
%\jiasi{to do: add 1-2 sentences broadly comparing to other approaches}
%Yue: minimize cost without degrading task performance. Indirectly account for cost constraint by tuning a threshold
%Automix: has POMDP but only for single query, horizon=1. State def is just simple, complex, or unsolvable query. Also set threshold to control cost.

% \xuechen{Optimal planning. Both latency and API cost (the time taken for executing the action or the amount of money that needs to be spent to execute an action)}
% \xuechen{A single model that can adapt to changing cost(API price, latency) and budget}

%\zijian{contribution order: data collection, DQN query, Cascade calibration, experiment}

\squishlist
    \item \textbf{Characterization of the accuracy, monetary cost, and latency of LLMs.} To understand the trade-offs between the LLMs, we quantify the accuracy and cost of 5 different LLMs (Llama and GPT variants) with 3 different prompt strategies (standard, domain expert, and CoT) on 3 datasets (GSM8K, CSQA, and LLC).
    %\jiasi{should we offer to release this now, or save it for a separate dataset paper?} \xuechen{collect text?}%, constituting 12 LLM-prompting strategy pairs. 
    %\jiasi{Jiasi stopped here}
   % As shown in \cref{fig:model_tradeoffs}, the reasoning accuracy improves with larger models and more sophisticated prompts, resulting in higher monetary cost.

    \item \textbf{An adaptive LLM and prompt selection policy based on reinforcement learning.} $\alg$ dynamically chooses the right LLM and prompt for each question.
    %based on the question context and history of responses so far. %\xuechen{based on the question embedding and history of responses so far?} 
   % It incorporates long-term constraints on total latency and price and can be easily adapted to changing API costs and budgets in practical real-world scenarios.
    It does this by leveraging context about the current question, re-querying the models if needed to verify the consistency of the responses, and thinking ahead about the remaining budget. We also provide some theoretical justification for $\alg$'s key design choices.
    %which demonstrates its efficiency and near optimality.  
    %when making a decision on the current question. %\carlee{maybe combine this with the list of contributions?}
    %\vspace{-0.2cm}\item \textbf{Uncertainty quantification.} We design an uncertainty module that outputs calibrated likelihood of correctly answering the query. Through this, we motivate uncertainty quantification as a secondary accuracy objective and demonstrate that, with more budget, calibration error improves in tandem with reasoning accuracy. \jiasi{to revisit later to blend better with the story} \xuechen{remove?} %Importantly, this benefit comes for 
    %Along with self-consistency among queries, we employ wide range of features including query trajectories and question and answer features, such as embeddings, to quantify uncertainty. 
    %\textbf{Cascade calibration method that outputs query-dependent uncertainty given whole cascade of (query, output) pairs}Introducing and improving uncertainty quantification as an objective besides accuracy and cost-efficiency.
  
    %We examine whether re-querying the same (LLM, prompting scheme) and the consistency of the resulting responses is a good signal of confidence. 
    %\item \textbf{leverages consistency across different model configurations (LLM, input prompt, temperature) for robust uncertainty quantification and a family of consistency cascade algorithms as baselines}
    %Extensive experiments with several LLMs, prompting schemes, and datasets
    \item \textbf{Extensive evaluations.} 
    We show that \alg 
    %can enable approximately 85\% cost savings compared to using a single LLM and 
    %saves approximately \jiasi{X\%} \xuechen{85\%} cost compared to baselines from the literature \jiasi{are we 85\% compared to FrugalGPT? or just from individual models?} \xuechen{from individual models}
    substantially saves on cost while maintaining high accuracy on mathematical and commonsense reasoning tasks. 
    We demonstrate its robustness to different budgets, question difficulty, price changes, new LLMs, and new unseen task types. 
    %While developing CURE, we also introduce 
%    \vspace{-0.2cm}\item \xuechen{Highlight total budget?}
\squishend
%\vspace{-0.4cm}

The paper is organized as follows. 
We describe related work (\S\ref{sec:related}), the problem statement (\S\ref{sec:problem}), and our framework (\S\ref{sec:alg}).
We then describe our experiments (\S\ref{sec:experiments}) and conclusions (\S\ref{sec:conclusions}).
\vspace{-5pt}

%\carlee{include an outline of the paper here}
% \begin{table}[]
% \scriptsize
% \caption{\small{Compare to related works}}
% \begin{tabular}{|l|lll|lll|}
% \hline

% \end{tabular}
% \vspace{-5pt}
% \label{tab:compare}
% \end{table}

%Extensive experiments with several LLMs, prompting schemes, and datasets
% $\bullet$ \textbf{Cascade calibration method that outputs query-dependent uncertainty given whole cascade of (query, output) pairs}

% $\bullet$ \textbf{An efficient LLM querying policy based on cascade calibration method}

% Unique ability to jointly pick (LLM, prompt, temperature)

% $\bullet$ \textbf{Introducing and improving uncertainty quantification as an objective besides accuracy and cost-efficiency}

% $\bullet$ \textbf{leverages consistency across different model configurations (LLM, input prompt, temperature) for robust uncertainty quantification and a family of consistency cascade algorithms as baselines} a
\section{Related Work}
\label{sec:related}

FrugalGPT~\cite{chen2023frugalgpt} is perhaps the closest to this work, as they considered
 a similar cost-constrained LLM selection problem with a threshold-based policy to select from a sorted list of LLMs. % proposed a simple strategy to cascade a set of LLMs. 
Our approach differs in several key aspects: we utilize a reinforcement learning policy that chooses both LLMs and prompts, rather than a threshold-based scheme; we utilize the full context of the current question to make decisions, including the text embedding of the current question and the history of past responses; and our method 
%The intuition is that the text embedding will represent the type or difficulty of the query, and the history of past responses gives an idea of how 
can \emph{re-query} the same LLM and aggregate previous responses to estimate the correctness of the current response.
 %and determine whether to return it as the final response.%\cite{chen2023frugalgpt} propose an LLM cascade, in which LLMs were queried  in the increasing order of model performance sequentially until the answer is accepted by a fine-tuned evaluator model. However, this answer evaluation scheme is proved unproductive with reasoning tasks due to the nuanced nature of errors \cite{huang2023large}. 
%\carlee{can take out details here} 
\begin{table}[h]
\centering
\caption{\small{Comparison to related works.}}\label{tab:compare}
    \resizebox{1.01\textwidth}{!}{%
\begin{tabular}{|c|c|c|c|c|c|c|c|}
\hline
& Query embedding& Response consistency& Prompt \emph{and} LLM selection &  Long-term budget &  Robust to new models\\ \hline
\textbf{FrugalGPT} \cite{chen2023frugalgpt} & \cmark & \xmark &  \xmark &    \xmark &    \xmark\\
\textbf{AutoMix} \cite{madaan2023automix}&  \cmark & \cmark & \xmark &   \xmark &    \xmark\\
\textbf{MoT} \cite{yue2023large}&  \xmark &  \cmark & \xmark &   \xmark &    \xmark\\
\textbf{TREACLE} & \cmark & \cmark &  \cmark &   \cmark& \cmark\\
\hline
\end{tabular}%
}
\vspace{-15pt}
\end{table}
Mixture of Thought~\cite{yue2023large} explored the idea of response consistency
%as a measure of its correctness, 
in order to choose the right LLMs.
The intuition is that
%that the consistency of answers among different prompting strategies such as chain-of-thought (CoT) \cite{wei2022chain} and program-of-thought (PoT) \cite{chen2023program} 
higher consistency in the re-queries
implies higher confidence in the correctness of the response.
\alg employs response consistency as an input feature, along with other features, for LLM selection.
%but also incorporate a diverse range of features for model selection, enhancing the model's flexibility and adaptability significantly. 
  AutoMix~\cite{madaan2023automix} introduces a ``meta-verifier'' to estimate whether a response is correct or a more powerful LLM is needed. 
  Both works measure cost as a by-product of combining multiple LLMs rather than long-term constraint across all questions, as we do.
  %The existence of wide range of techniques with diverse features highlights the importance of having a flexible and adaptable strategy in challenging domains like reasoning tasks, which our method \alg addresses by introducing an LLM cascade policy leveraging reinforcement learning.  
%\carlee{This last sentence isn't needed} \zijian{addressed}
% \textbf{Prompting Strategy} There are many prompting strategies to improve the performance of LLMs. These include fundamental paradigms such as few-shot querying \cite{fewshot2020language}, and more complex prompting strategies such as CoT \cite{wei2022chain}, tree-of-thought (ToT) \cite{yao2023tree}, PoT \cite{chen2023program}, and various others \cite{liu-etal-2022-generated, yasunaga2023large, naik2023diversity, cheng2023batch}. One observation common among all these prompting strategies is that they lead to cost-performance trade-offs. As challenging reasoning tasks necessitate intelligent prompting schemes for optimized performance, our method involves choosing the prompting strategy contextually.  
Other lines of work include uncertainty estimation or prompt engineering to improve accuracy
%With the finding that it is possible to evaluate the uncertainty of LLMs' answers, 
\cite{lin2022teaching, xiong2023llms, yue2023large, si2023prompting,  cai2023humanintheloop, naik2023diversity}, which is complementary to our work.
%uses it to improve the accuracy-cost trade-off.}
The related work is summarized in \Cref{tab:compare}.

%\jiasi{to discuss: do we need this last paragraph now that we don't have \algU? need description of DIV-SE to match with table}
%\textbf{Uncertainty Estimation} The relevant works include \cite{lin2022teaching, xiong2023llms, yue2023large, si2023prompting,  cai2023humanintheloop}. \cite{lin2022teaching} utilizes verbalized expressions, called verbalized  probabilities, for uncertainty assessment. On the other hand, \cite{cai2023humanintheloop} use "diversity entropy", providing an aligned but separate metric to answer consistency.

%\jiasi{yue already discussed above, remove here}
%Similar in spirit to our work, \cite{yue2023large} uses answer consistency to estimate uncertainty in a cost-efficient querying setting. Distinctively, we leverage comprehensive set of features, including the query trajectory and stationary statistics in a reinforcement learning setting to quantify uncertainty, thereby having an answer along with the model's respective confidence in the end. \xzC{remove?} \zijian{@zijian edit}

\vspace{-0.3cm}

\section{Problem Statement}
\label{sec:problem}

%\zijian{Keep this paragraph}
%could delete this para if needed for space
\vspace{-5pt}
We study the natural language query problem of providing correct responses to a series of questions.
We focus on reasoning problems (\eg grade school math problems) because they are challenging with multiple logical steps required to reach a final correct response.
%\carlee{it is also important that they have a single correct answer, right?} \xuechen{Yes. all the datasets have a single correct answer}
%\jiasi{why do we focus on reasoning problems? because they can benefit from different prompting schemes?}\xuechen{unique challenge? It is hard to estimate the performance because the logic might looks reasonable} \ege{+ non-trivial problem in the context of efficient querying of LLMs. Self-verification etc is not straightforward due to the nuanced errors made in solutions}
The problem involves 
%iterative querying of a sequence of language models (i.e., a cascade) 
%selecting language models
%to 
answering a sequence of $n$ questions $\mathcal{Q}$ with correct responses $\mathcal{Y}$; in other words, we have a set of questions and responses $\{(Q_1,Y_1), (Q_2,Y_2), \ldots, (Q_n,Y_n)\}$.
We have a set $\mathcal{\LM}$ of language models (LLMs) at our disposal,
%\zijian{model zoo?},
which can be accessed either locally or remotely through APIs: $\mathcal{\LM} = \{\LM_1, \LM_2, \ldots, \LM_m\}$. Also, we have a choice between $p$ prompt types, $\mathcal{P} = \{P_1, P_2, \ldots, P_p\}$.
These models and prompts have different costs (in terms of latency and monetary price) and accuracy.
Each ``question'' can be asked multiple times to the same or different LLMs, which we call a ``re-query'', in order to possibly obtain a more final accurate response.

The goal is to ensure that as many responses as possible are correct, while simultaneously minimizing the associated costs. This problem can be formulated as a Constrained Markov Decision Process (CMDP), 
%\jiasi{do we use this acronym elsewhere, otherwise delete}
which is represented by a tuple $(\mathcal{Q},\mathcal{S},\mathcal{A},T,r,c,\gamma,B)$, where $\mathcal{Q}$ is the ordered question set; $\mathcal{S}$ is the state space; $T:(\mathcal{Q},\mathcal{S})\times \mathcal{A}\times (\mathcal{Q},\mathcal{S})\rightarrow[0,1]$ is
the transition probability function, \ie $(Q,s)_{t+1}\sim T(\cdot |(Q,s)_t,a_t)$; $r:(\mathcal{Q},\mathcal{S})\times\mathcal{A}\rightarrow [R_{\text{min}},R_{\text{max}}]$ and $c:(\mathcal{Q},\mathcal{S})\rightarrow \mathbb{R}^+$ denote the reward and cost function; $\gamma\in[0,1]$ is the discount factor for future reward and cost; and $B$ is the total budget. A policy $\pi:(\mathcal{Q},\mathcal{S})\rightarrow P(\mathcal{A})$ maps the question-state pairs to a probability distribution over actions. 
%Note that we will refer to question-state pairs, instead of state alone, to make our later theoretical analysis clearer.
A trajectory $\tau$ is composed of a sequence of (question, state)-action pairs: $\tau=\left\{\tau_{(Q,s)_0},\tau_{a_0},\tau_{(Q,s)_1},\tau_{a_1},...,\tau_{(Q,s)_L},\tau_{a_L}\right\}$, where $L$ is the total number of times the LLMs are queried. Note that $L\geq n$, due to the possible re-queries. The cumulative reward and cumulative cost of trajectory $\tau$ are denoted as $R(\tau)=\sum_{t=0}^L\gamma^t r(\tau_{(Q,s)_t},\tau_{a_t})$ and $C(\tau)=\sum_{t=0}^L\gamma^tc(\tau_{(Q,s)_t},\tau_{a_t})$, respectively.
%We assume no access to online interactions with the environment and offline dataset $\mathcal{D}=\left\{((q,s),a,(q',s'),r,c)\right\}$ is the only data available for training. 
The goal of our problem is to learn a policy $\pi$ from $\mathcal{D}$ that maximizes the expected cumulative reward, while satisfying the cumulative cost at the trajectory level:
%\jiasi{the above paragraph has notation $q$ rather than $Q_i$}
% \vspace{-3pt}
\begin{align}
    \max_\pi \mathbb{E}_{\tau\sim\pi,T}[R(\tau)], \;
    \text{s.t.} \; \forall \tau\sim\pi,T\quad C(\tau)\leq B. 
    \label{eqn:rl_problem}
\end{align}
% \vspace{-3pt}
where $\tau\sim\pi,T$ denotes that $\tau$ is generated by executing $\pi$ in $T$. 
%Because the next question will not be queried before the query of the current question finishes, the cumulative reward and the cumulative cost can be split as \jiasi{discuss happened to $\gamma$?} 
By grouping the rewards and costs by question instead of enumerating all re-queries, the cumulative reward and cost can be re-written as
%\label{eq:prob-stat-1}
$R(\tau) = \sum_{i=1}^{n}\reward(Y_i,\hat{Y}_i)$ and 
$C(\tau) = \sum_{i=1}^{n}\cost(Q_i)$,
%\label{eq:prob-stat-2}.
where $\hat{Y}_i$ is the final response for question $Q_i$, $\reward(\cdot)$ is the function that measures the correctness of the final response, $\cost(\cdot)$ is the cost function of giving a final response $\hat{Y}_i$ for question $Q_i$.
%\jiasi{Zijian: please make it consistent throughout whether we use $q, Q_i, \qb$ to represent a question}
% \begin{wrapfigure}{r}{0.65\textwidth}
%  \vspace{-3pt}
% \centering
% % \begin{figure}[]
% \centering
% %\hspace{-35pt}
% %\begin{subfigure}[b]{0.45\textwidth}
%     %\centering
% 	%\begin{tikzpicture}
% 	%	\node at (0,0) [scale=0.25]{
%   \includegraphics[width=0.63\textwidth]{img_new/overview_3-crop.pdf}
% %\end{tikzpicture}
% \caption{\small{Overview of \alg framework. \alg decides on the next (LLM, prompt) to query in a context-aware fashion, summarized in the state variable. It can adapt to unseen tasks by projecting the new queries into the text embedding space.  %\carlee{maybe explicitly show a sequence here (have multiple boxes, one for each query?)} \xuechen{return final response and go to next question. should we add a arrow saying go to the next question?}
% }}
% \vspace{-10pt}
% \label{fig:overview}
% \end{wrapfigure}
%\subsection{Theoretical Justification \zijian{maybe better section names?}}
%\jiasi{Should we put this later in Section 4, after we explain the re-querying?}
\begin{figure}
\centering
% \begin{figure}[]
\centering
%\hspace{-35pt}
%\begin{subfigure}[b]{0.45\textwidth}
    %\centering
	%\begin{tikzpicture}
	%	\node at (0,0) [scale=0.25]{
  \includegraphics[width=0.8\textwidth]{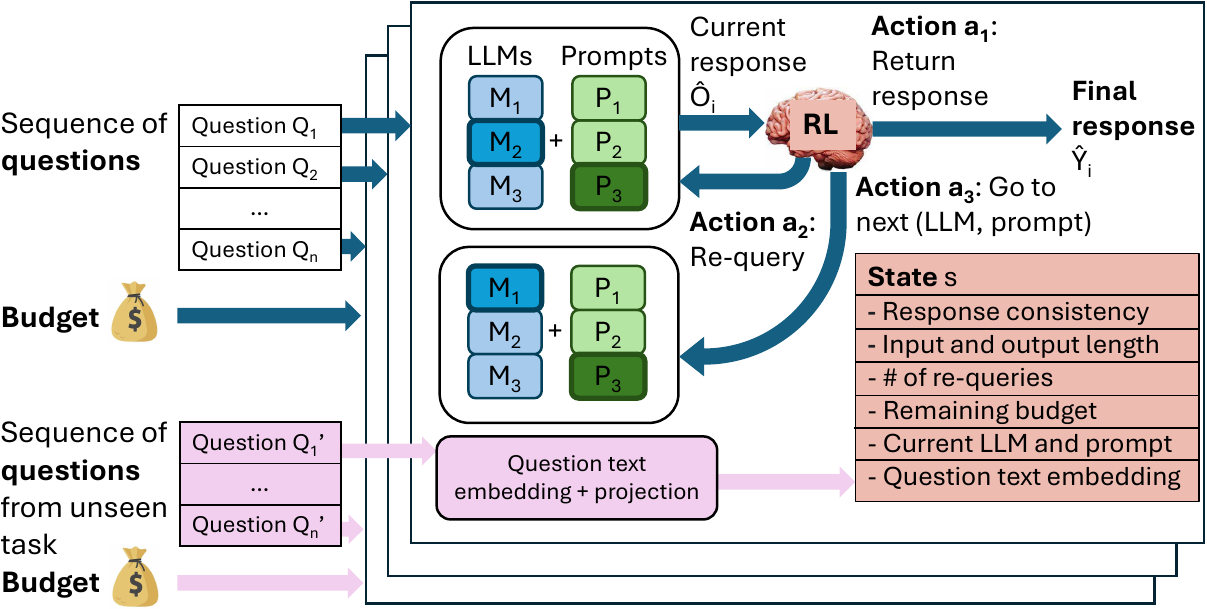}
%\end{tikzpicture}
\caption{Overview of \alg framework. \alg decides on the next (LLM, prompt) to query in a context-aware fashion, summarized in the state variable. It can adapt to unseen tasks by projecting the new queries into the text embedding space. %\carlee{maybe explicitly show a sequence here (have multiple boxes, one for each query?)} \xuechen{return final response and go to next question. should we add a arrow saying go to the next question?}
}\vspace{-10pt}
\label{fig:overview}
\end{figure}
%\zijian{Move to end of this section or later section?}
\textbf{Cost functions.}
%Since there is no standard cost function for local LLMs, and users may have varying preferences regarding latency and monetary price, we explore two distinct cost functions: pure monetary price, and monetary price and latency combinations. 
%\xuechen{use $\cost(\cdot)$. we define the function but didn't use it later}
We consider two types of costs in this work, monetary price and latency, resulting in two types of cost functions.
%Users may have varying preferences regarding these two factors, so we explore two distinct forms of $\cost(\cdot)$.
%pure monetary price, and a monetary price plus latency combination. %\carlee{can't these be combined into one (weighted sum of monetary cost and latency?} 
%Since there is no all-encompassing pricing structure \carlee{I'm not sure what you mean by this--why can't you attach a cost to each query? maybe just say the cost can be hard to estimate for local models} for both local models and API calls \carlee{explain here that some of the LLMs are run locally}, we experiment with the following 2 different pricing strategies to comprehensively study the effects of pricing: monetary price-latency combination \carlee{I think including latency as a metric is separate from not having a well-defined pricing structure--to me, including latency seems more about adding another metric entirely to user budget. Unless you mean the latency is proportional to the cost of local model serving?} and pure monetary price. The each price for each model under different pricing strategies 
    % \item \textbf{Latency-based Price} In the latency-based pricing strategy, we calculate the price of local deployed model by multiplying the price of GPU server with the latency time and keep using the actual price for models called by their API. \xuechen{Llama13b with CoT is more expensive than GPT3.5 with "Math solver" prompt}
     \textbf{\emph{(1) Pure monetary price.}} LLMs can run remotely, where the monetary price per token is set by the provider (\eg OpenAI). LLMs can also run locally, where the monetary price depends on a number of factors such as capital expenditures, server cooling and maintenance, electricity, etc.
    %Although the above pricing strategy include the common pricing scheme in real applications, we also test our policy with monetary price alone,     by setting synthetic monetary prices in order to show that our method generalizes to a wide-range-of prices.
    In our setup, the GPT models run remotely and the Llama models, which are free and open-source, run locally.
    %\carlee{Do we need the last sentence here? It sounds like an experimental setup choice} 
    %We set synthetic prices for Llama-2 because the models are free and open-source.
   \textbf{\emph{(2) Monetary price-latency combination.}} Monetary price is important for some users (\eg small companies) while latency plays a more crucial role in  other settings (\eg real-time voice assistants). Users who are latency-sensitive may be willing to pay more for lower latency, whereas others might be more patient and prefer lower prices. 
    \alg allows users to choose the trade-off between monetary cost and latency by adjusting a trade-off coefficient $\beta$, where $\cost=\text{latency} + \beta * \text{monetary price}$. 
    \vspace{-5pt}
    %\jiasi{double check this way, or the other way around?}
%\end{itemize}

\section{Proposed Framework: TREACLE}
\label{sec:alg}

We propose the \alg framework, depicted in \Cref{fig:overview}.
 Let the possible unique combinations of language models and prompts be denoted by $\{ \LMPP_1, \LMPP_2, \ldots, \LMPP_K \}$, where $K \leq mp$.
 When a new question $Q_i$ arrives, \alg starts by selecting a model $\LM \in \mathcal{M}$ and choosing an associated prompt $P \in \mathcal{P}$ to generate a response, denoted as $\hat{O}_i = \LM(P(Q_i))$. \alg returns this as the final response $\hat{Y}_i$ for this question if it has a high degree of confidence in its correctness, and deducts the cost of the question and its response from the total budget $B$. 
 %\carlee{Does CURE return both the query response and its confidence for each (model, prompt) selection?} \xuechen{response only, no confidence currently, but easy to add.}
 Otherwise, \alg can select another LLM $M$ and prompt $P$ (whose choice may be informed by the result of all previously chosen models, prompts, and their responses) and re-query.
 This iterative process continues until \alg returns a final response (based on its learned policy).
 %\carlee{who determines what is a final response? Is there a hard-coded criterion, e.g., on the confidence, or does CURE learn when to stop submitting this query?} \xuechen{Action $a_1$: return the current response as final response}%response $\hat{O}_i$ in which it has confidence, and finally returns it as the final response $\hat{Y_i}$.
  \alg then proceeds to the next question with the remaining budget and repeats the process, until all questions have been answered or there is no remaining budget. 
 We model the problem as a Markov decision process as described in \Cref{sec:problem}. %, consisting of a set of states $\mathcal{S}$ and a set of actions $\mathcal{A}$. 

\textbf{States.}
%In the state-space, 
%The state vector captures the current context: the history of responses, the history of how many times the models were previously re-queried, information about the current question, the most recent current model and prompt, and the remaining budget. 
%\carlee{does it also include the current model/prompt used? That would be useful for determining the next model/prompt given the action taken} \xuechen{yes. }
The state vector contains the following information: 
\squishlist
% \setlength\itemsep{0em}
% \vspace{-10pt}
\item \textbf{Response consistency:} Records all previous responses
%\jiasi{for the current query? for all queries?}\xuechen{for all queries} 
and the normalized frequency of their occurrences.
The intuition is that the consistency of the previous responses can be used as a measure of confidence in the response correctness~\cite{wang2022self,madaan2023automix}.
% We also explicitly calculate the $z$ score and include that in the state, $z = \frac{data - mean(data)} {std(data)}$, to measure the variance in the responses \carlee{$z$ score of what? Is ``data'' the frequency  of occurrences of each response?}
\item \textbf{Input and output length:} The number of tokens in the current query and any preceding responses to the same query.
%\jiasi{of all queries? or just the current one?}\xuechen{current one},
This helps \alg understand the monetary price of each query and response, which can differ for each query. It also helps capture the difficulty, as question with longer input or output tend to be harder.
%\xuechen{Also difficulty? questions with longer input/output are harder.}
\item \textbf{Current question's text embedding:} %\xuechen{Text embedding of current question}
Intuitively, we want to capture the question type or difficulty, which can impact the model and prompt selection decision.
\alg does this using a text embedding of the query \cite{openai_embedding_model}. 
\item \textbf{Number of re-queries:} 
%We keep track of the number of times each language models and prompting schemes pair $\LMPP_k$ queried in order to answer the current question. 
The number of re-queries for each model-prompt pair helps \alg decide whether to re-query again or move to the next question.

%This feature additionally keeps track of the most recently used model and prompt.
% \item \textbf{\xuechen{Number of questions remaining:}} 
% %Records how many queries have been seen so far.
% %\item \textbf{{Remaining number of questions:}} Number of remaining questions that need to be answered by the policy.  
% This impacts how thrifty the algorithm should be at the current time instance. If many queries \xuechen{questions} remain, \alg should select cheaper (LLM, prompt) combinations.
% %\jiasi{How do we know the number of remaining queries? Just assume it's given? Need to discuss}\xuechen{assume it's given} \jiasi{thinking how to overcome this...} \xuechen{can we highlight the algo is online? We don't know the diff distribution for all queries.} 
\item \textbf{{Normalized remaining budget:}} Based on the remaining budget, we compute the estimated number of queries for each model prompt pair as follows:
%We compute the remaining budget for each possible combination $k$ of model and prompting scheme,
$\mathcal{B}_k =  \frac{\text{total remaining budget}}{(\text{\# questions remaining})(\text{avg cost per query of $\LMPP_k$})}$.
%\jiasi{queries? questions?} \xuechen{current one is correct.we compute how many times we could use each prompt-model pair for a single question}and include this in the state. 
The average cost per query is estimated based on the questions seen so far.
%The remaining budget plays a role in model selection. 
If there is a large remaining budget, \alg may consider re-querying with large models.
\squishend

%\vspace{-0.3cm}
\textbf{Actions.}
%Given a state $s$ and the last used $\LMPP$ combination for the current question $Q_i$, if any,
%and the current respective $\LMPP_i$,
The action space $\mathcal{A}$ consists of the following: %that operate on $s_t$: 
\squishlist
% \setlength\itemsep{0em}
% \vspace{-10pt}
\item Action $a_1$: Return the current response for $Q_i$ and proceed to the next question $Q_{i+1}$. If no models have been queried yet and this action is chosen, it is equivalent to skipping the question.
\item Action $a_2$: Re-query the same model-prompt pair $\LMPP$ for $Q_i$.
\item Action $a_3$: Select a new model-prompt pair $\LMPP^\prime$ for $Q_i$. 
%\item \jiasi{Should there also be an action that allows to skip to the next query? Based on the experiment section} \xuechen{We don't allow skipping to the next query}
%Go to the next model-prompt pair, i.e., $\LMPP_{i+1}$ \jiasi{Is this supposed to be $(M+P)_{i+1}$? Need to clean up indices because $P_i$, as defined earlier, means the i'th type of prompting scheme, not the prompting scheme of the i'th query} \xuechen{should we change the notation of \LMPP ?}
% \vspace{-0.2in}
\squishend

\vspace{-3pt}
By allowing re-querying (action $a_2$), the current action influences the next state, by impacting the question under consideration and thus the relevant state features, making this a non-trivial MDP.
%\carlee{It might be useful to point out that, by allowing re-querying as an action, our actions can influence the next state (by determining which query is considered in the next state), thus making this a non-trivial MDP.}
For $a_3$, we constrained the set of possible model-prompt pairs to a sorted list.
%, which we experimentally found worked better.
In other words, instead of allowing \alg to select any possible model and prompting scheme, we sort the $\LMPP_k$ in ascending order of accuracy to cost ratio and only allow \alg to select the next element in this list $\LMPP_{k+1}$.
The ordering is based on \Cref{prop:llm_ordering} (discussed below), % and our experimental characterization of different $\LMPP$ combinations  (\S\ref{sec:data_collection_training}).
% \xzC{put lemma here?}
%\carlee{Is it always true that accuracy and cost both increase at the same time? If we want to include some theory in the paper, maybe we could come up with a model that justifies this. E.g., if we assume there exists a known sequence of model-prompt pairs such that for a given query, the probability of returning an accurate answer and the cost are both strictly increasing as we follow the sequence. Then we might be able to show it is optimal (or at least within some fraction of optimal) to follow the sequence, given we do not know the exact probability of success for each LLM-prompt pair. It will take some space to explain all this though, so not sure if it's worth it for ICML.}
%\jiasi{Is a model with better accuracy always more expensive? If there is a model that is more expensive but less accurate, how is this sorted? Related to concept of partial/total ordering.} \xuechen{yes. GPT$+$ standard prompt is better and also cheaper than GPT$+$ COT. So we just skip it}
%\jiasi{ask Xuechen}
% It is reasonable to assume that the model-prompt pairs are ordered (roughly) in ascending order of model capability and associated costs.

\textbf{Rewards.}
The reward function assigns a positive reward to correct responses.
Specifically, $r_{\tau_a}\left(\tau_{(Q,s)}, \tau_{(Q^{\prime},s^{\prime})}\right) = \mathbb{P}\left[\hat{Y} = Y | \tau_a=a_1 \right] + \lambda \mathbb{P}\left[\hat{O} = Y  | \tau_a \in \{a_1,a_2, a_3\} \right]$. 
%\jiasi{earlier notation is $r(\tau_{(q,s)_t},\tau_{a_t})$, which is different?}
For a given question, this combines the accuracy of the final response $\hat{Y}$ with the accuracy of the current response $\hat{O}$ (if there have been re-queries), with a scaling factor 
%\carlee{what does this mean?} 
$\lambda$ between the two terms. 
%\carlee{More formally, to keep this an MDP shouldn't we write the incurred reward as a function of the action?}
We introduced the second term because without it, we observed that if \alg repeatedly chose action $a_2$ (re-querying), this would result in multiple state transitions with 0 reward, until the final response was returned.
In other words, including the second term avoids the issue of sparse rewards that resulted from the first term alone.
Note that the correct response $Y$ is known only when training \alg; during test, the policy executes using the expected reward calculated by the trained policy. %\jiasi{right?}
%accuracy of the final response would lead to sparse rewards. 
%To harmonize the influence of these two terms, we introduce a parameter $\lambda$.
%\jiasi{Describe what this means. What is $\lambda$?}
%$\hat{Y}$ only changes once the final answer is returned, leading to sparse rewards. Hence, we also assess the correctness of the current answer.
%\jiasi{can you write the reward equation explicitly to help the reader understand the partial and full reward?}
% \subsection{Requery also help with trustworthy:}

%\carlee{Another place to add some (simple) theory: show that if we know the rewards from each model/prompt choice and dis-allow requeries, this reduces to an offline knapsack problem}

%The denominator is adaptive to different model costs and can be updated online.
%\jiasi{how often is is it updated online? once per hour for the 24-h trace? or for every time you pick a new M+P?}\xuechen{once per hour}
%\xuechen{move to appendix?

 \textbf{Design choices and justifications.}
We next discuss two key design choices of \alg and their theoretical motivation.
Proofs are in \Cref{app:theory}.
%intuit TREACLE's design and are also supported by our experiment results. 

\textbf{\emph{(1) How should the LLMs and prompts be ordered in the cascade?}}
% Let the possible unique combinations of language models and prompts be denoted by $\{ \LMPP_1, \LMPP_2, \ldots, \LMPP_K \}$, where $K \leq mp$,
Recall that action $a_3$ moves to the next $\LMPP$ in the cascade. What is the best ordering of $\LMPP$?
Consider the following simplified setting.
Suppose each of the $\LMPP$ have a probability of correct response $p_k$ and cost $c_k$. %for
%Suppose the items in the set are ordered, which we call a \emph{cascade}.
If we had access to an oracle that 
could tell us when the response of a particular $\LMPP_k$ is incorrect, we could then move on and try the same question with the next option $\LMPP_{k+1}$ in the cascade.
We could achieve the highest accuracy using the oracle, and would only have to worry about minimizing the cost to avoid exceeding the budget.
In this setting, \Cref{prop:llm_ordering} below states that the best ordering of the $\LMPP$ options is according to the ratio $\frac{p_k}{c_k}$.
%. The procedure continues until a correct answer is obtained. The goal is to minimize the expected cost of inference. 

\vspace{-2pt}
\begin{proposition}
With $K$ (LLM, prompt) options, each with probability of correct answer $p_k$ and cost $c_k$,
%Also suppose their outputs are independent of each other. 
ordering the options according to their cost-normalized accuracies $\frac{p_k}{c_k}$ minimizes the total cost.
\label{prop:llm_ordering}
\end{proposition}
\vspace{-5pt}

This Proposition motivates \alg's ordering of the $\LMPP$ options in the cascade according to their accuracy and cost ratio.
%The ordering bears some similarity to online knapsack algorithms, where items are typically chosen according to their value-to-weight ratios~\cite{chakrabarty2008online}.
This is intuitive: instead of placing the most accurate (LLM, prompt) option early in the cascade, which might incur large cost, we first query LLMs that have high accuracy per unit cost.
 %The detailed proof is in \Cref{subsec:cascade-ordering}. 
Note that although the setup of \Cref{prop:llm_ordering} differs from \Cref{eqn:rl_problem}, as the cost is the objective rather than a constraint, the trajectory resulting from the ordering in  \Cref{prop:llm_ordering} is also a solution to \Cref{eqn:rl_problem}.
%Note that although the setup of \Cref{prop:llm_ordering} is different from \Cref{eqn:rl_problem}our problem statement above because here the cost is in the objective where it is a constraint. However, this can still motivate our design of leveraging different LLM-prompt pairs in order because the solution to the setup of Proposition 1 is also a solution to the original problem. 

%\jiasi{To discuss: notation of $O$ vs $Y$, $\Omega$ vs $K$ vs $L$, $\qb$ vs $Q_i$. Models not unique in the cascade. Proof outline.}

\textbf{\emph{(2) Do policies that consider response consistency perform well?}}
Recall that ``response consistency'' is one of the features in the state vector.
%Our problem formulation is general enough to allow the same question to be re-queried multiple times to the same or different LLMs.
%Intuitively, a consistent response to re-queries suggest the response is more likely to be accurate.
We seek to understand the performance of policies that consider this feature; a simple such policy is described in \Cref{def:n_consistent} below. It returns a final response to a question if the same response value is repeated $w$ times. %by the $\LMPP$ options in the cascade.

%Under Definition \ref{assump 1}, we have the following result for the optimal LLM cascade policy.
\begin{definition}
%Let $\hat{O}_i$ be the response from querying question $Q_i$.
% For each question $Q_i$, an $n$-consistent policy ($n\geq2$) sets the final response $\hat{Y}_i = \hat{O}_i$ as soon as $\exists \; a : \hat{O}_i=a$, $n$ times, or to another response otherwise.
For each question $Q_i$, an $w$-consistent policy ($w\geq2$) sets the final response $\hat{Y}_i = \hat{O}_i$ as soon as $\exists \; \hat{O}_i : \mathtt{count}(\hat{O}_i)=w$. If no such $\hat{O}_i$ exists, fall back to $w-1, w-2$, etc. %the fallback accounts for the old "1-consistent" policy in Prop 2.
\label{def:n_consistent}
\end{definition}
\vspace{-5pt}

\Cref{def:second_moment} below characterizes how likely the $\LMPP$ are to return an incorrect response.
A question can be asked $\Omega$ times to the $\LMPP$ options in the cascade, which may not be unique due to the re-queries.
%, which are not necessarily unique due to the re-queries.

%indices used in this paper:
%i: over the questions
%k: over the unique cascade options
%j: over the non-unique cascade options
\begin{definition} 
%Let $\hat{O}^*:=\hat{O}_{\qb}^*$ be the correct answer of a random problem $\qb$. 
Denote the $\Omega$ LLM-prompt options by $\LMPP_{j=1}^\Omega$. 
%\xuechen{$\Omega$, query index, $\sup$ might need removing}
Let $\Pro(M_j(P_j(Q_i)))$ be the output distribution of $\LMPP_j$ on problem $Q_i$. Let
%\zijian{might need modification because we consider a set of problems DOUBLE CHECK NEW DEF WITH SAMET} \xuechen{Should we define the distribution of requery only, regardless of model prompt pair?} \jiasi{we mention above that the re-query can be to the same/different LLM}
%\xuechen{current $\epsilon$ is directly influenced by the temperature}
$
    \epsilon :=   \sum_{i=1}^n \sup_{1\leq j \leq \Omega} \sum_{\hat{O} \neq Y_i }\Pro(M_j(P_j(Q_i))=\hat{O})^2. \nonumber
%\end{align}
$
\label{def:second_moment}
\vspace{-0.15in}
\end{definition}

%\jiasi{Experimentally, we find that $\eps$ is about X.}\xuechen{Experimentally, we find that $\eps$ is about 0.068. The output distribution is computed with all LLM-prompt pairs in our experiments. The value is an average of 500 GSM8K examples.} \xuechen{Using the GSM8K dataset and a budget of \$0.30, with $\alpha=\frac{1}{20}$, the final response output, determined by the RL policy for consistency, achieved an accuracy of 88.52\%.} \jiasi{accuracy is different than $\epsilon$?}
\vspace{-5pt}
With the definitions in hand, we can now lower bound the performance of a 2-consistent policy compared to the optimal learned algorithm in \Cref{prop:accuracy_loss_2_consistent} below.
Without loss of generality, we study the case when the reward function is the accuracy. % \xuechen{n is different in definition 1 and 2. here n is the number of queries}

\begin{proposition} For the problem stated in  \Cref{eqn:rl_problem} that achieves $C_*$, the optimal expected accuracy subject to budget constraints,
%Then there exists a policy that always returns an answer for the current question $Q_i$ after achieving 1- or 2-consistency 
there exists a 2-consistent policy that achieves an accuracy of at least $C_*-\frac{1}{2}\Omega^2\epsilon$. %\jiasi{ask Samet: should we keep 1-consistency?}
\label{prop:accuracy_loss_2_consistent}
\end{proposition}

\vspace{-5pt}
In other words, even a simple policy that allows for re-querying and considers response consistency can achieve close to the optimal reward.
This motivates \alg inclusion of ``response consistency'' as a state feature.
The proposition applies generally and allows for budgets or text embeddings, and does not require the $\LMPP$ to return accurate responses.
 %The proof of this proposition relies on reducing the optimal policy to an at most 2-consistent policy. We then control the excess error of this new policy by bounding the likelihoods of more than 2-consistent cascade instances, which is captured by the $0.5\Omega^2\eps$ term.
Experimentally, we find that our learned RL policy is similar to a 2-consistent policy, as 93.02\% of the responses are 2-consistent (GSM8K dataset, \$0.30 budget, $\alpha=\frac{1}{20}$). %, with the fraction decreasing to 85.68\% (with a budget of \$10).
This suggests that our learned policy may not be far from optimal.
%\xuechen{93.02\% is 1 or 2 consistency. Is it dangerous to say 'our learned RL policy is nearly a 2-consistent policy', The ratio of 2 consistency is also influenced by the budget. e.g. This percentage increases modestly to 14.32\% when the budget is raised628
%to \$10 which is more adequate.} 
%\jiasi{edited text to soften the claim. I don't think we need to include the numbers for other budgets, given the space constraints. Will update about 1-consistent}

%\input{sec/uncertainty}

\section{Experiments}
\label{sec:experiments}

We first describe the experiment setup (\S\ref{sec:exp_setup}) and then the main results (\S\ref{sec:exp_results}). %, \S\ref{sec:uncertainty_exp_results}).
Specifically, we examine robustness to new LLMs and changing API prices (\S\ref{sec:new_llm}), %time-varying API query latency (\S\ref{sec:latency}), 
shifts in question difficulty (\S\ref{sec:domain_shift}), and different reasoning datasets (\S\ref{sec:datasets}).

%To demonstrate the advantages of \alg and identify the factors contributing to these benefits, our evaluations center on the following questions:

%\begin{itemize}
%\vspace{-5pt}
%\setlength\itemsep{1pt}
%    \item Can $\alg$ adapt to the distribution shift of question difficulty between training and testing data? \carlee{is there always a shift?}
%    \item Can $\alg$ effectively adapt to varying latency over time?
%    \item How well does $\alg$ accommodate different question re-orderings? \carlee{does this mean different sequences of queries? Is it necessarily a re-ordering of the same queries or could different queries be introduced?}
%    \item Can $\alg$ handle questions of different reasoning datasets?
%    \item What impact does re-querying have on the cost-efficiency of the policy?
%    \item How does re-querying affect the estimation of uncertainty by $\alg$?
%\end{itemize}
%\jiasi{re-order questions to match with results}
%\xzC{add questions for new observations}

% \vspace{-0.2cm}
\subsection{Experiment Setup}
\label{sec:exp_setup}

% \iffalse
%  \begin{wrapfigure}{r}{0.4\textwidth}
%  \vspace{-10pt}
% % \vspace{-0.1in}
%   \includegraphics[width=0.4\textwidth]{img/accuracy_cost_latency.pdf}
%  \caption{\small{Characterizing accuracy, cost, latency of different model-prompt pairs $\LMPP$ on the GSM8K test dataset. Higher accuracy corresponds to higher price or and lower latency.}} % \jiasi{thanks for the update. larger font in all axis subtitles please} \xuechen{not mentioned in main text?}}
% \label{fig:model_tradeoffs}
% %\end{subfigure}
% %\vspace{-15pt}
% %\caption{}
% \vspace{-10pt}
% %\label{fig:framework}
% % \vspace{-0.6cm}
% \end{wrapfigure}
% \fi

%\vspace{-5pt}
We summarize the experiment setup, with full details in \Cref{app:implement}.
We use three representative datasets: \textbf{GSM8K \cite{cobbe2021training}}, which contains 8.5K high quality grade school math problems created by human writers; \textbf{CSQA \cite{saha2018complex}}, which consists of 12102 multiple choice commonsense reasoning questions encountered in daily life; and \textbf{LLC~\cite{wei2022chain}}, where the task is to concatenate the last letters of words in a name (e.g., ``Amy Brown'' $\rightarrow$ ``yn''). 
To evaluate our methods, we perform two steps.

\emph{\textbf{(1) Collect query-response pairs for (LLM, prompt) combinations}}.
We collected query-response pairs from each dataset for different combinations % $(M_i,P_j,T_k)$ 
of LLM, prompt, and LLM temperature.
%The accuracy, latency, and monetary price of the best combinations are shown in \Cref{fig:model_tradeoffs}.
%We selected those combinations according to \Cref{prop:llm_ordering}.
    We used 5 different LLMs: \textbf{Llama-2-7b-chat}, \textbf{Llama-2-13b-chat} \cite{touvron2023Llama}, \textbf{GPT-3.5-turbo}, \textbf{GPT-4}, and \textbf{GPT-4-turbo}~\cite{OpenAI2023GPT4TR}. These models are of varying sizes (7b, 13b, 154b and 1.76t respectively).
    The Llama models are open-source and run locally on our servers,  %(one A40 GPU for Llama-2-7b and two A40 for Llama-2-13b), 
    %\xuechen{I think we use 1 A40 for 7b and 2 A40 for 13b. @zijian, please double check} \zijian{checked. 1 A40 for 7b and 2 A40 for 13b} 
    while the GPT models rely on commercial APIs. 
    %\jiasi{add details about servers that Llama ran on, since it impacts our latency results.}
    %\xuechen{details in the original version: Our approach can also incorporate additional models such as Llama-2-70b-chat to fill the scale gap between smaller Llamas and GPTs. We opted to not include the larger Llama2 as it incurs high latency without inference optimization (16s/query for plain question query and 28s/query for CoT query in 5 A40 GPUs in our experiments). }
    %\jiasi{I removed this part about the larger Llama model since it might invite more questions, and does not seem like a very important point.}
We employ several prompting schemes.
%to elicit the reasoning abilities of LLMs.
%The full prompts are shown in \Cref{app:full_prompts}.
    %\jiasi{Explain what is ``system message'' and ``user's content message''} 
    A prompt generally consists of two parts: the ``content message'' containing the question, and the ``system message'' with additional context. 
    %The full prompts are given in \Cref{app:full_prompts}. % given to an LLM. %the roles they are or the environments they are in.

    \squishlist
         \item The \textbf{plain text prompt} 
         %For this easiest prompting strategy, we 
         submits the questions to the LLM as the content message. % (no system message). %, which is shown in 
        % \item The \textbf{``Math Solver'' prompt} \xuechen{explain in general. domain expert?} feeds the message ``You are a math solver. Give the answer to the following question in the boxed \{\}'' as a system message, while keeping the user's content message as a plain text, as shown in \Cref{tab:GSM8K Math Solver prompt}. 
        % This strategy is used in the GSM8K experiments \jiasi{why not CSQA?} \zijian{Because GSM8K is a math datset and this prompt is a special prompt designed for it.}, and is especially useful for GPT models \jiasi{what does ``useful'' mean? higher accuracy for slightly more cost, compared to standard prompt?} \zijian{Yes.}. 
        \item  The \textbf{domain expert prompt} feeds information about the question's domain as a system message (\eg ``math solver''), and keeping the user's content message as plain text. %The full prompts are shown in \Cref{tab:GSM8K Math Solver prompt} in the Appendix.
        
        \item The \textbf{standard few-shot prompt}  includes a system message (``Follow the given examples and answer the question'' \cite{wei2022chain}) and the content message, which consists of few-shot examples together with the plain text prompt.
        %It tends to improve response accuracy compared to the plain text prompt.

        \item The \textbf{Chain-of-Thought (CoT) few-shot prompt}~\cite{wei2022chain} %Compared to the standard prompt, CoT prompt 
        adds some intermediate explanations to the few-shot examples. %(see \Cref{tab:GSM8K CoT prompt,tab:CSQA CoT prompt} in the Appendix).
        %CoT techniques are widely used in reasoning tasks due to their mimicking of human thought processes.
        %Note that although some other variants of CoT such as complex CoT \jiasi{ref} \zijian{It is in a github repo} can outperform the original CoT prompt, but we do not employ them because they need much longer prompts and accordingly cost substantially more, and we observed that the \alg essentially did not make use of them due to their inefficiency. 
        %The cost is often more than the current model-bigger model price difference, effectively making them infeasible for our applications.
    \squishend
    %\emph{\textbf{Temperature.}}  The LLM temperature is a configurable parameter that influences the variety of the responses it generates.
    %With a higher temperature, the model may output more diverse but possibly inaccurate responses.
    %We set the temperature to 0 for a new query, and to 0.8 or 1.0 for a re-query for Llama and GPT, respectively. 
    
    %our method can be extended to more different temperature values.
% Presently, data collection involves gathering outcomes for the first 200 inquiries from the training set of gsm8k. Each individual question is subjected to five re-queries. The initial deployment encompasses three models: GPT3.5 and Llama 13B.

\textbf{\emph{(2) Train \alg with the query-response pairs.}}
%\label{subsubsec:trajectory generation}
%We designed our \alg framework as described in \S\ref{sec:rl} and 
We used Deep Q-Network (DQN)~\cite{mnih2015human} to train the reinforcement learning (RL) policy in \alg, consisting of a two-layer neural network.
%To generate diverse trajectories consisting of $(s_t,a_t,r_t,s_{t+1})$, we use the collected query-response data and employ $\epsilon$-greedy exploration.
%for \alg and \algU
%We sample a new action with probability $\epsilon$, or choose the action according to the current policy with probability $1-\epsilon$. As the training progresses, we decrease the value of $\epsilon$ slowly, resulting in a stable policy. \carlee{is this just $\epsilon$-greedy exploration? We can call it that and leave out the details to save some space}
For the monetary prices, we use the published per-token prices for the GPT models. 
Since our local Llama deployments do not have API costs, we set Llama-2-7b's price as $\alpha$ times Llama-2-13b's price, and Llama-2-13b's price as $\alpha$ times
%carlee{why $\alpha^2$? To explore variations in prices?} or $\alpha$ fraction of the 
GPT-3.5-turbo's price.
$\alpha$ varies between $\frac{1}{10},\frac{1}{20}$ or $\frac{1}{50}$.
Our pricing is grounded in reality and similar to actual market rates, as the offered price for Llama is approximately 15\% of GPT-3.5-turbo according to current providers~\cite{togetherpricing}. 
%A summary of the average cost per query of different models and prompt combinations are shown in \Cref{tab:single model cost} in the Appendix.
%Considering their profit margin, deploy the models locally could be even cheaper.
%To set the monetary price, we use the published per-token prices for the GPT models.   Since Llama runs locally and there is no API cost, to approximate the resource consumption costs, we set the Llama monetary prices to a fraction $\alpha$ of the GPT prices. %, in order to study \alg's performance for a range of prices.
For the latency-accuracy tradeoff, we evaluate different trade-off parameters $\beta=[50\text{k},500\text{k},1\text{M}]$ in the cost function.
We evaluated the following baseline methods, reproducing the methods as faithfully as possible with a common set of LLMs and prompt options.
%\carlee{These baselines are explained well, but I think it may be more useful to point out what comparing to each of them shows. That might also allow us to condense the explanations of results} \xzC{reproducing recent work with our model-prompt options}
%\ege{we can add that these baselines measure consistency actually}: 
\squishlist
    \item \textbf{FrugalGPT}~\cite{chen2023frugalgpt}. We reproduce FrugalGPT, which uses a DistilBERT model~\cite{sanh2019distilbert} to estimate the response accuracy. If this estimate is below a threshold, the next LLM in the cascade is queried. %\xuechen{estimate the accuracy with and without previous response(re-query)}
    This baseline shows how \alg compares to the state-of-the-art that lacks re-querying.
    %\jiasi{for response accuracy? confidence?} \xuechen{They use 'reliability score' in the paper. I think predicted accuracy is the same as confidence and reliability score.}. 
    %The threshold for the scoring function is fine-tuned based on the budget \jiasi{using the validation set?}. \xuechen{yes. Not fine-tuned, grid search}
    \item \textbf{Calibrated cascade.}
We build on FrugalGPT's response accuracy estimation and develop a 2-layer neural network, which takes as input \alg's state vector and outputs the estimated response accuracy.
%develop and only return responses that are above a confidence threshold. 
%The model 
%the consistency of query answers, the number of queries, Output statistics of queries, Input/Output length and Text embedding of the queries. 
 If this estimated accuracy is below a threshold, the next LLM in the cascade is queried. 
%The threshold is tuned on a validation set based on the total budget.
This baseline compares \alg to a modified FrugalGPT.
%\xuechen{hard-coded criterion, set a confidence threshold, the threshold is determined by budget}

    \item \textbf{\BMajority}. 
    %\xuechen{rename? MoT} \ege{just "Majority Voting"?} 
   % We conduct experiments based on the majority voting method from these works, which 
% Majority voting allows each LLM to be queried up to $N$ times \jiasi{\ABCascade also uses $B$. I changed to $N$} \xuechen{I think $B$ has the same meaning in \ABCascade and \BMajority. Both are query times}. This follows the intuition that more common answers are more likely to be correct.
    For each query, we output the final response based on the majority vote from $\Omega$ re-queries, based on \cite{wang2022self,yue2023large}.
    We set $N=2$ based on the best empirical results.
    %(or the most recent response if there is no clear majority). 
    The (LLM, prompt) combinations are progressively queried until their per-question budget 
    %(total budget divided by number of remaining questions)
    runs out.
    This baseline allows comparison with \alg's response consistency feature in the state vector.
    %sorted from the cheapest model to the most expensive model (equivalently, from the highest to lowest accuracy). Use the same model sequence as Consistency Cascade, \alg and single model.} 
    % larger $N$ results in significant increase in cost but limited improvement in accuracy. 
    %\jiasi{how does \BMajority use the following budget info?} \xuechen{decide query times}
     %We calculate the average cost for each query time, and determine the query times for each question according to the average remaining budget, which is calculated by dividing the remaining budget with the number of remaining questions. \carlee{is query time the same as latency? Does this method consider monetary cost?} 
    %\xuechen{Our \alg decide the final response combining several features instead of simply majority voting based on responses. (majority voting v.s. calibrated confidence?)}

        \item \textbf{Offline and online knapsack}. We formulate a multiple choice knapsack problem where the items are the $\LMPP$ combinations. %, the values are the correctness probabilities, and the costs are the latency and monetary price functions. 
        %We also assume that the number of queries is known. %It constitutes the most relaxed version of our problem setting and 
    We solve the offline knapsack to find the optimal solution when re-queries are not allowed, and also implement an online version~\cite{chakrabarty2008online}.
    %\jiasi{did we do knapsack with re-queries or not?} Note that we are able to apply the offline Knapsack algorithm to requery-allowed setting, but the latency of the algorithm scales exponentially since we need to calculate all possible requeries and sort them in terms of price. 
    %\item \textbf{\OnlineKnapsack}~\cite{chakrabarty2008online}.
    %Following the same assumptions as the \OfflineKnapsack, except without knowing the total number of queries, results 
   % We also implement an online approximation algorithm~\cite{chakrabarty2008online}.
    These baselines show how \alg compares to methods with perfect knowledge of question costs and accuracy. 
\squishend

%\jiasi{I feel we are a little loose about setting parameters based on experimental results, like $A, B, N$. Can we say that we determined them from the validation set? Or was it from the training or test set?}

% \vspace{-0.5cm}
\vspace{-5pt}
\subsection{Results}
%\subsection{$\alg$ optimizes performance within budget constraints}

% \input{tables/single model accuracy vs cost}

\label{sec:exp_results}

\begin{figure*}[]
\centering
% \vspace{-10pt}
%\hspace{-35pt}
\begin{subfigure}[b]{0.32\textwidth}
    \centering
	\begin{tikzpicture}
		\node at (0,0) [scale=0.30]{\includegraphics{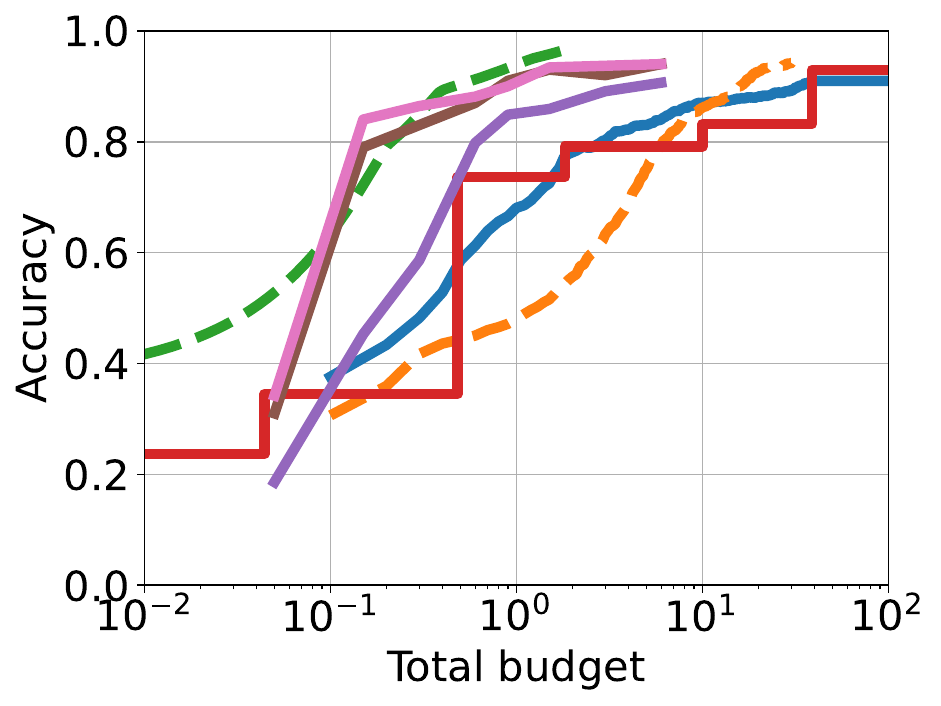}};
  % \draw[blue, very thick] (0,0) circle (1.5);
  \node at (0.8,-0.9) [scale=0.78] {GSM8K, $\alpha = \frac{1}{50}$};
\end{tikzpicture}
\end{subfigure}
\begin{subfigure}[b]{0.32\textwidth}
    \centering
	\begin{tikzpicture}
		\node at (0,0) [scale=0.30]{\includegraphics{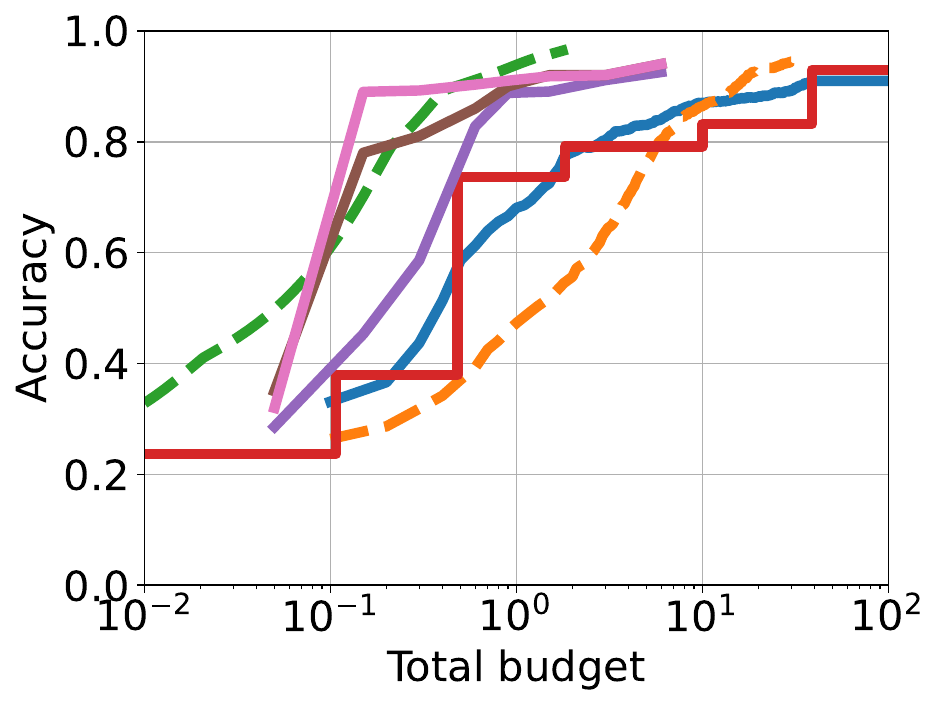}};
  \node at (0.8,-0.9) [scale=0.78] {GSM8K, $\alpha = \frac{1}{20}$};
\end{tikzpicture}\label{fig:ece}
\end{subfigure}
\begin{subfigure}[b]{0.32\textwidth}
    \centering
	\begin{tikzpicture}
		\node at (0,0) [scale=0.30]{\includegraphics{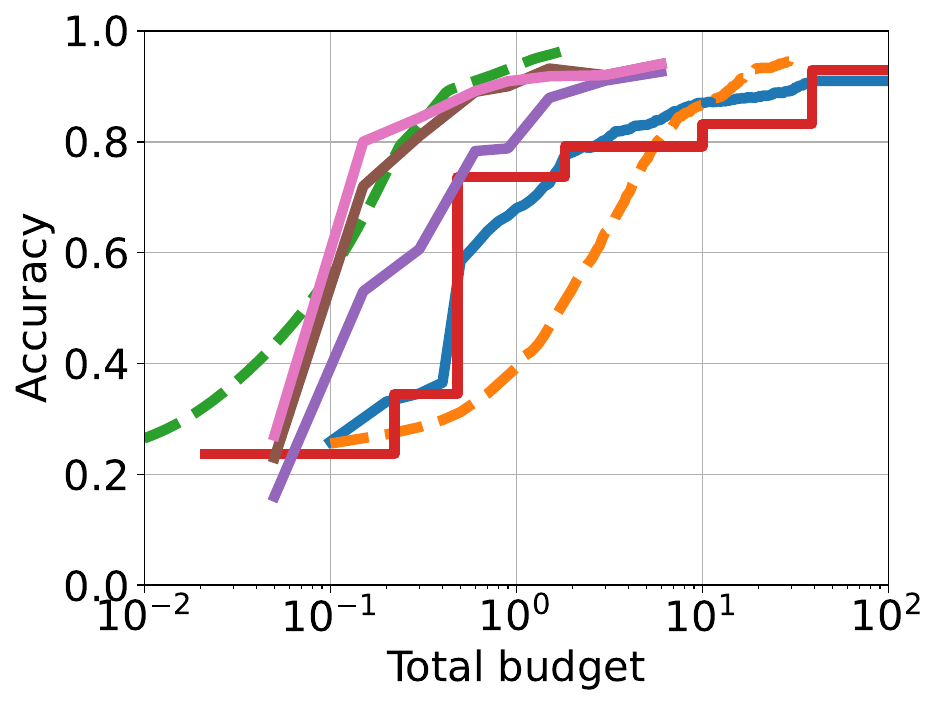}};
  \node at (0.8,-0.9) [scale=0.78] {GSM8K, $\alpha = \frac{1}{10}$};
\end{tikzpicture}
\end{subfigure}\\
\begin{subfigure}[b]{0.32\textwidth}
    \centering
	\begin{tikzpicture}
		\node at (0,0) [scale=0.30]{\includegraphics{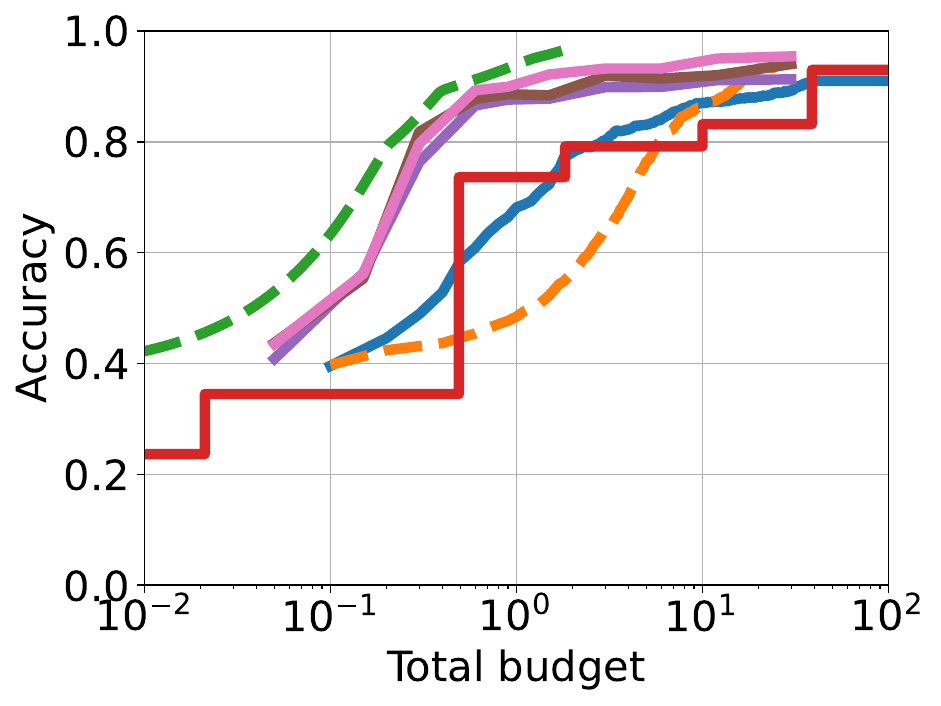}};
  \node at (0.8,-0.9) [scale=0.78] {GSM8K, $\beta $=1M};
\end{tikzpicture}
\end{subfigure}
\begin{subfigure}[b]{0.32\textwidth}
    \centering
	\begin{tikzpicture}
		\node at (0,0) [scale=0.30]{\includegraphics{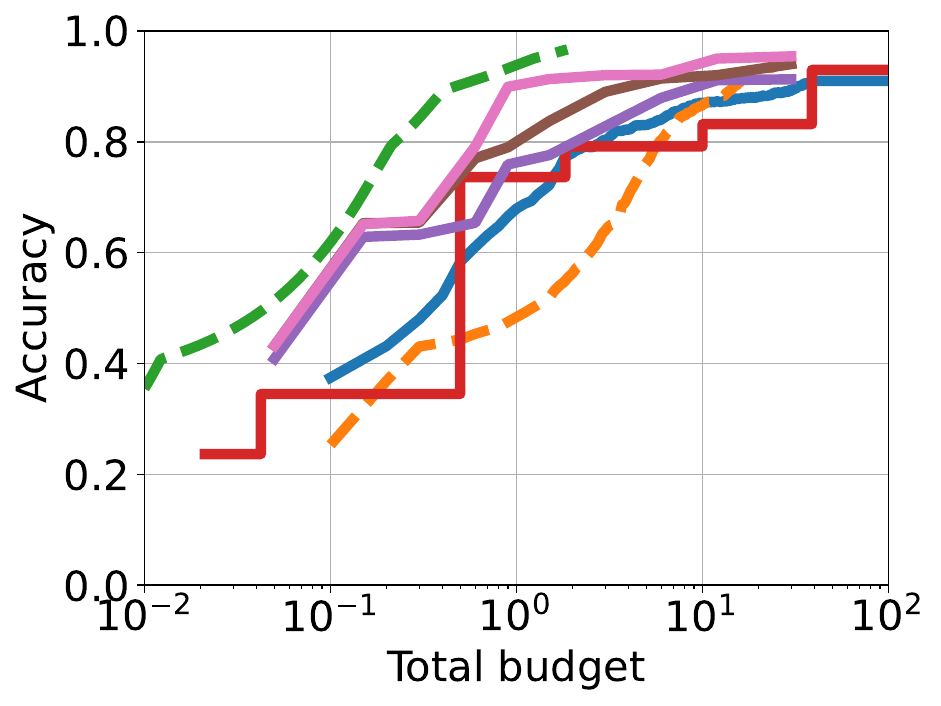}};
  \node at (0.8,-0.9) [scale=0.78]{GSM8K, $\beta =$500k};
\end{tikzpicture}
\end{subfigure}
\begin{subfigure}[b]{0.32\textwidth}
    \centering
	\begin{tikzpicture}
		\node at (0,0) [scale=0.30]{\includegraphics{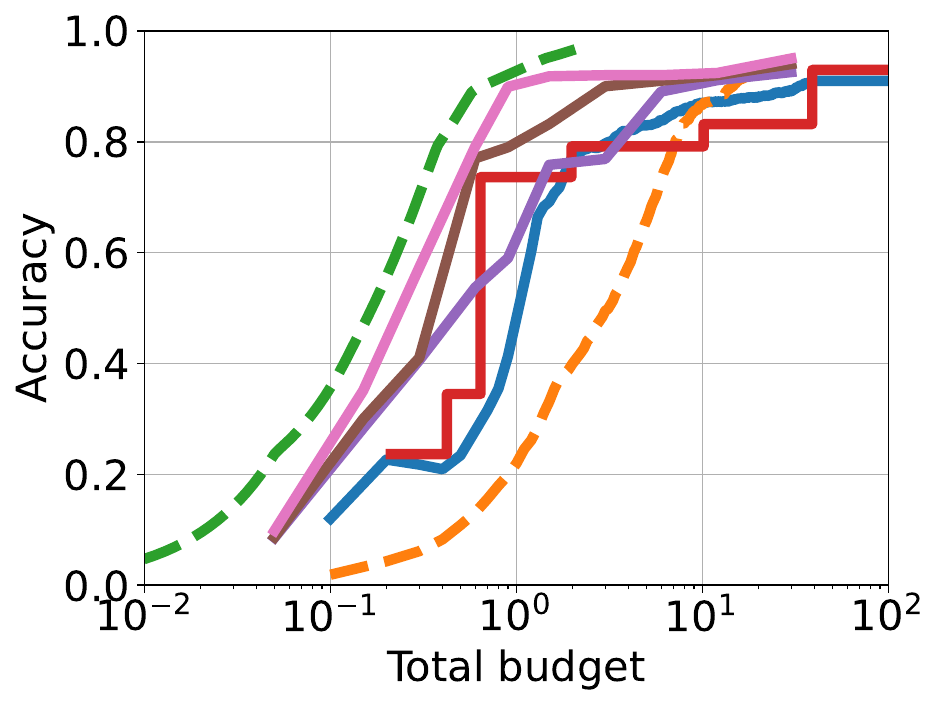}};
  \node at (0.95,-0.9) [scale=0.78] {GSM8K, $\beta =$50k};
\end{tikzpicture}
\end{subfigure}\\
\begin{subfigure}[b]{1\textwidth}
    \centering
	\begin{tikzpicture}
		\node at (0,0) [scale=0.38]{\includegraphics{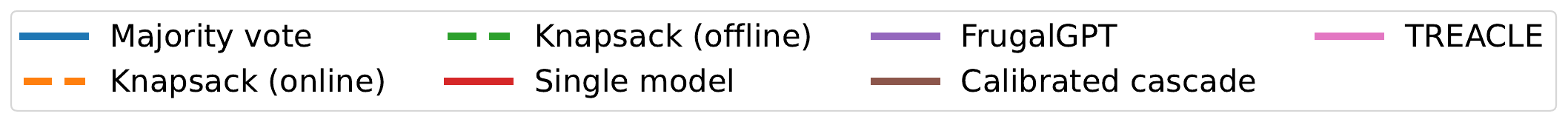}};
  % \node at (0,-2) [scale=0.8] {10 times};
\end{tikzpicture}
\end{subfigure}
\vspace{-10pt}
\caption{The performance of various methods for different cost functions and budget constraints. The dashed lines are methods that have ground knowledge, which is impractical but illustrates the best achievable performance. 
%\jiasi{what is the default $\beta$ when $\alpha$ is specified, and vice versa?}
%\jiasi{not sure if we need to keep all the plots for $\alpha, \beta$ since the text discussion is pretty short and it's not a very important point. Consider moving some plots to Appendix.}
%\xuechen{we won't use $\alpha, \beta$ at the same time. Pure monetary price with $\alpha$ only or monetary price-latency combination with $\beta$ only} 
%\xuechen{images and legends updated. remove ab cascade and change xlim} 
%\xuechen{Xuechen todo: thinking about 1. keep all the images? marker for observations. message: help to verify the answer, how? \$10, no llama}
%\jiasi{after I edited the Observations, I don't think we need to annotate Fig. 3 any more, since the Observations are mainly referring to other figures.}
}
% \vspace{-10pt}
\label{fig:all_12_new}
\vspace{-10pt}
\end{figure*}

%\textbf{Does $\alg$ improve accuracy compared to baseline methods, for different cost functions?}

%\vspace{-10pt}
To evaluate the performance of $\alg$, we conduct experiments for different total budgets and cost functions.
The results are presented in \Cref{fig:all_12_new} (additional results on CSQA and LLC are in \Cref{app:more_dataset}.)
%\xzC{need double check after changing \Cref{fig:all_12_new}.} 
Across different settings of $\alpha$, $\beta$, and total budget, \alg 
%and \algU \xuechen{\algU as baseline or a proposed method?} 
consistently outperforms the baselines and is close to the \OfflineKnapsack %, which directly selects the best model based on ground truth
-- an approach not feasible in practical deployments.
We note that the relatively good performance of the Calibrated Cascade is due to it using the same state vector we designed for \alg. 
%To delve deeper into the behavior of \alg, we perform additional experiments and 
We make the following additional observations. 
%\jiasi{should we merge/remove one of these for space?}
%\carlee{Are all of these observations in support of answering whether CURE improves performance compared to baselines? Maybe italicize the observation statements so they look different from the research questions (right now they're both bolded)}

%$\bullet$ \textbf{\emph{Observation 1: Compared to using a single LLM,
%\alg's automatic selection of the right of model and prompt can reduce over 85\% of the cost.}}
%\jiasi{Consider removing since we discuss it much earlier in Fig. 1.}
%For example, the most accurate model-prompt pair (GPT-4 with CoT) achieve 92.95\% average accuracy for \$$38.51$ on the entire GSM8K test set in \Cref{fig:clean-full}.
%$\alg$
%achieves similar accuracy ($92.17$\%) with allocated budget of \$$3$ and actual expenditures of \$$2.98$, which is a 90\% cost reduction.

$\bullet$ \textbf{\emph{Observation 1: \alg can adapt to different budgets and cost parameters.}} All the results in \Cref{fig:all_12_new}, with different budgets and $\alpha,\beta$ parameters were produced after training \alg only once.
%(with different parameter settings during the training).
This highlights \alg's adaptability to different cost function variations.

$\bullet$ \textbf{\emph{Observation 2: For limited budgets, \alg only answers questions that are more likely to produce accurate responses.}}
For example, for $\beta = 50\text{k}$ in \Cref{fig:all_12_new}, when the budget is only $\$0.05$ and insufficient for all queries, 52.7\% of the questions \alg chooses to answer are correct.
% which is much better than the random choice 
For context, the cheapest model (Llama-2-7b) can only answer 23.65\% of questions correctly.
%is compared to a baseline accuracy of 23.65\% by Llama-2-7b, the cheapest model \xuechen{which equivalent to the probability of selecting a question that can be answered correctly under random selection}.
This suggests \alg can evaluate question difficulty and opt not to respond to some questions. 
%\xuechen{Is it clear that why compare to Llama-2-7b?}
%\xzC{changed the example to GSM8K}
% For example, for CSQA with price-latency tradeoff factor $\beta = 50\text{k}$, when the budget is only $\$0.05$, which is insufficient for all the queries, 83.4\% of the queries the models \jiasi{\alg?} \xuechen{yes}picked to answer are correct.  \xuechen{GSM8K example? then we can move some of the images to appendix. Fig. 3 is too large. working on}
%Q: For example, for GSM8K with price-latency tradeoff factor $\beta = 50\text{k}$, when the budget is only $\$0.15$, which is insufficient for all the queries, 82.6\% of the questions \alg picked to answer are correct.
%As this selection is based on the overall budget and the number of queries that need to be answered, \alg outperform \OnlineKnapsack, which does not know the total number of queries.
%\jiasi{I removed this because don't want to highlight that we magically know the total number of queries}

%\xzC{merging observation 3 and 4?}

\begin{figure}[]
\centering
%\hspace{-35pt}
\begin{subfigure}[b]{0.48\textwidth}
    \centering
	\begin{tikzpicture}
		\node at (0,0) []{\includegraphics[width=1\textwidth]{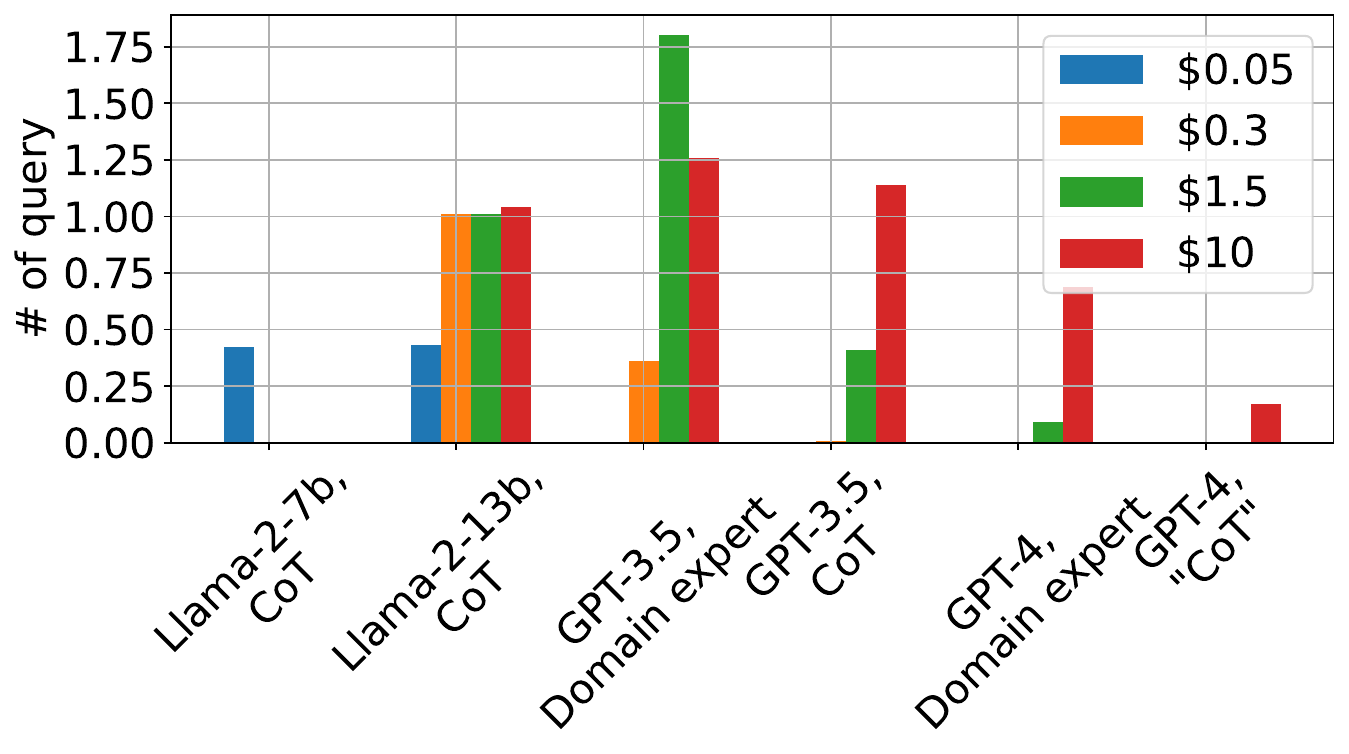}};
\end{tikzpicture}
\vspace{-20pt}
\caption{Varying budget with $\alpha=\frac{1}{20}$}
\vspace{-3pt}
\label{fig:Budget_query}
\end{subfigure}
\begin{subfigure}[b]{0.50\textwidth}
    \centering
	\begin{tikzpicture}
		\node at (0,0) []{\includegraphics[width=1\textwidth]{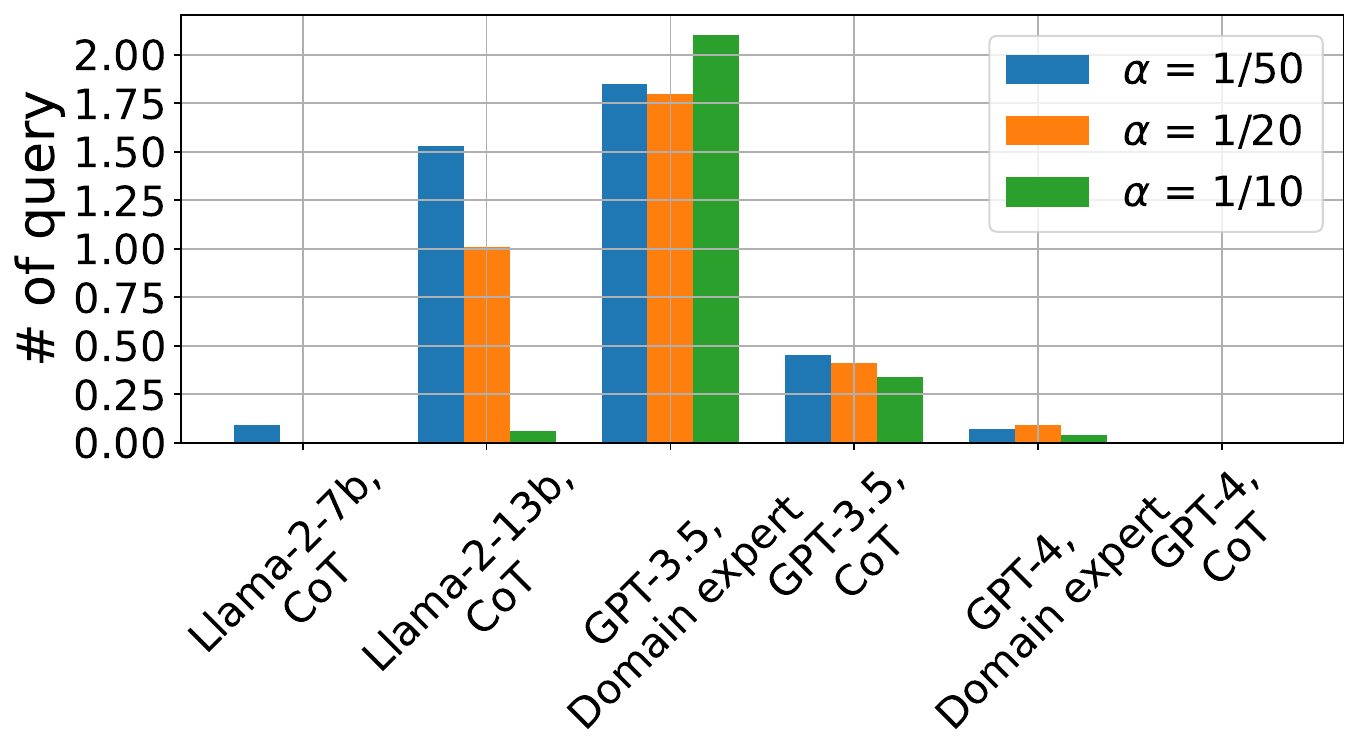}};
\end{tikzpicture}
\vspace{-20pt}
\caption{Varying $\alpha$ with \$1.5 budget. 
%The legend shows different $\alpha$ value for the pure monetary price.
}\vspace{-5pt}
\label{fig:Budget_query_alpha}
\end{subfigure}
\caption{Number of times each model is re-queried.}
\vspace{-15pt}
\end{figure}

% \begin{figure}[]
% \centering
% \begin{minipage}{.74\linewidth}
%     \centering
%      \vspace{-5pt}
%     \subcaptionbox{Varying budget with $\alpha=\frac{1}{20}$}{
% 	   \begin{tikzpicture}
% 		\node at (0,0) []{\includegraphics[width=0.45\textwidth]{img_new/Budget_query_large.pdf}};
%         \end{tikzpicture}
%         \vspace{-10pt}
%     }
%     \vspace{-5pt}
%     \subcaptionbox{Varying $\alpha$ with \$1.5 budget.\label{fig:Budget_query_alpha}}{
%   	    \begin{tikzpicture}
% 		      \node at (0,0) []{\includegraphics[width=0.45\textwidth]{img_new/Budget_query_alpha_large.pdf}};
%         \end{tikzpicture}
%         \vspace{-10pt}
%     }
% \caption{Number of times each model is re-queried.}
% \label{fig:Budget_query}
% \end{minipage}
% \begin{minipage}{0.25\linewidth}
%     \centering
%     \begin{tikzpicture}
% 		\node at (0,0) []{\includegraphics[width=0.95\textwidth]{img_new/no_requery_3.pdf}};
%     \end{tikzpicture}
%     \vspace{-20pt}
%     \caption{With and without re-querying. $\alpha=\frac{1}{20}$.}
%     \label{fig:requery}
% \end{minipage}
% \vspace{-0.1in}
% \end{figure}

% \begin{wrapfigure}{r}{0.35\textwidth}
% \centering
% \begin{minipage}{\linewidth}
%     \centering
%     \begin{tikzpicture}
% 		\node at (0,0) []{\includegraphics[width=0.95\textwidth]{img_new/no_requery_3.pdf}};
%     \end{tikzpicture}
%     \vspace{-20pt}
%     \caption{With and without re-querying. $\alpha=\frac{1}{20}$.}
%     \label{fig:requery}
% \end{minipage}
% \end{wrapfigure}

$\bullet$ \textbf{\emph{Observation 3: For larger budgets, $\alg$ chooses more powerful (LLM, prompt) combinations.}} %With larger budget, \alg can utilize more powerful LLMs and also the re-query action. 
This is shown in \Cref{fig:Budget_query},
%which shows the average number of times each model-prompt combination is re-queried.
where as the budget increases, the more powerful models (right side of x-axis) are increasingly selected.
%GPT-3.5 with the ``Math Solver'' prompt is often selected as it provides a large accuracy improvement compared to the next best combination in the sorted list. 
%\xzC{Should we add this?}
Interestingly, we observe that for budgets \$0.3 to \$10, the Llama-2-13b model is queried approximately once per question, despite its suboptimal performance. 
Even with these larger budgets, it's still beneficial to query Llama before moving onto more powerful models, to see whether its responses are consistent. %with others.

\begin{wrapfigure}{r}{0.4\textwidth}
\vspace{-5pt}
\centering
\begin{minipage}{\linewidth}
    \centering
    \begin{tikzpicture}
		\node at (0,0) []{\includegraphics[width=\textwidth]{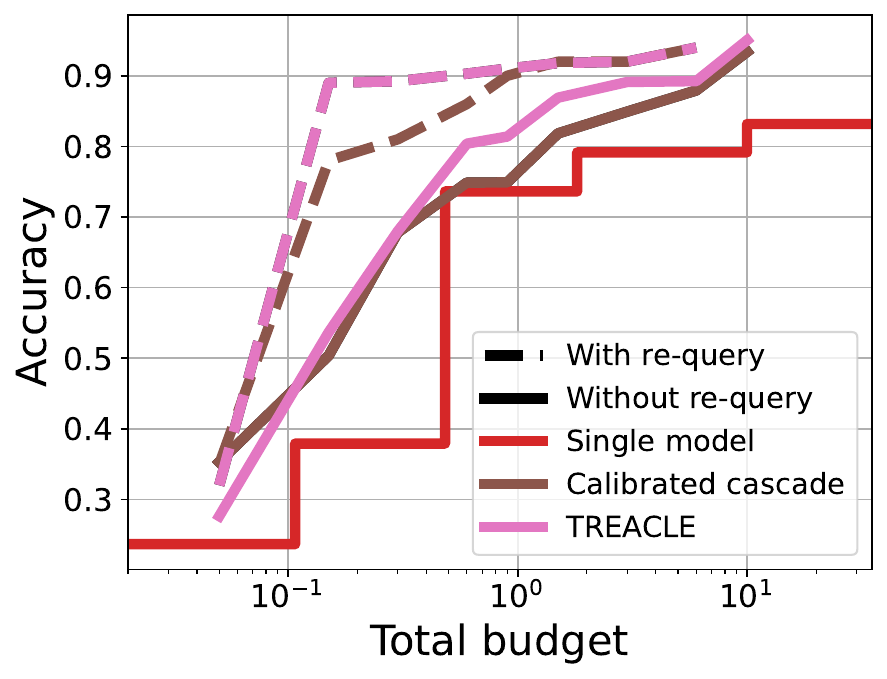}};
    \end{tikzpicture}
    \vspace{-20pt}
    \caption{With and without re-querying. $\alpha=\frac{1}{20}$.}
    \vspace{-10pt}
    \label{fig:requery}
\end{minipage}
\end{wrapfigure}
$\bullet$ \textbf{\emph{Observation 4: Re-querying helps.}} %This illustrates how re-querying with smaller models aids in tackling problems while reducing cost. This leads us to a key question: \emph{Does re-querying enhance the performance?}
%To evaluate the impact of re-querying,
%We conducted an ablation study where 
We trained both \alg and \algU baseline without the ability to re-query. The results are shown in \Cref{fig:requery}, where the dashed line represents method variants that permits re-querying. %, while the solid line represents variants that do not.
We observed a notable decrease in accuracy when re-querying was not allowed. Methods without re-querying eventually achieved comparable accuracy with those with re-querying capability, but with significantly larger budgets.% \xuechen{We did more ablation experiments showing that model-prompt pair selection also helps. The results
%Additional ablation experiments showing that re-querying or prompt selection help are shown in \Cref{app:ablation}.
%Compared to other baselines that didn't allow re-querying, the $\alg$ and Cascade methods continued to outperform their performance.

$\bullet$ \textbf{\emph{Observation 5: \alg's choice of model and prompt is impacted by relative LLM prices.
%denoted as $\alpha$, of Llama model prices to GPT model prices.
}}
As the relative cost of Llama models decreases ($\alpha$ decreases), \alg increasingly utilizes Llama to answer queries, allowing for cost savings, as shown in \Cref{fig:Budget_query_alpha}. This shift enables use of more expensive models like GPT-4 when tackling complex problems, thereby enhancing overall accuracy. When Llama becomes more expensive, \alg no longer chooses it. This aligns with our intuition that using Llama to verify response consistency becomes less economical.

\subsubsection{Addition of new LLMs}
\label{sec:new_llm}
\begin{wrapfigure}{r}{0.6\textwidth}
\centering
\vspace{-15pt}
\begin{minipage}{\linewidth}
    \centering
    \begin{tikzpicture}
		\node at (0,0) []{\includegraphics[width=1.05\textwidth]{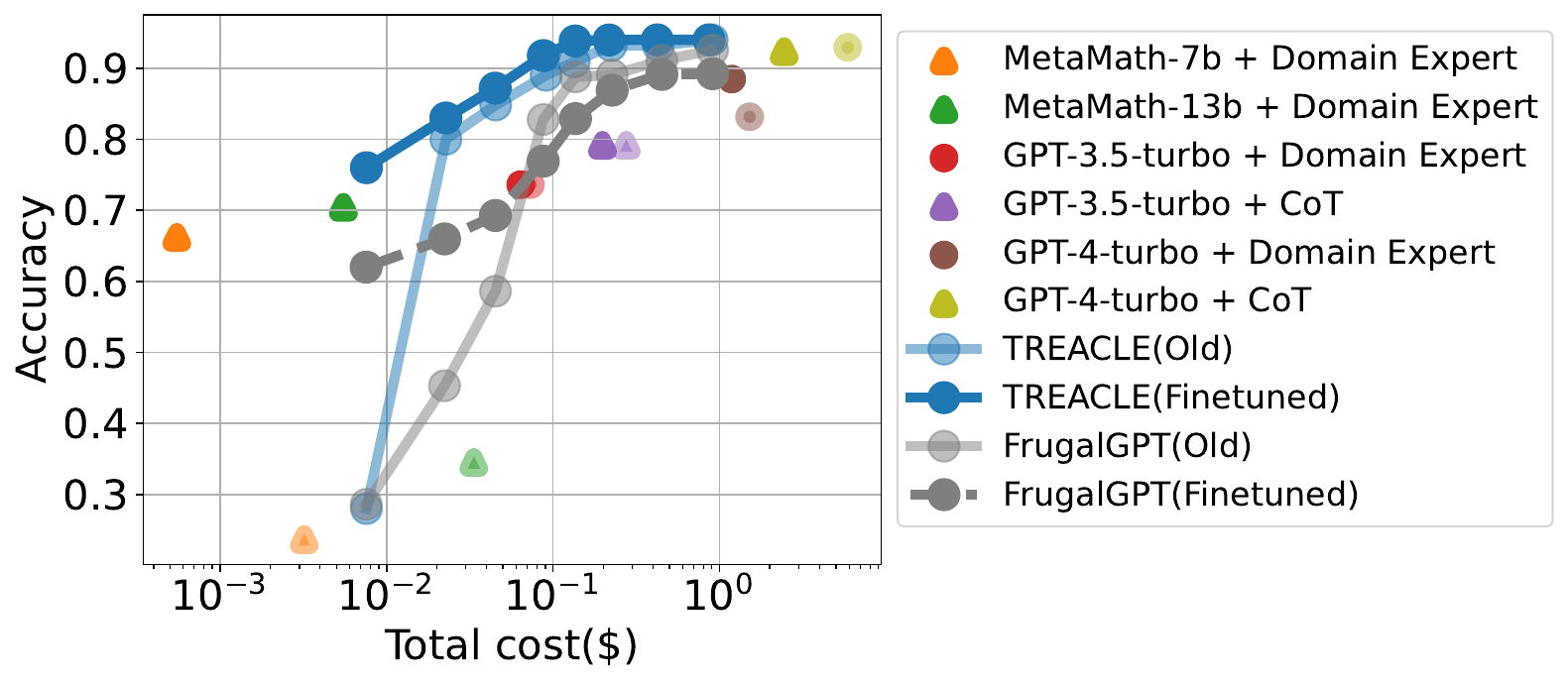}};
    \end{tikzpicture}
    \vspace{-20pt}
    \caption{Performance with new LLMs and lowered prices. 
Lines and dots in light (dark) colors are results with old (new) prices and LLMs. $\alpha=\frac{1}{10}$.}\vspace{-8pt}
    \label{fig:clean-full-new}
\end{minipage}
\end{wrapfigure}

% \iffalse
%  \begin{wrapfigure}{r}{4cm}
%  \vspace{-10pt}
% \centering
% % \begin{figure}
% % \centering
% %\vspace{-5mm}
% \begin{subfigure}{0.3\textwidth}
% \includegraphics[scale=0.28]{img_new/finetune_sample.pdf}
% \vspace{-15pt}
% \caption{With new GPT models and prices (API price adjustment)}
% \label{fig:finetune_sample_gpt}
% \end{subfigure}\\
% ~
% \begin{subfigure}{0.3\textwidth}
% \includegraphics[scale=0.28]{img_new/finetune_sample_2.pdf}
% \vspace{-15pt}
% \caption{With new Llama models (fine-tuned open-source LLMs)}
% \label{fig:finetune_sample_llama}
% \end{subfigure}
% % \begin{subfigure}{0.5\textwidth}
% % \includegraphics[scale=0.28]{img_new/legend_finetune_sample.pdf}
% % \label{fig:finetune_sample_llama}
% % \end{subfigure}
% \vspace{-15pt}
% \caption{Sample complexity for different budgets.}
% %\jiasi{Change ``DQN'' to ``CURE'' in all figs. Total cost rather than budget would be more useful for the intro discussion. Units for x-axis.}\xuechen{fixed}
% %\jiasi{Note: LLaMA is fake price, what to do?}
% \label{fig:finetune_sample}
% \vspace{-15pt}
% % \end{figure}
% \end{wrapfigure}
% %\xuechen{move earlier? I think this part is important}
% \fi 

LLM development is rapid, with better models continuously emerging, and the API prices set by providers can change at any time. \alg's ability to %replace or add model-prompt pairs, or adjust to new prices,
react to such changes is thus an important practical consideration. We show that \alg can adapt by fine-tuning itself using few samples.
 We study two types of LLM updates.
%the experiment under 3 different settings:\\
(1) \emph{API price adjustment:
%\xuechen{GPT-4-turbo is slightly worse than GPT-4 but cheaper}
} In November 2023, OpenAI released GPT-4-turbo, offering performance on par with GPT-4 but at a more affordable price. Concurrently, the price for GPT-3.5-turbo was lowered. %adjusted to \$0.0010 per 1,000 input tokens and \$0.0020 per 1,000 output tokens. 
%Accessing these models through the API tends to be more expensive but offers high performance.\\
(2) \emph{Fine-tuned open-source LLMs:} Several domain-specific fine-tuned models with higher accuracy have been released.
Specifically, we exchanged Llama-2 for MetaMath~\cite{yu2023metamath}, which is fine-tuned specifically for GSM8K.
%To make \alg domain-specific as well.
For both scenarios, we partitioned the GSM8K test data into 80\% validation and 20\% test samples, generated new state-action trajectories from the validation set, then fine-tuned \alg on these new trajectories.
To create a comparable baseline, we also fine-tuned FrugalGPT's DistilBERT. 

Firstly, we show the performance of \alg with both the API price adjustments and improved LLMs in \Cref{fig:clean-full-new}.
 The individual points on the plot illustrate the changes in the API prices for gpt-3.5-turbo.
 The lines show the performance of the new \alg with new models and prices and the old \alg (from previous subsections).
The new \alg can achieve the peak accuracy with only a \$1 budget, clearly benefiting from the new models and lowered prices. % which significantly decrease the cost.
 Benefits are also significant for lower budgets, where the improved \alg has significantly higher accuracy, because the lowest performing Llama-2 models were replaced by fine-tuned Metamaths.
%Our experiments clearly demonstrate that \alg can effortlessly benefit from these changes. 
Finally, for FrugalGPT that relies on a fine-tuned DistilBERT accuracy estimator, %directly fine-tuning DistilBERT with the additional GSM8K samples %query-response pairs generated with new models 
performance didn't improve and can even degrade due to distribution shifts and overfitting.
%\xuechen{training with trajectory or single samples}).
%  \begin{wrapfigure}{r}{4cm}
%  \vspace{-10pt}
% \centering
% % \begin{figure}
% % \centering
% %\vspace{-5mm}
% \begin{subfigure}{0.3\textwidth}
% \includegraphics[scale=0.28]{img_new/finetune_sample.pdf}
% \caption{With new GPT models and prices (API price adjustment)}
% \label{fig:finetune_sample_gpt}
% \end{subfigure}\\
% ~
% \begin{subfigure}{0.3\textwidth}
% \includegraphics[scale=0.28]{img_new/finetune_sample_2.pdf}
% \caption{With new Llama models (fine-tuned open-source LLMs)}
% \label{fig:finetune_sample_llama}
% \end{subfigure}
% % \begin{subfigure}{0.5\textwidth}
% % \includegraphics[scale=0.28]{img_new/legend_finetune_sample.pdf}
% % \label{fig:finetune_sample_llama}
% % \end{subfigure}
% \vspace{-15pt}
% \caption{Sample complexity for different budgets.}
% %\jiasi{Change ``DQN'' to ``CURE'' in all figs. Total cost rather than budget would be more useful for the intro discussion. Units for x-axis.}\xuechen{fixed}
% %\jiasi{Note: LLaMA is fake price, what to do?}
% \label{fig:finetune_sample}
% \vspace{-15pt}
% % \end{figure}
% \end{wrapfigure}
\begin{wrapfigure}{r}{0.6\textwidth}
\centering
\vspace{-5pt}
\hspace{-5pt}
\begin{subfigure}[b]{0.29\textwidth}
    \centering
		\includegraphics[width=\textwidth]{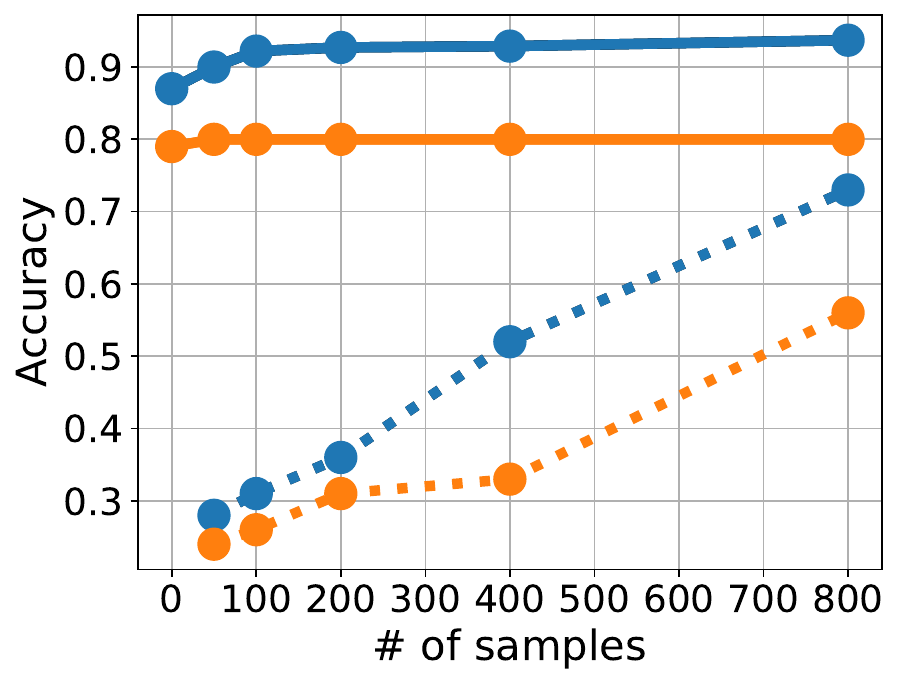}\vspace{-4pt}
\caption{New GPT models and prices (API price adjustment)}
\label{fig:finetune_sample_gpt}
\end{subfigure}
\hspace{10pt}
\begin{subfigure}[b]{0.28\textwidth}
    \centering
\includegraphics[width=\textwidth]{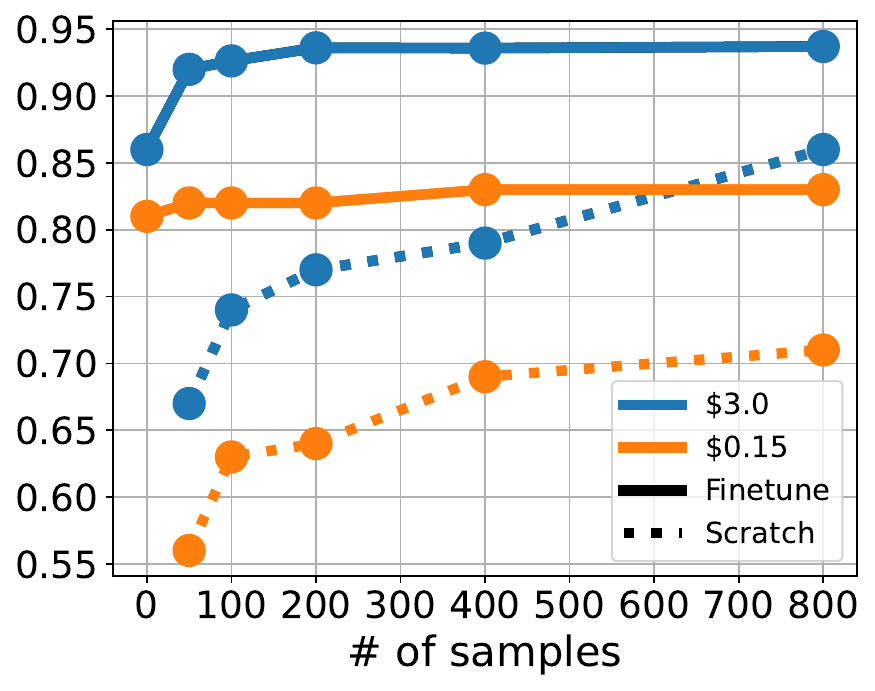}\vspace{-4pt}
\caption{\small{New Llama models (fine-tuned open-source LLMs) }}\label{fig:finetune_sample_llama}
\end{subfigure}
    \caption{Sample complexity for different budgets with new LLMs. $\alpha=\frac{1}{10}$.}\vspace{-5pt}
\label{fig:finetune_sample}
% \end{figure}
\end{wrapfigure}
 Secondly, in \Cref{fig:finetune_sample} we investigate the sample efficiency of fine-tuning the model with new API prices and LLMs (``Fine-tune'' in the figure) compared to training \alg from scratch with the new prices and LLMs (``Scratch''). The sample efficiency is important because it can be expensive to collect query-response pairs from new LLMs to further train \alg. The results indicate that when there are minor changes to the available LLMs, deploying the previously trained \alg can be sufficient.
 %directly can achieve comparable performance. 
 For instance, in \Cref{fig:finetune_sample_gpt} 
 %where model modifications occur with powerful models (GPTs) 
 when there is limited budget (\$0.15) and upgrades to the expensive models, deploying the previously trained \alg (\# samples = 0)
 %without any fine-tuning
 achieves comparable performance to the fine-tuning \alg (\# samples = 800). 
 On the other hand, when upgrades are introduced to cheaper models (\Cref{fig:finetune_sample_llama}), deploying the old \alg may initially result in poor accuracy, but \alg can quickly adapt to the new LLM options by fine-tuning with a few number of samples (around 300).

\subsubsection{Shifts in Question Difficulty}
\label{sec:domain_shift}

Thus far in the evaluations, easier and harder questions were randomly mixed throughout the training and test sets.
In practice, question difficulty may not be uniformly distributed, %, and we would like \alg to be robust to that.
so we study two types of difficulty distribution shifts: shifts across the training/test sets, and towards the end of the test set.

\begin{figure*}[]
\centering
\begin{minipage}{\textwidth}
\subcaptionbox{\small{Hard}\vspace{-8pt}}{
\includegraphics[width=0.26\textwidth]{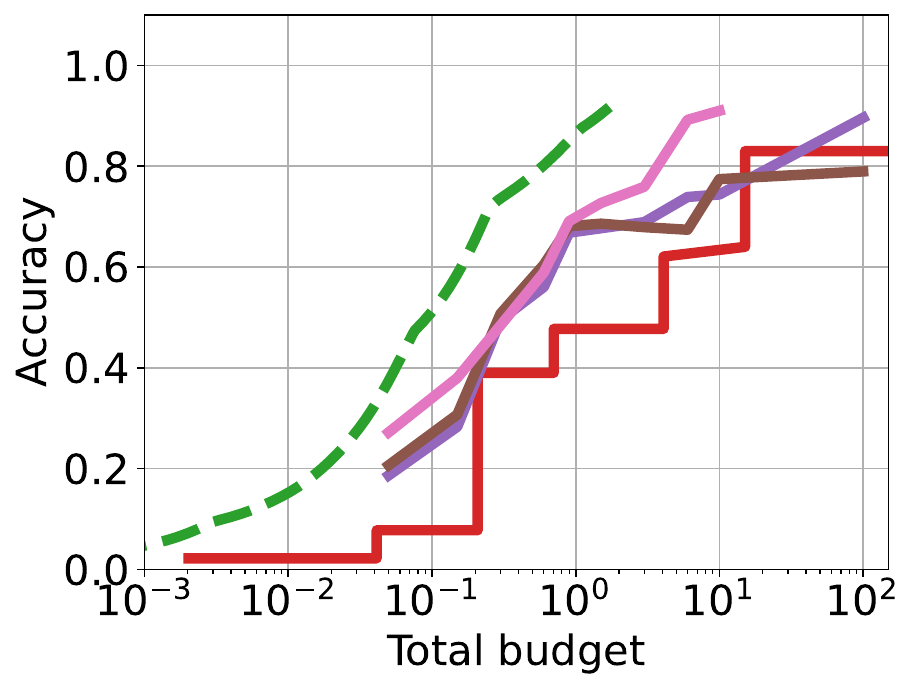}
\vspace{-4pt}
}
\subcaptionbox{\small{Easy}\vspace{-8pt}}{
\includegraphics[width=0.25\textwidth]{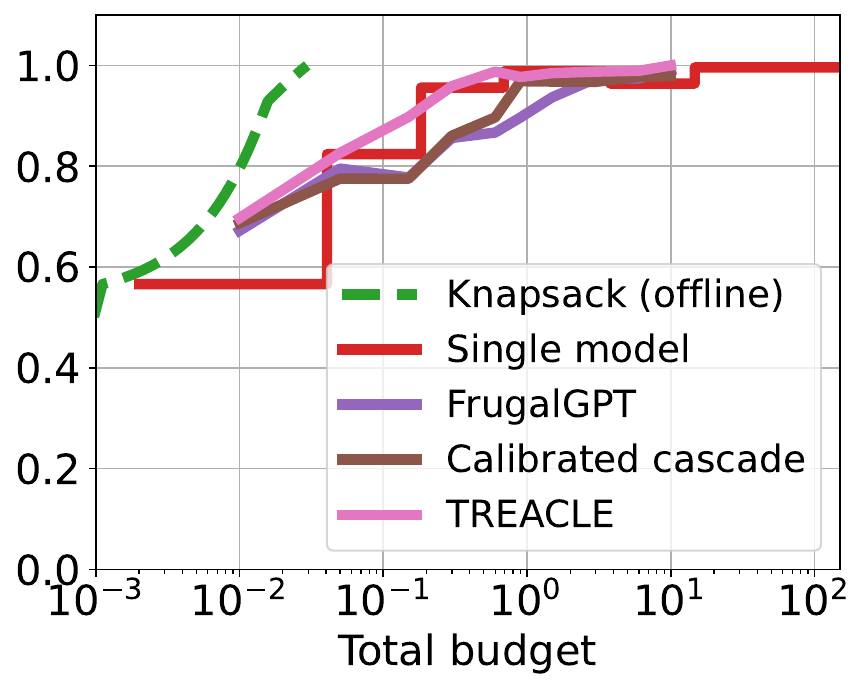}
\vspace{-4pt}
}
\subcaptionbox{\small{Statistics for budget = \$1.\vspace{-4pt}\label{table:easy_hard_stats}}}{
\tiny
\vspace{10pt}
\resizebox{0.46\textwidth}{!}{%
\begin{tabular}{|l|lll|lll|}
\hline
\multirow{2}{*}{} & \multicolumn{3}{p{1.5cm}|}{\# of questions unanswered}                         & \multicolumn{3}{l|}{Budget spent (\$)}                         \\ \cline{2-7} 
                  & \multicolumn{1}{p{0.5cm}|}{All (/1319)} & \multicolumn{1}{p{0.5cm}|}{Easy (/500)}   & \multicolumn{1}{p{0.5cm}|}{Hard (/500)}& \multicolumn{1}{p{0.5cm}|}{All} & \multicolumn{1}{l|}{Easy}  & Hard \\ \hline
{\sf \scriptsize TREACLE}               & \multicolumn{1}{l|}{1} & \multicolumn{1}{l|}{0} & 2  & \multicolumn{1}{l|}{0.97}  & \multicolumn{1}{l|}{0.84} & 0.99 \\ \hline
Single Model      & \multicolumn{1}{l|}{2} & \multicolumn{1}{l|}{0} & 0  & \multicolumn{1}{l|}{0.69}  & \multicolumn{1}{l|}{0.69} & 0.70 \\ \hline
Cal. Cascade           & \multicolumn{1}{l|}{2} & \multicolumn{1}{l|}{0} & 26 & \multicolumn{1}{l|}{0.931}  & \multicolumn{1}{l|}{0.702} & 0.99 \\ \hline
FrugalGPT         & \multicolumn{1}{l|}{3} & \multicolumn{1}{l|}{0} & 31 & \multicolumn{1}{l|}{0.96}  & \multicolumn{1}{l|}{0.713} & 0.99 \\ \hline
\end{tabular}
}
}
\caption{Performance on ``easy'' and ``hard'' partitions of the test set. Models are trained on original training data, but must handle a distribution shift in difficulty during test. 
$\alpha=\frac{1}{20}$.}
\label{fig:diff}
\end{minipage}
\vspace{-0.2in}
\end{figure*}

\textbf{Difficulty shifts between training and test.}
%\xuechen{highlight we don't know the diff distribution for all questions. online?}In this set of experiments, we sought to examine the robustness of \alg to %conducted an ablation study about the influence of 
%query difficulty distribution shift between the training and testing sets.
%We used the GSM8K training and testing data on the performance of different methods. In detail, 
We divided the GSM8K test set into ``hard'' and ``easy'' subsets based on the question difficulty.
The difficulty is defined by the number of LLM models correctly answering the questions (more models answering a question correctly roughly means it is easier). 
Basic performance on the easy and hard questions is shown in \Cref{table:easy_hard_stats}.
%\jiasi{Anything we want to point out about the table results?} \xuechen{following conclusion is observed from the table: 
When the questions are hard, each question ends up consuming too much budget, leaving insufficient budget for subsequent questions that then go unanswered. %Conversely, for the easy test set, the model allocates too little budget to each query, resulting in lower accuracy and wasted budget.
The single model baseline does well in terms of cost and unanswered questions, but has low accuracy.
%Single model seems to be doing quite well, lower cost and same number of unanswered questions as \alg...}
We plot the performance for variable budgets in \Cref{fig:diff}, and find that \alg's accuracy remains stable, no matter whether the test distribution shifts to an easier level or a harder level.
%, while \algU and \ABCascade are greatly affected by query difficulty.
This is because \alg can dynamically adjust based on the remaining budget in online fashion.
%In contrast, the $p_\text{low}$ and $p_\text{high}$ values in \algU are determined through a grid search on a validation dataset, so when the distributions of the validation and test datasets differ, the $p_\text{low}$ and $p_\text{high}$ values become outdated, significantly hurting performance. In the scenario where the test queries are hard, each query ends up consuming too much cost, leaving insufficient budget for subsequent queries. Conversely, for the easy test set, the model allocates too little budget to each query, resulting in lower accuracy and wasted budget. 
% This is because \alg and Calibrated Cascade evaluate the difficulty of questions so that they can use suitable amount of cost to solve each question regardless of the question hardness, compared with the conservative online Knapsack algorithm, which will waste much budget if there are too many difficult questions. 

%\subsubsection{Is $\alg$ robust to question reorderings?}
%\label{sec:reorder}

%JC edit here
%\begin{wrapfigure}{r}{0.4\textwidth}
%\centering
%\begin{minipage}{\linewidth}
%    \centering
%  \includegraphics[width=\textwidth]{img_new/reorder_all_small.pdf}
%  \vspace{-5pt}
%\caption{
%\alg is robust to re-ordered question difficulty in the test set. $\alpha=\frac{1}{20}$.}
%\label{fig:reorder}
%\end{minipage}
%\vspace{-15pt}
%\end{wrapfigure}

\textbf{Difficulty shifts within the test set.}
To further evaluate the robustness to question difficulty shifts, we test \alg with the full test set sorted from easy-to-hard queries or hard-to-easy queries.
The hope is that
%and compare their performance with the online Knapsack algorithm as a baseline. 
with the help of query text embedding in the state vector (which should capture some estimate of difficulty), \alg can remain relatively stable in terms of accuracy even if the ordering of the questions changes.
This hypothesis is borne out in \Cref{fig:reorder}, while \OnlineKnapsack performs significantly worse than \alg if the questions are sorted from hard to easy. This is because much of the budget is wasted on the difficult queries that arrive at the beginning.

\subsubsection{Different types of reasoning tasks}
\label{sec:datasets}

%\jiasi{possibly move to appendix since we are not showing any algorithm results, more like characterizing the datasts}

%We conducted experiments on various types of reasoning datasets to understand \alg's performance. %, with full results given in the Appendix.

\begin{figure}
\centering
\begin{minipage}{0.33\textwidth}
 \centering
   \includegraphics[width=\textwidth]{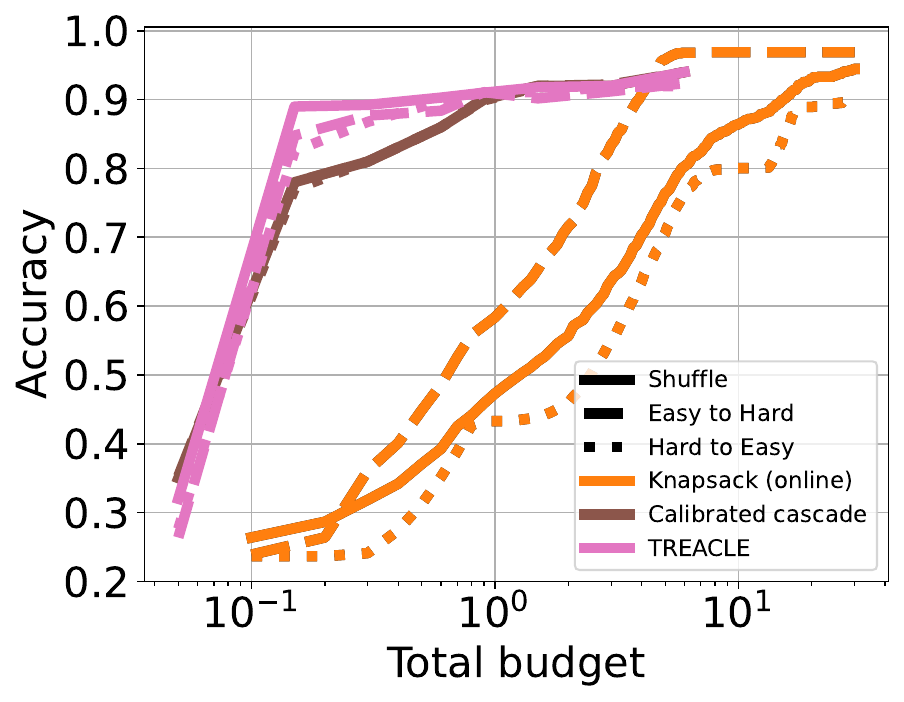}
 \caption{
 \alg is robust to re-ordered question difficulty in the test set. $\alpha=\frac{1}{20}$.}
 \label{fig:reorder}
 \end{minipage}
\begin{minipage}{0.66\textwidth}
\addtocounter{figure}{+1}
\begin{subfigure}{0.5\textwidth}
\includegraphics[width=\textwidth]{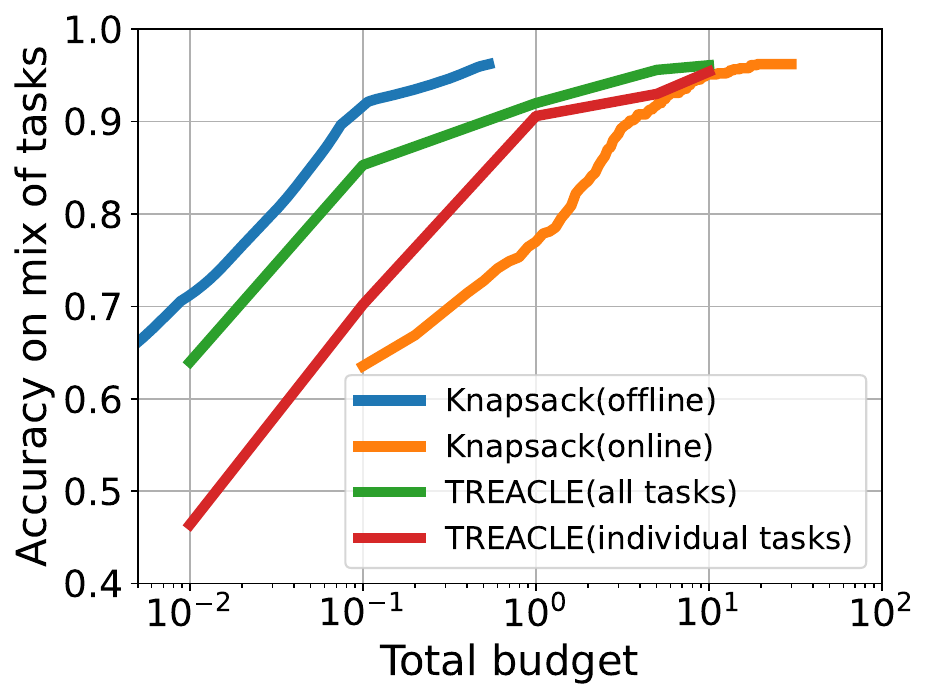}
\caption{\small{Mixture of tasks}}
\label{fig:mix}
\end{subfigure}
\begin{subfigure}{0.49\textwidth}
\includegraphics[width=\textwidth]{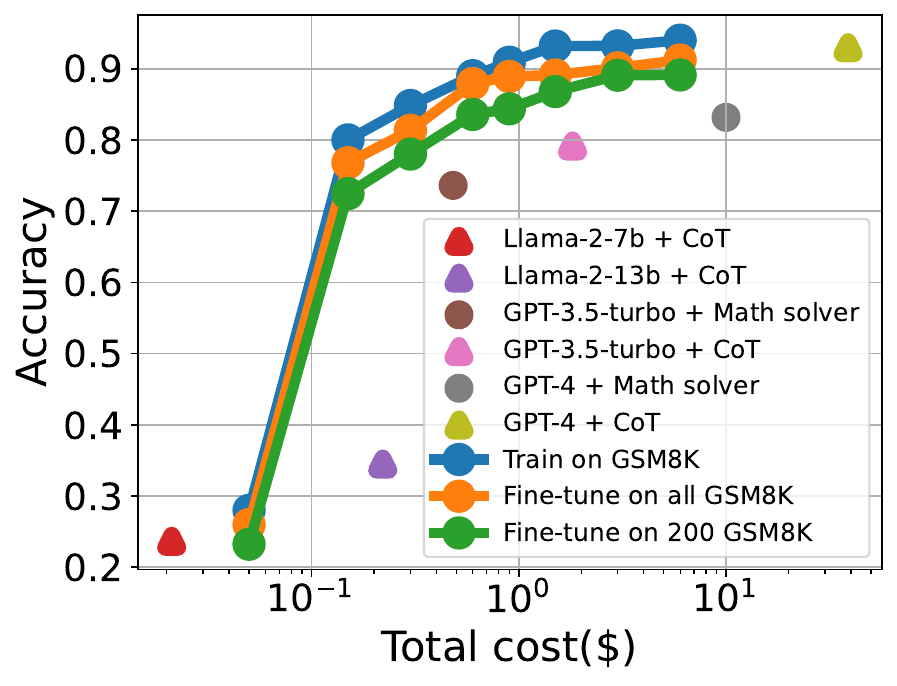}
\caption{Addition of new task}
\label{fig:finetune}
\end{subfigure}
\addtocounter{figure}{-1}
\caption{Performance with different types of reasoning tasks.}
\end{minipage}
\vspace{-0.2in}
\end{figure}

\textbf{Mixture of tasks.} We seek to examine whether one model can handle multiple types of tasks under one common budget (in contrast to the previous experiments with a specialized model for each task).
We trained a single model with all 3 datasets and recorded the test accuracy on those datasets.
%also train a single model to handle multi tasks while ensuring that all tasks conform to a unified budget. 
The results shown in \Cref{fig:mix} for ``\alg (all tasks)'', offline knapsack, and online knapsack are the test accuracy from an equal mix of 
%Above, the first three lines are results on the dataset obtained by equally-mixing
CSQA, GSM8K and LLC queries.
``\alg (individual tasks)'' is the test accuracy on the same mix of queries, using the models from previous subsections,
%The last line is the average accuracy resulting from three policies trained on individual datasets
where each model (corresponding to a task) is assigned to 1/3 of the common budget.
``\alg (all tasks)'' can handle a mixture of tasks under a common budget (\eg outperforming online knapsack), and can significantly outperform the individual tasks baseline (``\alg (individual tasks)'') by effectively allocating its common budget across queries of different types.

% \iffalse
% \begin{wrapfigure}{r}{4cm}
%  \vspace{-10pt}
% \centering
% % \begin{figure}[]
% \centering
% \begin{subfigure}[b]{0.2\textwidth}
%     \centering
% 	\begin{tikzpicture}
% 		\node at (0,0) [scale=0.25]{\includegraphics{img_new/mix.pdf}};
%     % \node at (-0.5,-1.3) [scale=0.8] {GSM8K};
% \end{tikzpicture}
% \end{subfigure}
% % \begin{subfigure}[b]{0.4\textwidth}
% %     \centering
% % 	\begin{tikzpicture}
% % 		\node at (0,0) [scale=0.3]{\includegraphics{img_new/legend_mix.pdf}};
% %   % \node at (0,-2) [scale=0.8] {10 times};
% % \end{tikzpicture}
% % \end{subfigure}
% \caption{\alg can handle a mixture of tasks under a common budget.}
% \label{fig:mix}
% % \end{figure}
% \vspace{-10pt}
% \end{wrapfigure}
% %\vspace{-0.2cm}
% %\subsubsection{Addition of new dataset}
% \fi

\textbf{New unseen task type.} 
Consider the scenario where the model has not been trained on certain new tasks.
To show that \alg can adapt to new tasks easily, we performed additional experiments.
The base model is trained using the CSQA dataset, and the unseen new tasks are queries from GSM8K.
%whose results are shown in \Cref{fig:finetune}. 
Interestingly, in our design, we decouple decision making from the task embedding as follows. To transfer from CSQA to GSM8K, we freeze the base RL policy of CSQA (the decision making part), and fine-tune the ``text embedding'' feature in the state vector (see bottom pink part of \Cref{fig:overview}).
%(the task embedding part) \jiasi{what is the task embedding part?}
%that feeds to this RL policy. 
% This is illustrated in \Cref{fig:overview_2}.
How many samples are needed to fine-tune the text embedding for the new task?
As shown in \Cref{fig:finetune}, with a budget of 0.6, the original model fully-trained on GSM8K (``train on GSM8K'') achieves a test accuracy of 0.848, compared to 0.78 when trained on CSQA and fine-tuned with only 200 additional samples from GSM8K (``fine-tune on 200 GSM8K''). This highlights a relatively small accuracy loss when transferring to new types of unseen tasks.
The results suggest that our method can easily adapt to new tasks with only a small amount of additional training. %, which is important for overall cost savings.

\label{sec:new_dataset}

\vspace{-10pt}
\section{Conclusions}
\label{sec:conclusions}

\vspace{-5pt}
We propose \alg, a learning-based LLM querying framework that intelligently chooses between LLM and prompt combinations based on question context and past response history. 
Our experiments show that \alg outperforms other baselines and is robust to different budgets, LLM availability and prices, and so on.
%and even reach close to the optimal single query solution.
%Through the development of \alg, we also study the characteristic of LLMs (accuracy, monetary cost and latency) and how model re-querying can help to quantify uncertainty, which can help other researchers understand and explore LLMs. 
%As a future direction, we think that there is still a space to improve the cost-efficient LLM performance by cascading them, as more efficient LLM and prompt choices will result in more efficient and accurate policies. Moreover, reasoning is one of many domains to explore cost-efficient querying policies. Numerous other tasks exist, such as code writing and creative text generation, which have different intrinsic nature than reasioning tasks. On top of that, while our work only considers monetary cost and latency cost, there are some other aspects in LLM systems that can be evaluated as part of the general cost function, such as robustness and privacy.
For future work, we plan to incorporate % other non-reasoning tasks and incorporate 
other features such as privacy into the cost function.
We hope our framework can help spur cost-efficient utilization of LLM systems.
% \newpage
% \newpage
% \clearpage

\textbf{Limitations.} Our work focuses on reasoning problems, and could be extended to generative problems by incorporating new measures of response consistency.
%modifying response consistency to consider the similarity of the generated responses.
The RL policy's budget does not account for the cost of collecting the training data. We plan to freely release the datasets and code so that others can train the same basic policy and adapt that policy to new future LLMs and tasks (as shown through our experiments). 
%or ``assigning scores to the generations by a judge model''.

\textbf{Broader impact.}
This work can make LLMs more accessible to cost-sensitive users.

\subsection*{Acknowledgements}
This work was supported in part by an Adobe Data Science Research award, NSF CCF-2046816, a gift from Google Research, and credits from the Microsoft Accelerating Foundation Models Research grant program. Thank you to Dr.~Koyel Mukherjee for helpful discussions and insights on this work.
\newpage
\bibliography{refs}
\bibliographystyle{plain}
%\newpage
\appendix
\section*{Appendix}

\section{Experiment Setup and Implementation Details}
\label{app:implement}

\subsection{Datasets} 

We use three representative datasets for the experiments.
%GSM8K %\cite{cobbe2021training} for mathematical reasoning and CSQA %\cite{saha2018complex} for commonsense reasoning. 
\squishlist
    \item \textbf{GSM8K \cite{cobbe2021training}}: The Grade School Math 8K dataset contains 8.5K high quality grade school math problems created by human writers, in which 7.5K are in the training data and 1K are in the testing data. We further split the 7.5K training data into 6K training data and 1.5K validation data. %for calibration. 
    \item \textbf{CSQA \cite{saha2018complex}}: The Complex Sequential Question Answering dataset consists of 12102 multiple choice commonsense reasoning questions encountered in daily life. The training set, validation set, and testing set contain 9741, 1221 and 1140 samples respectively. 
    \item \textbf{LLC~\cite{wei2022chain}} 
    %\jiasi{add description} \xuechen{Following \cite{wei2022chain}, we employ the symbolic reasoning task Last letter concatenation(LLC). 
    The Last Letter Concatenation task is to concatenate the last letters of words in a name (e.g., ``Amy Brown'' $\rightarrow$ ``yn''). 
    %We use both the in-domain test set, which includes two-word names as the few-shot exemplars and the test, and the out-of-domain test set, which contains three and four-word names in the test set but not the exemplars.

\squishend

 \begin{wrapfigure}{r}{6cm}
 \vspace{-10pt}
% \vspace{-0.1in}
  \includegraphics[width=0.39\textwidth]{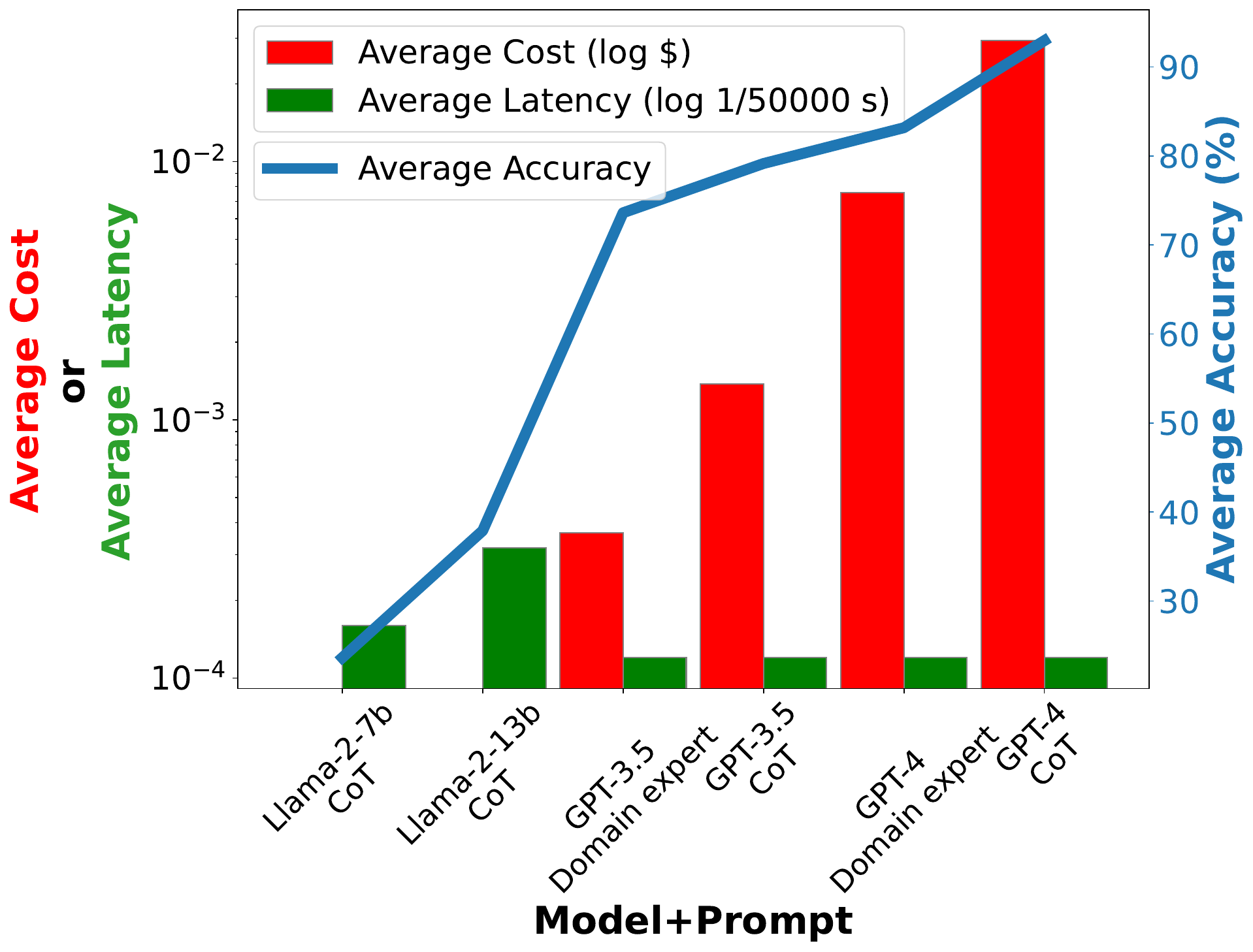}
 \caption{\small{Characterizing accuracy, cost, latency of different model-prompt pairs $\LMPP$ on the GSM8K test dataset. Higher accuracy corresponds to higher price or and lower latency.}} % \jiasi{thanks for the update. larger font in all axis subtitles please} \xuechen{not mentioned in main text?}}
\label{fig:model_tradeoffs}
%\end{subfigure}
%\vspace{-15pt}
%\caption{}
\vspace{-20pt}
%\label{fig:framework}
% \vspace{-0.6cm}
\end{wrapfigure}

% \vspace{-0.2cm}
\subsection{Data collection and training}
\label{sec:data_collection_training}

%To generate the trajectory of model outputs, we split the data generation pipeline into two phases: 
To evaluate our methods, we perform two steps: (1) Collect query-response pairs for different combinations of LLMs and prompt, then (2) train \alg with these pairs.
%(Section \ref{subsubsec:state-action pair collection})
%and trajectory generation.
%(Section \ref{subsubsec:trajectory generation}). 

\paragraph{(1) Collecting query-response pairs.}
%\label{subsubsec:state-action pair collection}
We collect query-response pairs from each dataset for different combinations % $(M_i,P_j,T_k)$ 
of LLM, prompt, and LLM temperature.
%consistently choosing the next token that has the highest probability.
%model, prompting strategies, temperature, and various details of each configuration
%Selected results of 
The accuracy, latency, and monetary price of the best combinations are shown in \Cref{fig:model_tradeoffs}, with full results in \Cref{tab: model performance} in the Appendix.
%\Cref{tab: model performance} \xzC{move to appendix?}.
%Specifically, we selected the (LLM, prompt) combinations at the Pareto frontier of accuracy and cost, and made them available to \alg.
We selected those combinations according to \Cref{prop:llm_ordering}.
%balancing reasoning accuracy with respective monetary costs. 
%This approach showed that
%Generally, larger models and more sophisticated prompts significantly improve accuracy, but with a higher financial cost.
%The best-performing combinations were selected for use in \alg, and are shown in 

\vspace{-0.1cm}
\emph{\textbf{LLMs.}}
    %To demonstrate the generalization capabilities of our policy to different LLMs,
    We used 5 different LLMs: Llama-2-7b-chat, Llama-2-13b-chat \cite{touvron2023Llama}, GPT-3.5-turbo, GPT-4, and GPT-4-turbo~\cite{OpenAI2023GPT4TR}. These models are of varying sizes (7b, 13b, 154b and 1.76t respectively).
    The Llama models are open-source and run locally on our servers (one A40 GPU for Llama-2-7b and two A40 for Llama-2-13b), 
    %\xuechen{I think we use 1 A40 for 7b and 2 A40 for 13b. @zijian, please double check} \zijian{checked. 1 A40 for 7b and 2 A40 for 13b} 
    while the GPT models rely on commercial APIs. 
    %\jiasi{add details about servers that Llama ran on, since it impacts our latency results.}
    %\xuechen{details in the original version: Our approach can also incorporate additional models such as Llama-2-70b-chat to fill the scale gap between smaller Llamas and GPTs. We opted to not include the larger Llama2 as it incurs high latency without inference optimization (16s/query for plain question query and 28s/query for CoT query in 5 A40 GPUs in our experiments). }
    %\jiasi{I removed this part about the larger Llama model since it might invite more questions, and does not seem like a very important point.}

    \vspace{-0.1cm}
    \emph{\textbf{Prompt types.}} We employ several prompting schemes to elicit the reasoning abilities of LLMs. 
    %The full prompts are shown in \Cref{app:full_prompts}.
    %\jiasi{Explain what is ``system message'' and ``user's content message''} 
    A prompt generally consists of two parts: the ``content message'' containing the question, and the ``system message'' with additional context. %The full prompts are given in \Cref{app:full_prompts}. % given to an LLM. %the roles they are or the environments they are in.
    
    \squishlist
         \item The \textbf{plain text prompt} 
         %For this easiest prompting strategy, we 
         submits the questions to the LLM as the content message (no system message). %, which is shown in 
        % \item The \textbf{``Math Solver'' prompt} \xuechen{explain in general. domain expert?} feeds the message ``You are a math solver. Give the answer to the following question in the boxed \{\}'' as a system message, while keeping the user's content message as a plain text, as shown in \Cref{tab:GSM8K Math Solver prompt}. 
        % This strategy is used in the GSM8K experiments \jiasi{why not CSQA?} \zijian{Because GSM8K is a math datset and this prompt is a special prompt designed for it.}, and is especially useful for GPT models \jiasi{what does ``useful'' mean? higher accuracy for slightly more cost, compared to standard prompt?} \zijian{Yes.}. 
        \item  The \textbf{domain expert prompt} feeds information about the question's domain as a system message (\eg ``math solver''), and keeping the user's content message as plain text. %The full prompts are shown in \Cref{tab:GSM8K Math Solver prompt} in the Appendix.
        
        \item The \textbf{standard few-shot prompt}  includes a system message (``Follow the given examples and answer the question'' \cite{wei2022chain}) and the content message, which consists of few-shot examples together with the plain text prompt. %(see \Cref{tab:CSQA Standard prompt} in the Appendix).
        It tends to improve response accuracy compared to the plain text prompt. %, and even outperforms CoT in the CSQA dataset.

        \item The \textbf{Chain-of-Thought (CoT) few-shot prompt}~\cite{wei2022chain} %Compared to the standard prompt, CoT prompt 
        adds some intermediate explanations to the few-shot examples. %(see \Cref{tab:GSM8K CoT prompt,tab:CSQA CoT prompt} in the Appendix).
        %CoT techniques are widely used in reasoning tasks due to their mimicking of human thought processes.
        %Note that although some other variants of CoT such as complex CoT \jiasi{ref} \zijian{It is in a github repo} can outperform the original CoT prompt, but we do not employ them because they need much longer prompts and accordingly cost substantially more, and we observed that the \alg essentially did not make use of them due to their inefficiency. 
        %The cost is often more than the current model-bigger model price difference, effectively making them infeasible for our applications.
    \squishend
    \emph{\textbf{Temperature.}}  The LLM temperature is a configurable parameter that influences the variety of the responses it generates.
    %With a lower temperature, the model's outputs become more predictable, while 
    With a higher temperature, the model may output more diverse but possibly inaccurate responses.
    %The LLM model temperature is a configurable parameter, and 
    We set the temperature to 0 for a new query, and to 0.8 or 1.0 for a re-query for Llama and GPT, respectively. 
    %and set Llama's temperature to 0.8 and GPT's temperature to 1.0 while doing requery.
    
    %our method can be extended to more different temperature values.
% Presently, data collection involves gathering outcomes for the first 200 inquiries from the training set of gsm8k. Each individual question is subjected to five re-queries. The initial deployment encompasses three models: GPT3.5 and Llama 13B.

\paragraph{(2) Training \alg.}
%\label{subsubsec:trajectory generation}
%We designed our \alg framework as described in \S\ref{sec:rl} and 
We used Deep Q-Network (DQN)~\cite{mnih2015human} to train the reinforcement learning (RL) policy in \alg, consisting of a two-layer neural network.
To generate diverse trajectories consisting of $(s_t,a_t,r_t,s_{t+1})$, we use the collected query-response data and employ $\epsilon$-greedy exploration.
%for \alg and \algU
%We sample a new action with probability $\epsilon$, or choose the action according to the current policy with probability $1-\epsilon$. As the training progresses, we decrease the value of $\epsilon$ slowly, resulting in a stable policy. \carlee{is this just $\epsilon$-greedy exploration? We can call it that and leave out the details to save some space}
For the monetary prices, we use the published per-token prices for the GPT models. 
Since our local Llama deployments do not have API costs, we set the Llama-2-7b's price as $\alpha$ times Llama-2-13b's price, and Llama-2-13b's price as $\alpha$ times
%carlee{why $\alpha^2$? To explore variations in prices?} or $\alpha$ fraction of the 
GPT-3.5-turbo's price.
$\alpha$ varies between $\frac{1}{10},\frac{1}{20}$ or $\frac{1}{50}$.
Our pricing is grounded in reality and similar to actual market rates, as the offered price for Llama is approximately 15\% of GPT-3.5-turbo according to current providers~\cite{togetherpricing}. 
%A summary of the average cost per query of different models and prompt combinations are shown in \Cref{tab:single model cost} in the Appendix.
%Considering their profit margin, deploy the models locally could be even cheaper.
%To set the monetary price, we use the published per-token prices for the GPT models.   Since Llama runs locally and there is no API cost, to approximate the resource consumption costs, we set the Llama monetary prices to a fraction $\alpha$ of the GPT prices. %, in order to study \alg's performance for a range of prices.
For the latency-accuracy tradeoff, we evaluate different trade-off parameters $\beta=[50\text{k},500\text{k},1\text{M}]$ in the cost function.
%, $\cost(\cdot) = 1/\beta \text{ latency} + \text{monetary cost}$.
We set the smallest $\beta$ to 50k because then the two terms in the cost function have similar order of magnitude. %average monetary cost and the average latency cost is similar. 
%The final values associated with each model under these two cost functions are shown in \Cref{tab:single model cost}.
The details of the cost function values are provided in \Cref{tab:single model cost} in the Appendix.
Unless stated otherwise, we set $\lambda = 5$.

 During training, we used the Adam optimizer with a learning rate $1\times10^-4$, Huber loss as the loss function, and a batch size of 64. Our DQN has three layers with ReLU and softmax activations, and the size of the hidden layer is 128. We set $\lambda=5$ in the reward function.
 For re-queries, we set different temperature settings for Llama and GPT (0.8 and 1, respectively) because their ranges 
    %of their temperature values
    are different ($[0,1]$ and $[0,2]$ respectively).
The actions are selected according to an $\epsilon$-greedy policy. Simply put, the actions are sometimes chosen by the DQN and sometimes sampled uniformly. The probability of choosing a random action starts at $\varepsilon_\text{START} = 0.9$ and decays exponentially towards $\varepsilon_\text{END} = 0.05$. For the reward decay, we use $\gamma=0.99$.
\subsection{Baselines}
\label{sec:baselines}

%To show the effectiveness of our method in terms of the trade-off between cost-efficiency and answer reliability, 
% \vspace{-0.2cm}
We evaluated the following baseline methods, reproducing the methods as faithfully as possible with a common set of LLMs and prompt options.
%\carlee{These baselines are explained well, but I think it may be more useful to point out what comparing to each of them shows. That might also allow us to condense the explanations of results} \xzC{reproducing recent work with our model-prompt options}
%\ege{we can add that these baselines measure consistency actually}: 
\squishlist
    \item \textbf{FrugalGPT}~\cite{chen2023frugalgpt}. We reproduce FrugalGPT, which uses a DistilBERT model~\cite{sanh2019distilbert} to estimate the response accuracy. If this estimate is below a threshold, the next LLM in the cascade is queried. %\xuechen{estimate the accuracy with and without previous response(re-query)}
    This baseline shows how \alg compares to the state-of-the-art that lacks re-querying.
    %\jiasi{for response accuracy? confidence?} \xuechen{They use 'reliability score' in the paper. I think predicted accuracy is the same as confidence and reliability score.}. 
    %The threshold for the scoring function is fine-tuned based on the budget \jiasi{using the validation set?}. \xuechen{yes. Not fine-tuned, grid search}
    \item \textbf{Calibrated cascade.}
We build on FrugalGPT's response accuracy estimation and develop a 2-layer neural network, whose input is a state vector to \alg and whose output is the estimated response accuracy.
%develop and only return responses that are above a confidence threshold. 
%The model 
%the consistency of query answers, the number of queries, Output statistics of queries, Input/Output length and Text embedding of the queries. 
 If this estimate is below a threshold (tuned on the validation set), the next LLM in the cascade is queried. 
%The threshold is tuned on a validation set based on the total budget.
This baseline shows how \alg compares to an improved version of FrugalGPT.
%\xuechen{hard-coded criterion, set a confidence threshold, the threshold is determined by budget}

    \item \textbf{\BMajority}. 
    %\xuechen{rename? MoT} \ege{just "Majority Voting"?} 
   % We conduct experiments based on the majority voting method from these works, which 
% Majority voting allows each LLM to be queried up to $N$ times \jiasi{\ABCascade also uses $B$. I changed to $N$} \xuechen{I think $B$ has the same meaning in \ABCascade and \BMajority. Both are query times}. This follows the intuition that more common answers are more likely to be correct.
    For each query, we output the final response based on the majority vote from $N$ re-queries, based on \cite{wang2022self,yue2023large}.
    We set $N=2$ based on the best empirical results.
    %(or the most recent response if there is no clear majority). 
    The (LLM, prompt) combinations are progressively queried until their per-question budget 
    %(total budget divided by number of remaining questions)
    runs out.
    This baseline allows comparison with \alg's response consistency feature in the state vector.
    %sorted from the cheapest model to the most expensive model (equivalently, from the highest to lowest accuracy). Use the same model sequence as Consistency Cascade, \alg and single model.} 
    % larger $N$ results in significant increase in cost but limited improvement in accuracy. 
    %\jiasi{how does \BMajority use the following budget info?} \xuechen{decide query times}
     %We calculate the average cost for each query time, and determine the query times for each question according to the average remaining budget, which is calculated by dividing the remaining budget with the number of remaining questions. \carlee{is query time the same as latency? Does this method consider monetary cost?} 
    %\xuechen{Our \alg decide the final response combining several features instead of simply majority voting based on responses. (majority voting v.s. calibrated confidence?)}

        \item \textbf{Offline and online knapsack}. Given the cost of LLM responses and their accuracy, we formulate a multiple choice knapsack problem where the items are the $\LMPP$ combinations, the values are the correctness probabilities, and the costs are the latency and monetary price functions. 
        %We also assume that the number of queries is known. %It constitutes the most relaxed version of our problem setting and 
    Solving this offline knapsack problem gives the optimal solution when re-queries are not allowed. 
    %\jiasi{did we do knapsack with re-queries or not?} Note that we are able to apply the offline Knapsack algorithm to requery-allowed setting, but the latency of the algorithm scales exponentially since we need to calculate all possible requeries and sort them in terms of price. 
    %\item \textbf{\OnlineKnapsack}~\cite{chakrabarty2008online}.
    %Following the same assumptions as the \OfflineKnapsack, except without knowing the total number of queries, results 
    We also implement an online approximation algorithm~\cite{chakrabarty2008online}.
    These baselines show how \alg compares to methods with perfect knowledge of question costs and accuracy. 

    \item \textbf{Single model.} 
The (LLM, prompt) combinations are sorted by increasing cost and accuracy, then the most capable option that fits within the allocated budget is selected for all questions.
This baseline shows how \alg compares to a fixed single LLM and prompt.
%When deploying, we use only one model, specifically the most capable one that fits within our allocated budget. The price of each model is estimated based on the cost of training data.
\squishend

\section{Theoretical Justifications for the Cascade Strategy} \label{app:theory}

\subsection{How should the $\LMPP$ be ordered in the cascade?}
\label{subsec:cascade-ordering}
\noindent\textbf{Setup 1:} Suppose there are $K$ LLM-prompt pairs each with probability of correct answer $p_i$ and inference cost $c_i$ for $1\leq i\leq K$. During the cascade, we assume access to an oracle that tells when the answer of a model is incorrect so that we can move to the next model. The procedure continues until a correct answer is obtained. The goal is to minimize the expected cost of inference.
%\jiasi{doesn't align with RL prob formulation where cost is a constaint, not objective} \zijian{explain in the main paper}

\begin{lemma} Suppose there are 2 models with probability of correct answers $p_1$ and $p_2$ and inference costs $c_1$ and $c_2$, respectively. Then, the optimal cascade rule is to first query the model with larger cost-normalized accuracy $\frac{p_i}{c_i}$.
\end{lemma}
\begin{proof} The cascade terminates when the first correct answer is obtained. The expected inference cost if we query model 1 first is $C_1 = c_1p_1 + (c_1+c_2)(1-p_1)$, and the expected inference cost if we query model 2 first is $C_2 = c_2 p_2 + (c_1 + c_2) (1-p_2)$. Note that the accuracy is independent of the cascade order thanks to the oracle. 
It can be easily shown that $C_1 \leq C_2 \iff \frac{p_1}{c_1} \leq \frac{p_2}{c_2}$.
%From this, we wish to minimize $c_ip_i + (c_1+c_2)(1-p_i)=c_1+c_2-c_1c_2p_i/c_i$. This is minimized when $p_i/c_i$ is larger.
\end{proof}

% \begin{repproposition}{prop:llm_ordering} Suppose there are $K$ LLM-prompt pairs with probability of correct answer and cost pairs $(p_i,c_i)$ and suppose their outputs are independent of each other. Then, the optimal cascade rule is ordering models according to their cost-normalized accuracies $p_i/c_i$.
% \end{repproposition}
\begin{repproposition}{prop:llm_ordering} With $K$ (LLM, prompt) options, each with probability of correct answer $p_k$ and cost $c_k$,
ordering the options according to their cost-normalized accuracies $\frac{p_k}{c_k}$ minimizes the total cost.
\end{repproposition}
\begin{proof} Suppose the (LLM, prompt) options are not ordered in terms of $\frac{p_k}{c_k}$. We will prove that swapping the order results in a better cascade. Again recall that the accuracy is independent of the cost because we stop as soon as oracle confirms the answer. The expected accuracy is equal to the probability of at least one (LLM, prompt) option being correct. To proceed, suppose (LLM, prompt) options are not ordered and there is some $\kappa$ such that $\frac{p_\kappa}{c_\kappa} < \frac{p_{\kappa+1}}{c_{\kappa+1}}$. Let us compute the expected change in inference cost when we flip their order.

The change in inference cost arises from the scenarios where (LLM, prompt) option $\kappa$ will be used but $\kappa+1$ will not be, and vice versa. Define $q_{\kappa-1}=\prod_{i=1}^{\kappa-1}(1-p_i)$ which is the probability that the first $\kappa-1$ (LLM, prompt) options fails. The probability of the ``excess cost'' ($ec$) associated with (LLM, prompt) option $\kappa$ and then $\kappa+1$ is
\[ 
ec_\kappa= q_\kappa c_\kappa+q_\kappa(1-p_\kappa)c_{\kappa+1}.
\]
This is because if the first $\kappa-1$ (LLM, prompt) options fail, we definitely pay $c_\kappa$ and only pay $c_{\kappa+1}$ if (LLM, prompt) option $\kappa$ fails. This ``excess cost'' definition does not reflect the (LLM, prompt) options strictly before $\kappa$ or strictly after $\kappa+1$. This is because the expected cost of other (LLM, prompt) options arise from the symmetric events with respect to (LLM, prompt) options $\kappa$ and $\kappa+1$ (either problem is already solved by the time we reach $\kappa$, or both $\kappa$ and $\kappa+1$ failed).

Conversely, the excess cost associated with $\kappa+1$ is (\ie using $\kappa+1$ first and then $\kappa$)
\[ 
ec_{\kappa+1}= q_\kappa c_{\kappa+1}+q_\kappa (1-p_{\kappa+1})c_\kappa.
\] 
To proceed, observe that if $\frac{p_\kappa}{c_\kappa} < \frac{p_{\kappa+1}}{c_{\kappa+1}}$, then we would be better off by switching the models because  
\begin{align}
ec_{\kappa+1}<ec_{\kappa}&\iff q_\kappa c_{\kappa+1}+q_\kappa(1-p_{\kappa+1})c_\kappa<q_\kappa c_\kappa+q_\kappa(1-p_\kappa)c_{\kappa+1}\\
& \iff c_{\kappa+1}+(1-p_{\kappa+1})c_\kappa<c_\kappa+(1-p_\kappa)c_{\kappa+1}\\
& \iff p_{\kappa}c_{\kappa+1}<p_{\kappa+1}c_\kappa\\
& \iff p_\kappa/c_\kappa<p_{\kappa+1}/c_{\kappa+1}.
\end{align}
\end{proof}
%\section{LLM Ordering Lemma}
%\begin{lemma} Consider a cascade of models with the following decision rule: Query until the same output is observed for $K\geq 2$ times. If $K$ is reached, return it, otherwise, abstain. Suppose that the output distribution of these models are such that, no $K$ models return the same incorrect output almost surely. Then, the optimal cascade strategy is sorting these models according to their (increasing) cost.\end{lemma}
Experimentally, we verify that our models used \alg are sorted according to the decreasing cost-normalized accuracies.
The model order (using the purely monetary cost function with $\alpha = \frac{1}{10}$) is shown in \Cref{fig:cost_app}. With other $\alpha$ values, Llama-2-13b will be cheaper, so the cost-normalized accuracy will always be decreasing.
\begin{figure}[t]
%\vspace{-5mm}
\begin{center}
\includegraphics[scale=0.45]{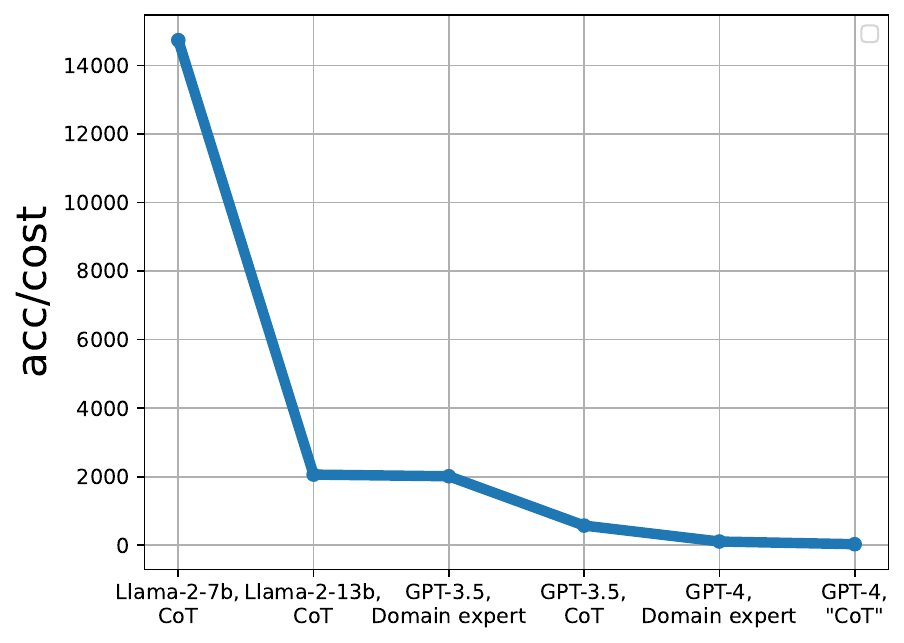}
\end{center}
\vspace{-3mm}
\caption{Cost-normalized accuracy
}
\label{fig:cost_app}
\vspace{-15pt}
\end{figure}

\subsection{Do policies that consider response consistency perform well?}

%A 2-consistent policy is a policy that waits for 2 consistent answer before terminating.

While \Cref{prop:llm_ordering} provides the optimal ordering rule of LLMs in the cascade, understanding the properties of the optimal trajectory on the cascade is more challenging. To understand the optimal decision rule, we study when a policy terminates for a given question. Given a question, suppose the policy issues up to $\Omega$ queries of the same question to reach a response, where $1\leq j\leq \Omega$. 
%Let $\mathcal{O}=(\hat{O}_j)_{j=1}^\Omega$ be the responses returned by the $\Omega$ options. 
%Let us also call $\mathcal{O}$ $n$-consistent if the most frequent answer is repeated $n$ times. 
Experimentally, we have found that the majority of the response statistics are 2-consistent, that is, the policy stops once it observes the same answer twice. Specifically, in our GSM8K experiments with a budget of \$0.30 and $\alpha=\frac{1}{20}$, only 6.98\% of the responses achieve consistency greater than 2. This percentage increases modestly to 14.32\% when the budget is raised to \$10 (a factor of 33$\times$). 
In this section, we develop a theoretical explanation for this observation by characterizing of the performance of 2-consistent policies and establishing a formal proof of \Cref{prop:accuracy_loss_2_consistent} . 
The theory relies on \Cref{def:second_moment}, which bounds how likely the wrong answers are to be repeated.

\begin{repproposition}{prop:accuracy_loss_2_consistent} For the problem stated in  \Cref{eqn:rl_problem} that achieves $C_*$, the optimal expected accuracy subject to budget constraints,
%Then there exists a policy that always returns an answer for the current question $Q_i$ after achieving 1- or 2-consistency 
there exists a 2-consistent policy that achieves an accuracy of at least $C_*-\frac{1}{2}\Omega^2\epsilon$. %\jiasi{ask Samet: should we keep 1-consistency?}
\end{repproposition}
% Note that this statement applies generally, namely, we are allowed to have budgets or text embeddings, and we don't require the models to return accurate responses. The proof of this proposition relies on reducing the optimal policy to an at most 2-consistent policy. We then control the excess error of this new policy by bounding the likelihoods of more than 2-consistent cascade instances, which is captured by the $0.5K^2\eps$ term. 

% bounds the likelihood of the scenarios where the optimal policy can benefit from achieving 3 or more consistency during inference

\begin{proof} The proof will be achieved via a reduction. Let $\pi_*$ be the optimal policy that achieves an expected accuracy of $C_*$. We now derive a new policy $\pi_2$ from $\pi_*$. Specifically, let $\pi_2$ be same as $\pi_*$ if the trajectory does not attain 2-consistency i.e.~the answers are not repeated during the cascade. Otherwise, we let $\pi_2$ terminate early once $2$-consistency is achieved for the first time. By construction $\pi_2$ will satisfy the budget constraints as it queries the LLMs strictly less or equal to $\pi_*$. Let $\hat{Y}_i(\pi)$
%:=\hat{Y}(\pi,Q_i)$
denote the final (stochastic) answer of a policy $\pi$ given problem $Q_i$. 
Recall that $Y_i$ %$a^*=a^*_{\qb}$ \zijian{maybe $Y_i$ instead?} 
denotes the correct answer. Let the random variable $C(\pi_2)$ be the total accuracy of $\pi_2$ summed over all queries where randomness arises from the stochasticity of the LLM responses as well as the policy $\pi_*$. We can lower bound the accuracy of $\pi_2$ as follows
\begin{align}
\E[C(\pi_2)]&=\sum_{i=1}^n \Pro(\hat{Y}_i(\pi_2)=Y_i\bgl Q_i)\\
&\geq \sum_{i=1}^n \Pro(\hat{Y}_i(\pi_2)=Y_i\bgl Q_i) -\sum_{i=1}^n \Pro(\hat{Y}_i(\pi_2)\neq \hat{Y}_i(\pi_*)\bgl Q_i)\\
&=C_*-\sum_{i=1}^n \Pro(\hat{Y}_i(\pi_2)\neq \hat{Y}_i(\pi_*)\bgl Q_i)\\
&=C_*-\sum_{i=1}^n \Pro(\text{First 2-consistent answer is wrong}\bgl Q_i)\\
&\geq C_*- \sum_{i=1}^n \Pro(\text{Any wrong 2-consistent answer within $\Omega$ responses}\bgl Q_i) \label{eqn:2-cons}.
\end{align}
For example, if the responses are (5, 10, 3, 5, 10) to a given question $\Pro(\text{First 2-consistent answer is wrong})$ is the probability that ``5'' is the wrong answer, and $\Pro(\text{Wrong 2-consistent answer within $\Omega$ responses})$ is the probability that either ``5'' or ``10'' is a wrong answer.
To proceed, we will upper bound the right hand side of (\ref{eqn:2-cons}). Let $W_{any}$ be the event $\text{``Any wrong 2-consistent answer within $\Omega$ responses''}$. Similarly let $W_{uv}$ be the event that the  $\LMPP_u$ and $\LMPP_v$ return the same incorrect answer. Note that, for $W_{\text{any}}$ to happen, a $W_{uv}$ event has to happen. Thus, through union bound across $\Omega$ queries given $Q_i$, we have that 
\[ 
\Pro(W_{\text{any}}|Q_i)\leq \sum_{u<v}\Pro(W_{uv}|Q_i).
\]
Plugging in the definition of $W_{uv}$, we obtain
\begin{align}
    \Pro(W_{\text{any}}|Q_i)&\leq \sum_{u<v}\Pro(W_{uv}|Q_i)\\
    &=\sum_{u<v} \sum_{\hat{O}\neq Y_i} \Pro(M_u(P_u(Q_i))=\hat{O})\Pro(M_v(P_v(Q_i))=\hat{O})\\
    &\leq \sum_{u<v} \sum_{\hat{O}\neq Y_i} \frac{\Pro(M_u(P_u(Q_i))=\hat{O})^2+\Pro(M_v(P_v(Q_i))=\hat{O})^2}{2}\\
    &\leq \sum_{\hat{O}\neq Y_i} \frac{\Omega^2}{2}\sup_{1\leq u\leq K}\Pro(M_u(P_u(Q_i))=\hat{O})^2.
\end{align}
To finalize, we sum over all $n$ questions and use the definition of $\eps$ to obtain 
\[ 
\sum_{i=1}^n\Pro(W_{\text{any}}|Q_i)\leq \frac{\Omega^2}{2}\sum_{i=1}^n\sup_{1\leq u\leq K}\sum_{\hat{O}\neq Y_i} \Pro(M_u(P_u(Q_i))=\hat{O})^2= \frac{\Omega^2}{2}\eps.
\] 
Plugging the right hand side in \eqref{eqn:2-cons}, we conclude with the advertised theorem statement.
\end{proof}
%Let $E_i$ be the indicator function of $m_i(\qb)=a$, which occurs with probability $p_i$. We have that
%\[ \Pro(\sum_{i=1}^K E_i\geq 2)=1-\Pro(\sum_{i=1}^K E_i\leq 1)=1-\prod_{i=1}^K(1-p_i)-\sum_{j=1}^Kp_i\prod_{i\neq j}^K(1-p_i).\]
%Observe that $\Pro(\sum_{i=1}^K E_i\geq 2)\leq 1-\sum_{j=1}^Kp_i\prod_{i\neq j}^K(1-p_i)\leq 1-\sum_{j=1}^Kp_i$.
\begin{comment}
%Note that $i$ and $j$ may not be unique, \ie $W_{ij}$ could be the event that a re-query of the same model returns the same wrong answer. 
By \Cref{def:second_moment}, we have that
% \begin{align}
% \sum_i \Pro(W_{uv}\bgl Q_i ) &=\sum_i \sum_{a\neq a^*} \Pro(m_u(\qb)=a)\Pro(m_v(\qb)=a)\\
% &\leq \sum_i \sum_{a\neq a^*} \frac{\Pro(m_u(\qb)=a)^2+\Pro(m_v(\qb)=a)^2}{2}\\
% &\leq \eps.
% \end{align}
\begin{align}
\sum_{i=1}^n \Pro(W_{uv}\bgl Q_i ) &=\sum_{i=1}^n \sum_{\hat{O}\neq Y_i} \Pro(M_u(P_u(Q_i))=\hat{O})\Pro(M_v(P_v(Q_i))=\hat{O})\\
&:=\sum_{i=1}^n \bar{\eps}_{uvi},\end{align}
where $\bar{\eps}_{uvi}$ is defined to be the probability of LLM-prompt pairs $u$ and $v$ hitting same incorrect answer. 
Secondly, through the mean inequalities and union bound, we have 
\end{comment}
% \xuechen{change definition, i,j are 2 times re-query}

\textbf{Further analysis of 2-consistent policies as a function of incorrect answer likelihoods}. %\label{sec refined 2-consistent}
To provide further intuition on \Cref{def:second_moment}, let us denote the probability of 1-shot correct answer of a model by $p_{cor}=\Pro(m_i(\qb)=a^*)$. Small $p_{cor}$ implies that individual LLM responses are unreliable. However, if $\eps$ is relatively small, 2-consistency should return the correct answer with high probability because the majority of the 2-consistent instances will arise from the correct answers. Formalizing this intuition, below we upper-bound the likelihood of error conditioned on 2-consistency as $\order{\eps/p_{cor}^2}$. 

\begin{proposition}
Recall that $((M+P)_k)_{k=1}^K$ are the LLM-prompt pairs used in the cascade. To obtain a more refined bound in terms of the probabilities of correct and wrong answers, let $p_k=\Pro(M_k(P_k(Q))=Y)$ for $1\leq k\leq K$. Also set $p_{\min}=\min_{1\leq k\leq K}p_k>0$ and $p_{\max}=\max_{1\leq k\leq K}p_k$. Let us suppose these hold uniformly over all problem instances $Q$. As $\eps\rightarrow 0$, over the 2-consistent cascade instances, the probability of error is upper bounded by 
\begin{align}
\text{err}_{\text{upper}}=(1+o(1))\frac{\sum_{k=2}^{K}(1-p_{\max})^{-2}(k-1) \mathbf{Q}_k\eps}{\sum_{k=2}^{K} \mathbf{P}_k}.\label{refined err bound}
\end{align}
Here $\mathbf{P}_k$ is the likelihood that $2$-consistent correct answer will first be achieved at the $k$'th model and $Q_k$ is the likelihood that first $k-1$ models fail to achieve 2-consistency. $o(1)$ term captures the higher order terms that arise from the probability of the 2-consistent incorrect cascades (which is vanishing compared to $\mathbf{P}_k,\mathbf{Q}_k$). 
\end{proposition}
\begin{proof}
In \eqref{refined err bound}, the ``$(1-p_{\max})^{-2}\eps(k-1)\mathbf{Q}_{k}$'' term in the numerator upper bounds the likelihood that, an incorrect $2$-consistent answer will be generated precisely at the $k$th model under \ref{def:second_moment}. 

Denote the probability of the event ``correct answer never appears until model $k-1$'' by $F_k=\prod_{j<k}(1-p_j)$. As $\eps\rightarrow0$, for $2\leq k\le K$, these take the simplified closed forms 
\begin{align}
\mathbf{P}_k&=p_k F_k\sum_{j<k}\frac{p_j}{1-p_j}=p_k\sum_{j<k}p_j\prod_{l<k, l\neq j}(1-p_l) \\ 
\mathbf{Q}_k&=F_k(1+\sum_{j<k}\frac{p_j}{1-p_j})=\sum_{j\le k}p_j\prod_{l< k, l\neq j}(1-p_l)+\prod_{j< k}(1-p_j).
\end{align}
For a fixed $k$, let us set the ratio 
\begin{align}
    \text{rat}_k=\frac{\mathbf{Q}_k}{\mathbf{P}_k}\leq \frac{1+\sum_{j<k}\frac{p_j}{1-p_j}}{p_k\sum_{j<k}\frac{p_j}{1-p_j}}=\frac{1+M}{p_kM}
\end{align} 
where $M=\sum_{j<k}\frac{p_j}{1-p_j}$. If $M\geq 1$, we have that $\text{rat}_k\leq \frac{2M}{p_kM}\leq \frac{2}{p_k}$. If $M\leq 1$, we have that $\text{rat}_k\leq \frac{2}{p_kM}\leq \frac{2}{p_k\sum_{j<k}p_j}$. Thus, we obtain 
\begin{align}
\frac{\mathbf{Q}_k}{\mathbf{P}_k}\leq \frac{2}{p_k\cdot\min(1,\sum_{j<k}p_j)}
\end{align}
This in turn implies that the error is controlled by the worst case ratio 
$(k-1)\mathbf{Q}_k/\mathbf{P}_k$ given by $\max_{k\leq K} \frac{2(k-1)}{p_k\cdot\min(1,\sum_{j<k}p_j)}$. Using $p_k\cdot\min(1,\sum_{j<k}p_j)\geq \min((k-1)p_{\min}^2,p_{\min})\geq p_{\min}^2$, we obtain the conclusion that 
\begin{align}
\text{err}_{\text{upper}}\leq (2+o(1))\frac{(1-p_{\max})^{-2}\eps}{\min(p_{\min}^2,p_{\min}/(K-1))}\leq (2+o(1))\frac{(K-1)\eps}{(1-p_{\max})^2p_{\min}^2}\propto \frac{(K-1)\eps}{p_{\min}^2}
\end{align}
\end{proof}
%whenever $1-p_{\max}$ is lower bounded by a constant. 

%Note that, if $p_i$'s are small, we would obtain the refined bound $\text{err}_{\text{upper}}\leq (1+o(1)) \frac{\eps}{p_{\min}^2}$ by approximating $\min(1,\sum_{j<i}p_j)\geq (i-1)p_{\min}$.
% where $P_1=0$ and $Q_1=1$
%Here, as $\eps\rightarrow0$, $P_i$ equates to the likelihood that $2$-consistent correct answer will first be achieved at the $i$'th model.  Note that, this bound captures the fact that as the accuracies $p_i$ grow, }

\section{Additional Results}
\label{app:additional results}
\subsection{Time-Varying API Query Latency}
\begin{figure}[]
\centering
%\hspace{-35pt}
    \centering
	\begin{tikzpicture}
		\node at (0,0) [scale=0.25]{\includegraphics{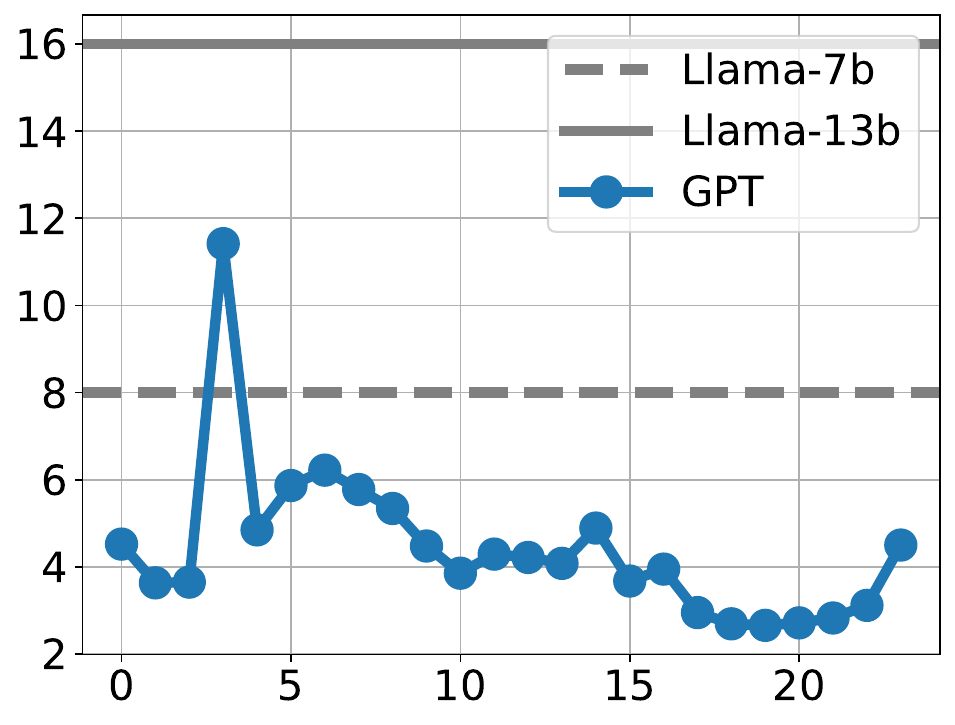}};
  \node at (0,-1.55) [scale=0.8] {o'clock};
  \node at (-2.15 ,0) [scale=0.8,rotate=90] {Latency (s)};
  % \node at (0,-2) [scale=0.8] {50 times};
\end{tikzpicture}
\vspace{-7pt}
\caption{End-to-end latency of querying LLM models over a 24-hour period.}\vspace{-13pt}
\label{fig:Latency_Time}
% \vspace{-0.1cm}
\end{figure}
\label{sec:latency}

\begin{table}[]
\scriptsize
    \centering
    \scriptsize
    \caption{Performance with time-varying API query latency.}
    % \vspace{-10pt}
     \label{tab:changing_latency}
     \small
     \scriptsize
    \begin{tabular}{|p{1.4cm}|p{2cm}|p{2cm}|p{1.1cm}|}
\hline
   Method   & Accuracy with Time-varying Latency & Accuracy assuming Constant Latency    & Update time (s)\\ \hline
  {\sf \scriptsize TREACLE}   & 86.4          &    76.1  & 0.02 \\ \hline
  Calibrated  & \multirow{2}{*}{80.1}          &   \multirow{2}{*}{75.7} &  \multirow{2}{*}{617.30} \\ 
  Cascade & & & \\\hline
    \end{tabular}
        % }
        % \vspace{-20pt}
\end{table}
%\xuechen{move to appendix? xuechen stop here}
The latency of querying LLM APIs (such as OpenAI's GPT models) may vary over the short term time based on network congestion or LLM inference time. %with an updated cost estimate. 
%\carlee{is this the estimate of the cost of each query?} \xuechen{No. We estimate the average re-query latency for each model-prompt pair. See table 2}
%For example, in our experiments later in \Cref{tab:changing_latency}, we update the estimated latency per query every hour.
%To evaluate whether our methods can adapt to-varying latency, 
To showcase this, we recorded traces of the API latency (including communication and communication latency) over a 24-hour period. The measurements are shown in \Cref{fig:Latency_Time}. 
%The latency for the locally run Llama models remains constant at 8 s per query for the smaller Llama model and 16 s for the larger Llama model. The latency for cloud-based GPT models fluctuates over time, with a variance of 3.13 s.
%\textbf{Time-varying costs.} 
%Costs may change over the long term (\eg from updated API prices) or the short term (\eg from network congestion to the servers of LLM providers).
%We assume that users continue to use LLMs to answer questions during this period of time-varying latency.
%To adapt to time-varying latency, 
We also modified \alg and the experimental setup slightly.
Each hour, \alg attempts to answer the entire GSM8K test set with a budget of \$0.6,
%\jiasi{why this number? 0.6*24 = 14.4 = total budget for ``Constant Latency''?} \xuechen{Select a budget that \alg has clear improvement. With too low or too high a budget, almost all problems use gpt4, so there is only minor improvement. each hour has a budget separately. Couldn't use the budget for other hours. }
using the historical average latency from the previous hour to update the per-query cost in the denominator of $\mathcal{B}_k$. The budget resets every hour.
%To handle this, \alg periodically updates its estimate of the remaining budget $\mathcal{B}_k$ in the state vector.
%Specifically, it periodically updates the denominator of $\mathcal{B}_k$ with the new  average per-query cost of $\LMPP_k$.
We compare this to the vanilla \alg shown previously, which 
%As a baseline, we assume that the algorithm 
has the same total budget and
uses a fixed latency in the denominator of $\mathcal{B}_k$ for the entire 24-h period; for example, for GPT models, it assumes a fixed average latency of 6 s.
The results in \Cref{tab:changing_latency} indicate that \alg that adapts to time-varying achieves higher accuracy than the version that assumes constant latency. The update time to calculate the new $\mathcal{B}_k$ each hour is also minimal, at 20 ms.
%For $\alg$, we use $\mathcal{B_\text{remaining}} = \{\frac{\text{Total remaining budget}}{\text{Average cost per query for }\mathcal{\LMPP}}\}$ as the state, allowing us to update the "$\text{Average cost per query for }\mathcal{\LMPP}$" with real-time data quickly.
%We also made simmilar modifications to \algU to incorporate time-varying latency.
%We update the $p_{\text{low}}$ and $p_{\text{high}}$ parameters every hour \jiasi{right?}
%based on the new cost for the Calibrated Cascade algorithm, 
%based on the new latency values,
%but this requires incorporating new latency values and
%but this requires conducting a grid search on the validation set. This can be slow when the validation set is large, and contributes to the high update time in \Cref{tab:changing_latency}.
%, and using a smaller validation set may impact the performance. 
%The results are displayed in Table ~\ref{tab:changing_latency}. They indicate that $\alg$ quickly adapts to new latency conditions.

\subsection{Model and Prompt Characterization}\label{app:Characterization}
We show the model details, prompt strategies, temperature, and various details of each configuration in \Cref{tab: model performance} Note that these are the model-prompt combinations chosen in our framework because of the evaluation in \Cref{tab:single model cost}.
%\vspace{-0.2in}
\begin{table*}[ht]
    \centering
    \caption{Characterization of LLM performance in terms of accuracy, latency, and price, for a single query 
    %\jiasi{right? re-query results not included in the table?} \zijian{Yes}
    with temperature equal to 0. 
    The Llama models do not have a direct monetary price because they are open-source and we run them locally.}
    \label{tab: model performance}
    \begin{subtable}[h]{0.95\textwidth}
    \scriptsize
    \centering
    \vspace{-0.1in}
    \caption{GSM8K dataset}
    % \caption{\zijian{fill in caption}}
    % \resizebox{\linewidth}{!}{
    \begin{tabular}{c|c|cc|c|cc}
    \hline
        \multirow{2}{*}{Model} & \multirow{2}{*}{Prompt} & \multicolumn{2}{c|}{Accuracy (\%)} & \multirow{2}{*}{Avg Latency (sec/query)} & \multicolumn{2}{c}{Avg Monetary Price (\$/query)}\\
         & & train & test & & train & test\\\hline\hline
        Llama-2-7b-chat & CoT few-shot & 23.36 & 23.65 & 8 & n/a & n/a \\\hline
        Llama-2-13b-chat & CoT few-shot & 37.90 & 37.91 & 16 & n/a & n/a \\\hline
        MetaMath-7b & domain expert & 92.48 & 66.19 & 8 & n/a & n/a \\\hline
        MetaMath-13b & domain expert & 92.81 & 70.43 & 16 & n/a & n/a \\\hline
        \multirow{2}{*}{GPT-3.5-turbo (old)} & domain expert & 76.60 & 73.62 & 6 & $3.83\times10^{-4}$ & $3.66\times10^{-4}$\\
                                       & CoT few-shot & 82.00 & 79.15 & 6 & $1.37\times10^{-3}$ & $1.38\times10^{-3}$\\\hline
        \multirow{2}{*}{GPT-3.5-turbo (new)} & domain expert & 76.60 & 73.62 & 6 & $3.38\times10^{-4}$ & $3.20\times10^{-4}$\\
                                       & CoT few-shot & 82.00 & 79.15 & 6 & $9.87\times10^{-4}$ & $9.90\times10^{-4}$\\\hline
        \multirow{2}{*}{GPT-4-turbo} & domain expert & 88.18 & 88.48 & 6 & $5.88\times10^{-3}$ & $5.91\times10^{-3}$\\
                                       & CoT few-shot & 92.61 & 92.34 & 6 & $1.21\times10^{-2}$ & $1.22\times10^{-2}$\\\hline
        \multirow{2}{*}{GPT-4} & domain expert & 84.33 & 83.17 & 6 & $7.30\times10^{-3}$ & $7.57\times10^{-3}$\\
                                       & CoT few-shot & 93.59 & 92.95 & 6 & $2.92\times10^{-2}$ & $2.94\times10^{-2}$\\\hline
    \end{tabular}
        % }
    % \caption{GSM8K dataset}
    \label{tab:GSM8K model performance}
    \end{subtable}
    
    \begin{subtable}[h]{0.95\textwidth}
    \scriptsize
    \centering
    \vspace{0.05in}
    \caption{CSQA dataset}
    % \caption{\zijian{fill in caption}}
    % \resizebox{\linewidth}{!}{
    \begin{tabular}{c|c|cc|c|cc}
    \hline
        \multirow{2}{*}{Model} & \multirow{2}{*}{Prompt} & \multicolumn{2}{c|}{Accuracy (\%)} & \multirow{2}{*}{Avg Latency (sec/query)} & \multicolumn{2}{c}{Avg Monetary Price (\$/query)}\\
         & & train & test & & train & test\\\hline\hline
        Llama-2-7b-chat & CoT few-shot & 64.72 & 67.65 & 16 & n/a & n/a \\\hline
        Llama-2-13b-chat & CoT few-shot & 68.19 & 71.17 & 31 & n/a & n/a \\\hline
        GPT-3.5-turbo & standard few-shot & 74.09 & 76.82 & 0.3 & $6.02\times10^{-4}$ & $6.02\times10^{-4}$\\\hline
        GPT-4 & standard few-shot & 84.29 & 87.14 & 0.7 & $1.20\times10^{-2}$ & $1.20\times10^{-2}$\\\hline
    \end{tabular}
        % }
    % \caption{CSQA dataset}
    \end{subtable}

    \begin{subtable}[h]{0.95\textwidth}
    \scriptsize
    \centering
    \vspace{0.05in}
    \caption{LLC dataset}
    % \caption{\zijian{fill in caption}}
    % \resizebox{\linewidth}{!}{
    \begin{tabular}{c|c|cc|c|cc}
    \hline
        \multirow{2}{*}{Model} & \multirow{2}{*}{Prompt} & \multicolumn{2}{c|}{Accuracy (\%)} & \multirow{2}{*}{Avg Latency (sec/query)} & \multicolumn{2}{c}{Avg Monetary Price (\$/query)}\\
         & & train & test & & train & test\\\hline\hline
        Llama-2-7b-chat & CoT few-shot & 44.23\% & 44.6\% & 16 & n/a & n/a \\\hline
        GPT-3.5-turbo & plain text & 62.71\% & 63.20\% & 0.3 & $1.20\times10^{-4}$ & $1.14\times10^{-4}$ \\\hline
        GPT-3.5-turbo & CoT few-shot & 86.53\% & 87.13\% & 0.3 & $5.83\times10^{-4}$ & $5.82\times10^{-4}$\\\hline
        GPT-4 & CoT few-shot & 92.68\% & 93.2\% & 0.7 & $1.29\times10^{-2}$ & $1.29\times10^{-2}$\\\hline
    \end{tabular}
        % }
    % \caption{CSQA dataset}
    \end{subtable}
\end{table*}

% \subsection{Question difficulty for different models}

% \jiasi{What goes here?}

\subsection{Different types of reasoning tasks}

To visualize the differences between the three reasoning datasets, in \Cref{fig:pie} we plotted the fraction of questions where the most powerful (LLM, prompt) combination in the sorted list correctly answered the question (the ``in order'' pie slice), versus those questions where a less powerful combination succeeded and a more powerful combination failed (all other slices of the pie).
Interestingly for all datasets, there are minority cases where less powerful LLMs (the smaller pieces of the pie) can answer the question correctly.
Such cases are most prevalent in the GSM8K dataset and least prevalent in LLC, possibly because the math questions of GSM8K are more difficult.
Despite these dataset differences, \alg still chooses the right (LLM, prompt) combination to achieve higher accuracy in all datasets than the baselines.

\begin{figure}[]
\centering
%\hspace{-35pt}
\begin{subfigure}[b]{0.45\textwidth}
    \centering
	\begin{tikzpicture}
		\node at (0,0) [scale=0.4]{\includegraphics{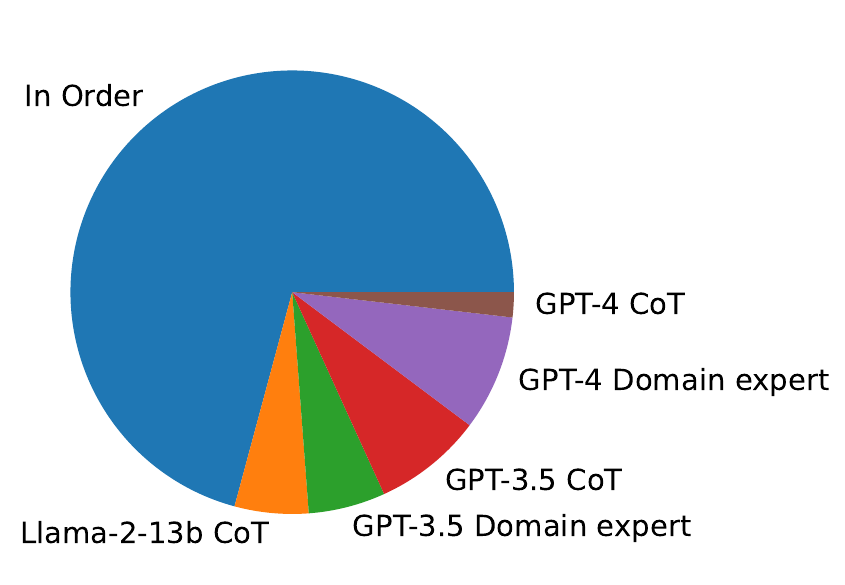}};
        % \node at (0,0) [scale=0.5]{\includegraphics{GSM8K_pie.pdf}};
    \node at (-0.5,-2) [scale=0.8] {GSM8K};
\end{tikzpicture}
\end{subfigure}
\begin{subfigure}[b]{0.45\textwidth}
    \centering
	\begin{tikzpicture}
		\node at (0,0) [scale=0.4]{\includegraphics{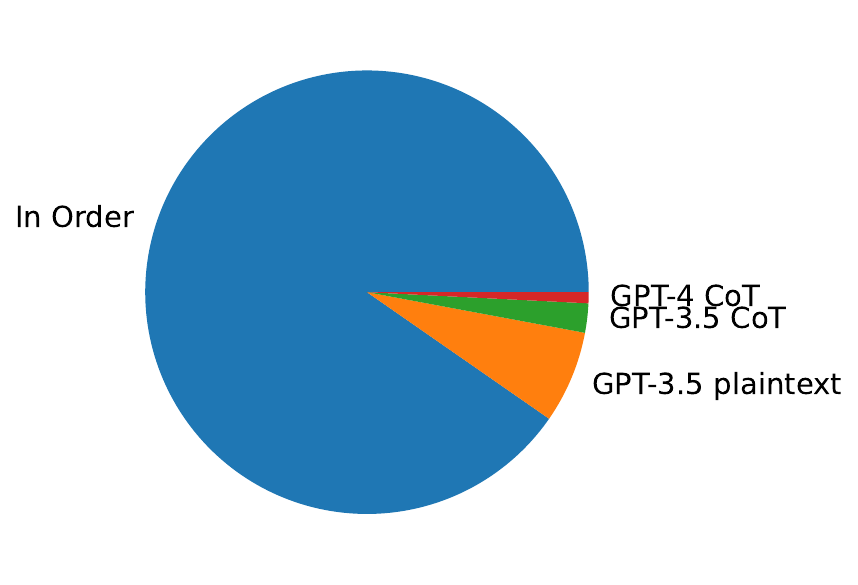}};
        % \node at (0,0) [scale=0.5]{\includegraphics{img/LLC_pie.pdf}};
  \node at (-0.5,-1.8) [scale=0.8] {LLC};
\end{tikzpicture}
\end{subfigure}
% \begin{subfigure}{}
%     \centering
% 	\begin{tikzpicture}
% 		\node at (0,0) [scale=0.35]{\includegraphics{img/legend_subset.pdf}};
% \end{tikzpicture}
% \end{subfigure}
\vspace{-5pt}
\caption{Fraction of questions that are solved by (LLM, prompt) combinations ordered from least to most powerful (``in order''). Minority slices are queries where less powerful combinations correctly answered.
%Larger models is always better than smaller models. Others: The model-prompt pair fails, but less powerful pairs answer correctly.
}\vspace{-10pt}
\label{fig:pie}
\end{figure}

\subsection{Different cost function parameters and datasets}
\label{app:more_dataset}
Providers may adjust the per-token API prices, or the user may value latency and monetary price differently.
%Additionally, the ratio of latency and monetary cost can be adjusted based on changing situations. 
Therefore, we conducted experiments using different settings of the $\alpha$ (defined in \Cref{sec:data_collection_training})and $\beta$ (defined in \Cref{sec:problem}) parameters in the cost function.
In \Cref{fig:all_12_app}, the cost ratio $\alpha$ increases from left to right, and hence the cost difference between more powerful (GPT) and weaker (Llama) models gradually decreases according to the definition.
%The specific costs are shown in the table. 
Under different pricing policies, $\alg$ consistently achieves better performance than the online baselines. In other words, a single $\alg$ model can easily accommodate varying budget requirements and cost functions, since it was trained under heterogeneous parameter settings.

% \input{tables/single model accuracy vs cost}
%\subsection{Additional datasets and baselines}%\label{app:more_dataset}
Also, we mainly show GSM8K results in the main paper, because of limited space.
Across the additional datasets in \Cref{fig:all_12_app} (CSQA and LLC), the results consistently show good performance.
%for the experiments and get consistent results.
%\alg always outperform other baselines. 
%Meanwhile, we add one one baseline called \ABCascade. In this \ABCascade algorithm, 
%    We query each \LMPP combination $B$ times, with the models sorted from the cheapest model to the most expensive model (equivalently, from the highest to lowest accuracy).
%    We stop querying when the same response appears $A$ times, or
%    there are no more LLMs to query.
%    We set $A=2,B=2$ since it is the most common way and best combination according to our empirical experimental results. The results are shown in \Cref{fig:all_12_app}

\begin{figure*}[]
\centering
\vspace{-10pt}
%\hspace{-35pt}
\begin{subfigure}[b]{0.26\textwidth}
    \centering
	\begin{tikzpicture}
		\node at (0,0) [scale=0.24]{\includegraphics{acc_50.pdf}};
  % \draw[blue, very thick] (0,0) circle (1.5);
  \node at (0,-1.5) [scale=0.5] {GSM8K, $\alpha = \frac{1}{50}$};
\end{tikzpicture}
\end{subfigure}
\begin{subfigure}[b]{0.26\textwidth}
    \centering
	\begin{tikzpicture}
		\node at (0,0) [scale=0.24]{\includegraphics{acc_20.pdf}};
  \node at (0,-1.5) [scale=0.5] {GSM8K, $\alpha = \frac{1}{20}$};
\end{tikzpicture}\label{fig:ece}
\end{subfigure}
\begin{subfigure}[b]{0.26\textwidth}
    \centering
	\begin{tikzpicture}
		\node at (0,0) [scale=0.24]{\includegraphics{acc_10.pdf}};
  \node at (0,-1.5) [scale=0.5] {GSM8K, $\alpha = \frac{1}{10}$};
\end{tikzpicture}
\end{subfigure}\\
\begin{subfigure}[b]{0.26\textwidth}
    \centering
	\begin{tikzpicture}
		\node at (0,0) [scale=0.24]{\includegraphics{acc_tradeoff-1000000.pdf}};
  \node at (0,-1.5) [scale=0.5] {GSM8K, $\beta $=1M};
\end{tikzpicture}
\end{subfigure}
\begin{subfigure}[b]{0.26\textwidth}
    \centering
	\begin{tikzpicture}
		\node at (0,0) [scale=0.24]{\includegraphics{acc_tradeoff-500000.pdf}};
  \node at (0,-1.5) [scale=0.5]{GSM8K, $\beta =$500k};
\end{tikzpicture}
\end{subfigure}
\begin{subfigure}[b]{0.26\textwidth}
    \centering
	\begin{tikzpicture}
		\node at (0,0) [scale=0.24]{\includegraphics{acc_tradeoff-50000.pdf}};
  \node at (0,-1.5) [scale=0.5] {GSM8K, $\beta =$50k};
\end{tikzpicture}
\end{subfigure}\\
% \vspace{-5pt}
% \begin{subfigure}[b]{1\textwidth}
%     \centering
% 	\begin{tikzpicture}
% 		\node at (0,0) [scale=0.24]{\includegraphics{legend.pdf}};
%   % \node at (0,-2) [scale=0.8] {10 times};
% \end{tikzpicture}
% \end{subfigure}\\
\begin{subfigure}[b]{0.26\textwidth}
    \centering
	\begin{tikzpicture}
		\node at (0,0) [scale=0.24]{\includegraphics{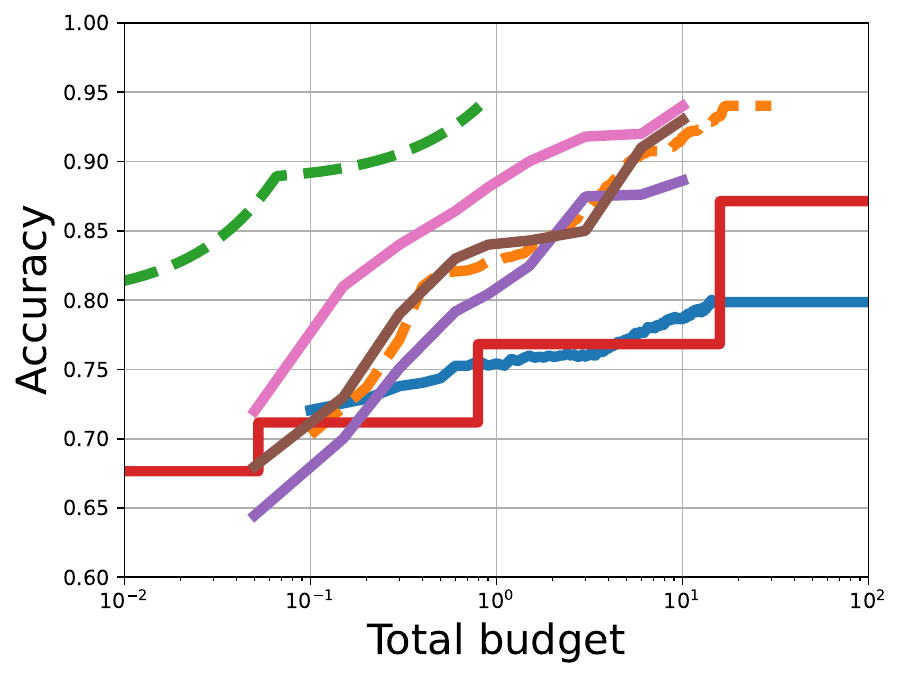}};
  \node at (0,-1.5) [scale=0.5] {CSQA, $\alpha = \frac{1}{50}$};
\end{tikzpicture}
\end{subfigure}
\begin{subfigure}[b]{0.26\textwidth}
    \centering
	\begin{tikzpicture}
		\node at (0,0) [scale=0.24]{\includegraphics{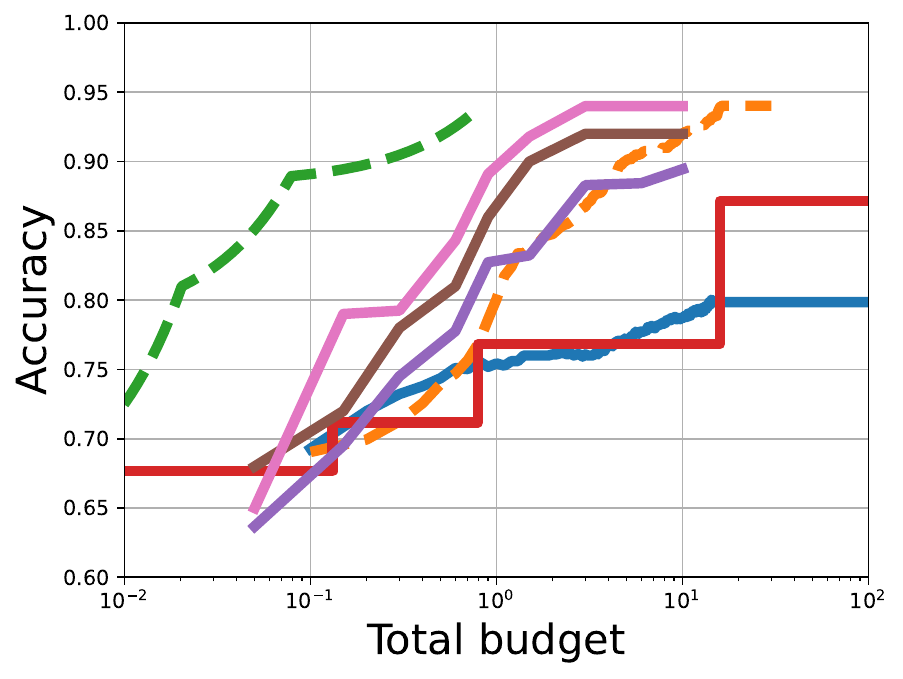}};
  \node at (0,-1.5) [scale=0.5] {CSQA, $\alpha = \frac{1}{20}$};
\end{tikzpicture}\label{fig:ece}
\end{subfigure}
\begin{subfigure}[b]{0.26\textwidth}
    \centering
	\begin{tikzpicture}
		\node at (0,0) [scale=0.24]{\includegraphics{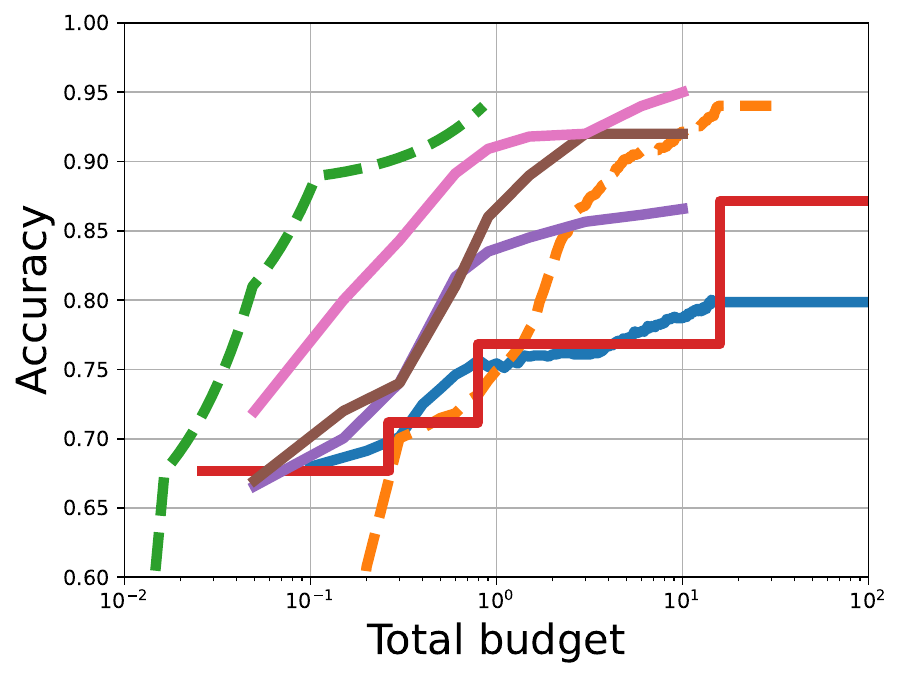}};
  \node at (0,-1.5) [scale=0.5] {CSQA, $\alpha = \frac{1}{10}$};
\end{tikzpicture}
\end{subfigure}\\
\begin{subfigure}[b]{0.26\textwidth}
    \centering
	\begin{tikzpicture}
		\node at (0,0) [scale=0.24]{\includegraphics{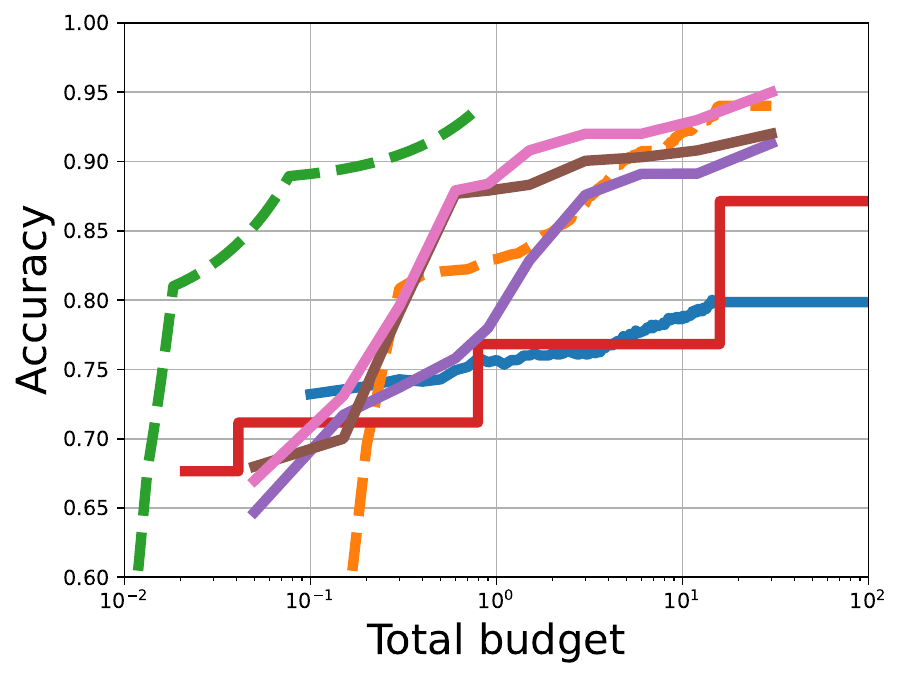}};
  \node at (0,-1.5) [scale=0.5] {CSQA, $\beta =$1M};
\end{tikzpicture}
\end{subfigure}
\begin{subfigure}[b]{0.26\textwidth}
    \centering
	\begin{tikzpicture}
		\node at (0,0) [scale=0.24]{\includegraphics{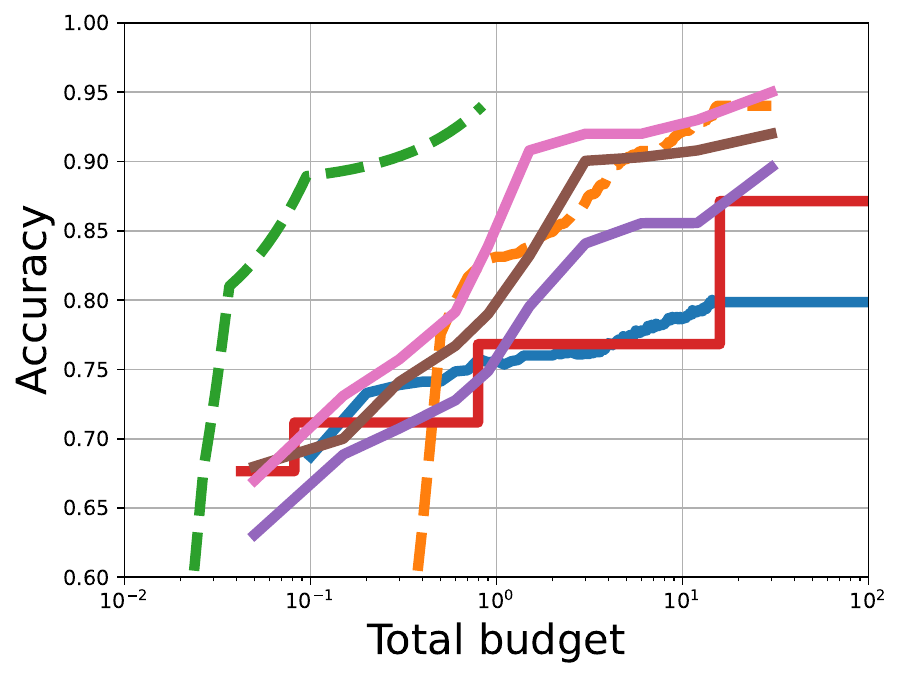}};
  \node at (0,-1.5) [scale=0.5] {CSQA, $\beta =$500k};
\end{tikzpicture}\label{fig:ece}
\end{subfigure}
\begin{subfigure}[b]{0.26\textwidth}
    \centering
	\begin{tikzpicture}
		\node at (0,0) [scale=0.24]{\includegraphics{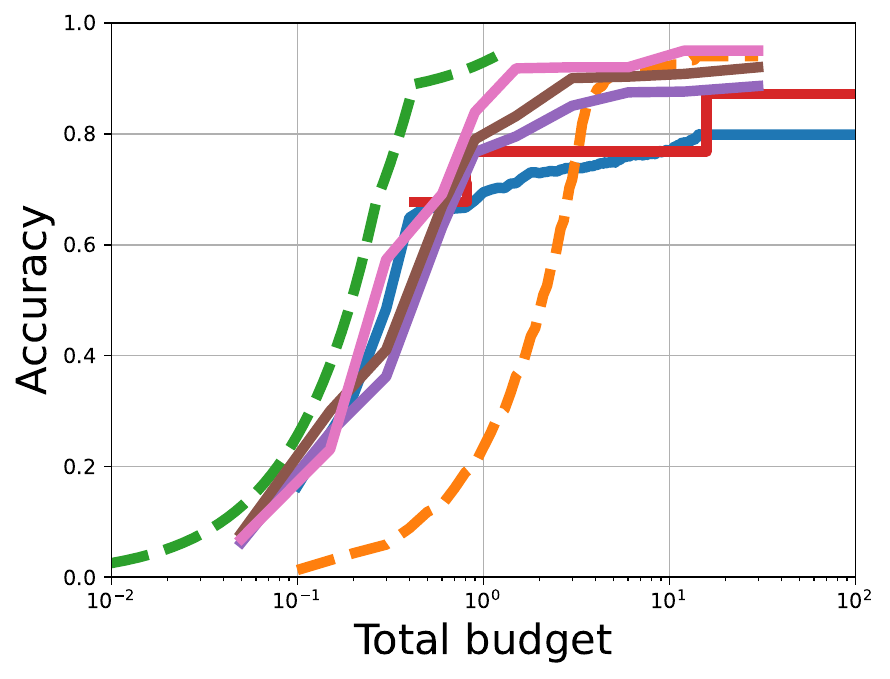}};
  \node at (0,-1.5) [scale=0.5] {CSQA, $\beta =$50k};
\end{tikzpicture}
\end{subfigure}\\
\begin{subfigure}[b]{0.26\textwidth}
    \centering
	\begin{tikzpicture}
		\node at (0,0) [scale=0.24]{\includegraphics{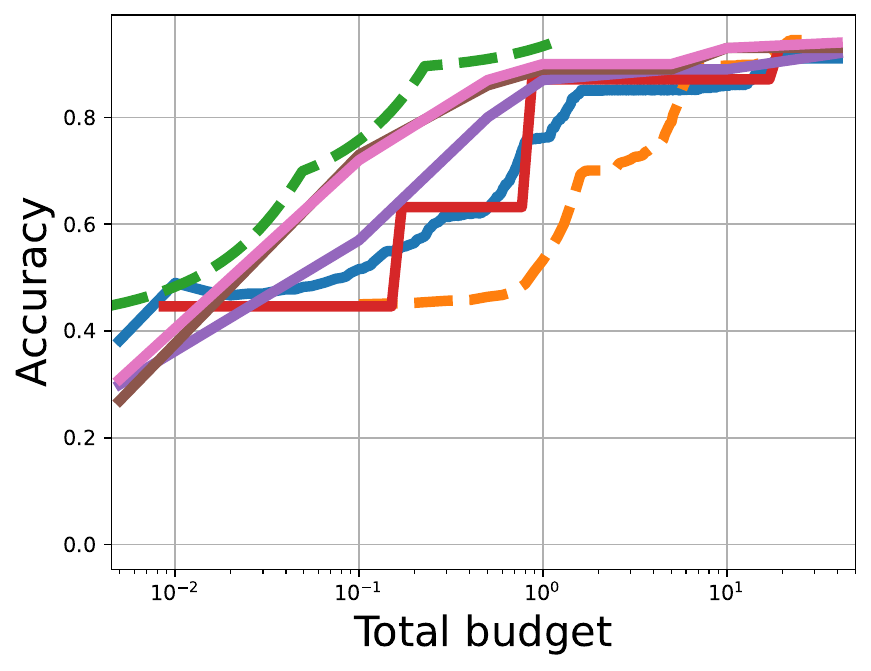}};
  \node at (0,-1.5) [scale=0.5] {LLC, $\alpha = \frac{1}{10}$};
\end{tikzpicture}
\end{subfigure}
\begin{subfigure}[b]{0.26\textwidth}
    \centering
	\begin{tikzpicture}
		\node at (0,0) [scale=0.24]{\includegraphics{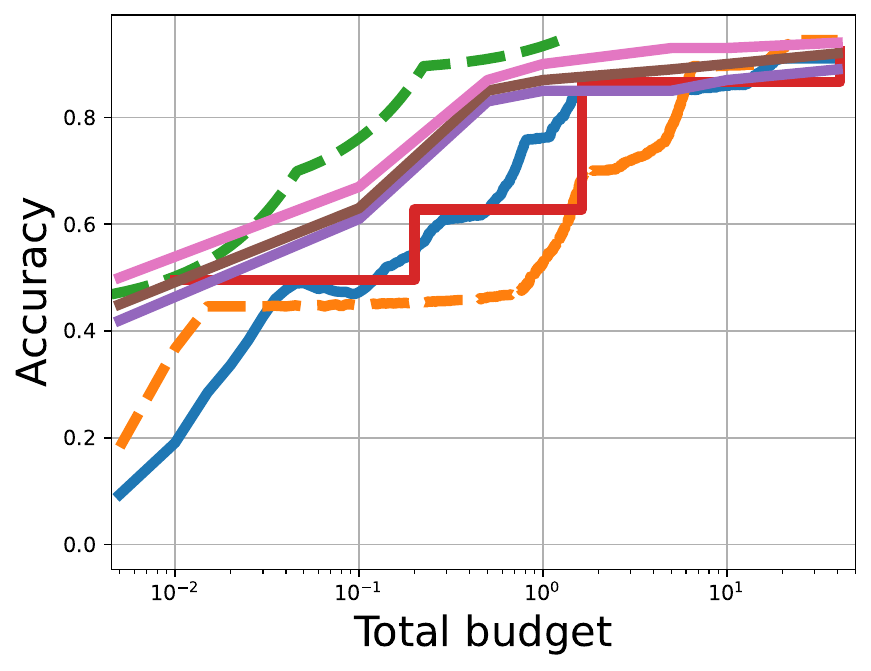}};
  \node at (0,-1.5) [scale=0.5] {LLC, $\alpha = \frac{1}{20}$};
\end{tikzpicture}
\end{subfigure}
\begin{subfigure}[b]{0.26\textwidth}
    \centering
	\begin{tikzpicture}
		\node at (0,0) [scale=0.24]{\includegraphics{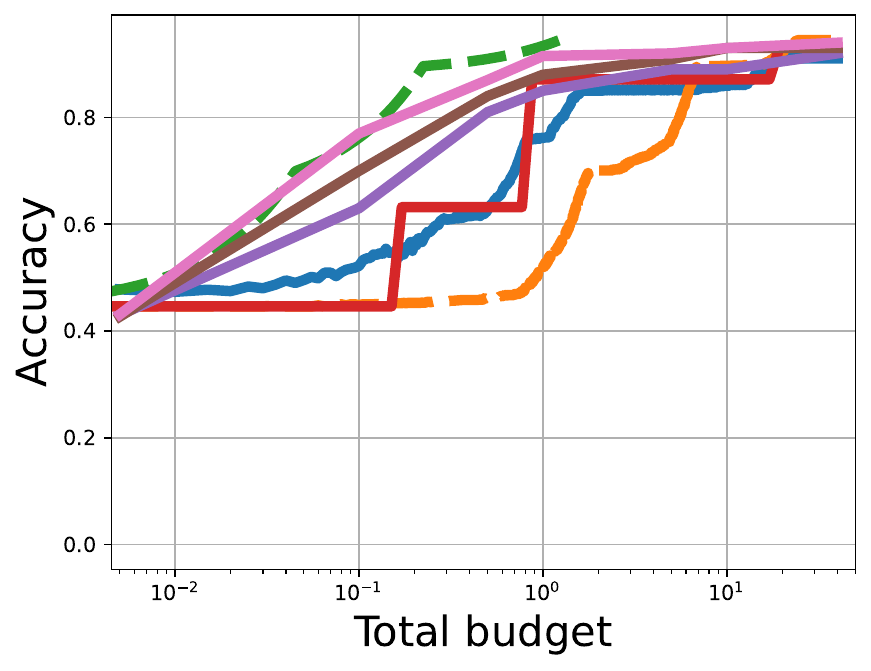}};
  \node at (0,-1.5) [scale=0.5] {LLC, $\alpha = \frac{1}{50}$};
\end{tikzpicture}
\end{subfigure}\\
% \begin{subfigure}[b]{0.24\textwidth}
%     \centering
% 	\begin{tikzpicture}
% 		\node at (0,0) [scale=0.24]{\includegraphics{img/acc_50_csqa.pdf}};
%   \node at (0,-1.5) [scale=0.5] {CSQA, $\alpha = \frac{1}{50}$};
% \end{tikzpicture}
% \end{subfigure}
\begin{subfigure}[b]{0.24\textwidth}
    \centering
	\begin{tikzpicture}
		\node at (0,0) [scale=0.24]{\includegraphics{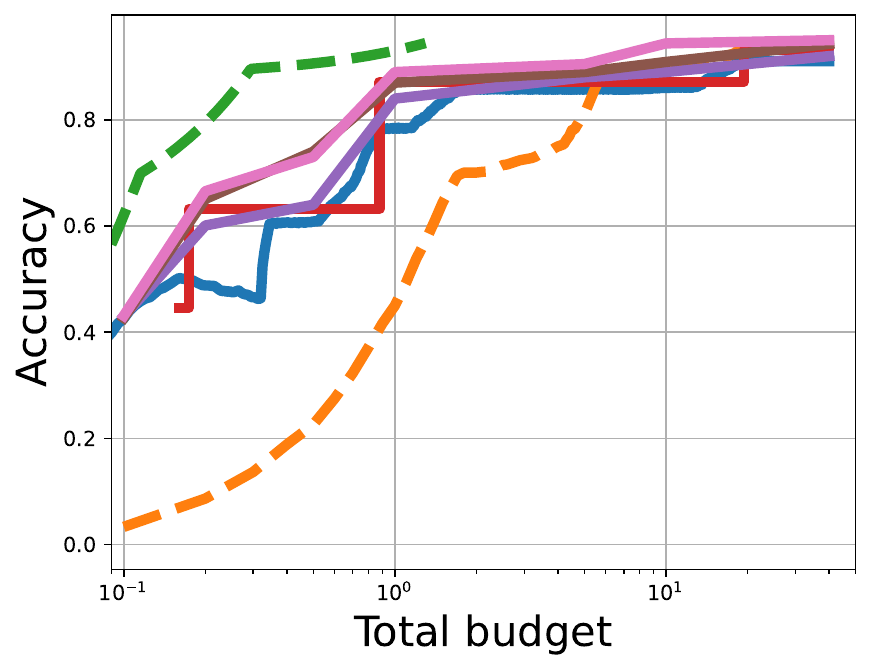}};
  \node at (0,-1.5) [scale=0.5] {LLC, $\beta =$150k};
\end{tikzpicture}
\end{subfigure}
\begin{subfigure}[b]{0.24\textwidth}
    \centering
	\begin{tikzpicture}
		\node at (0,0) [scale=0.24]{\includegraphics{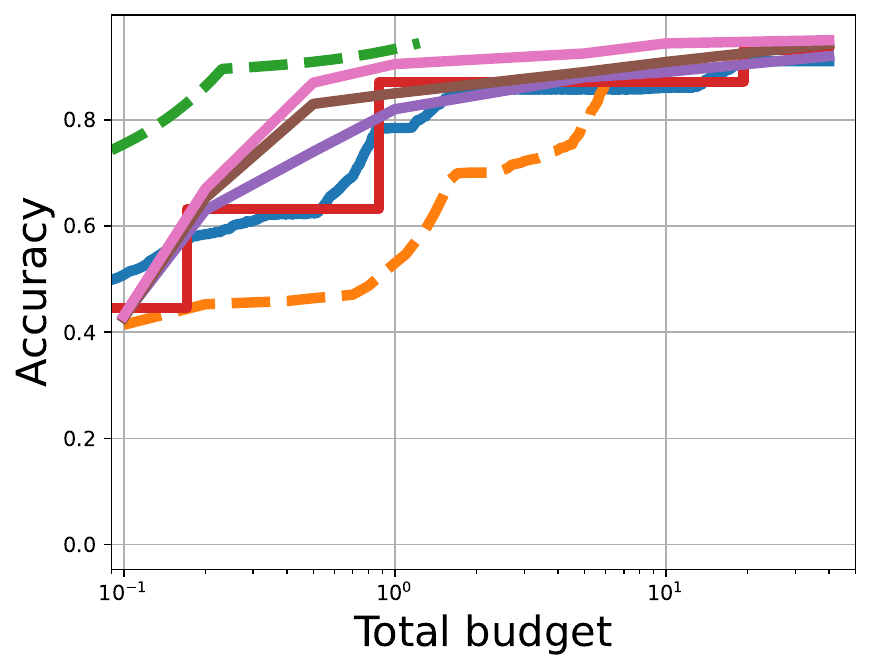}};
  \node at (0,-1.5) [scale=0.5] {LLC, $\beta =$1.5M};
\end{tikzpicture}\label{fig:ece}
\end{subfigure}
\begin{subfigure}[b]{0.24\textwidth}
    \centering
	\begin{tikzpicture}
		\node at (0,0) [scale=0.24]{\includegraphics{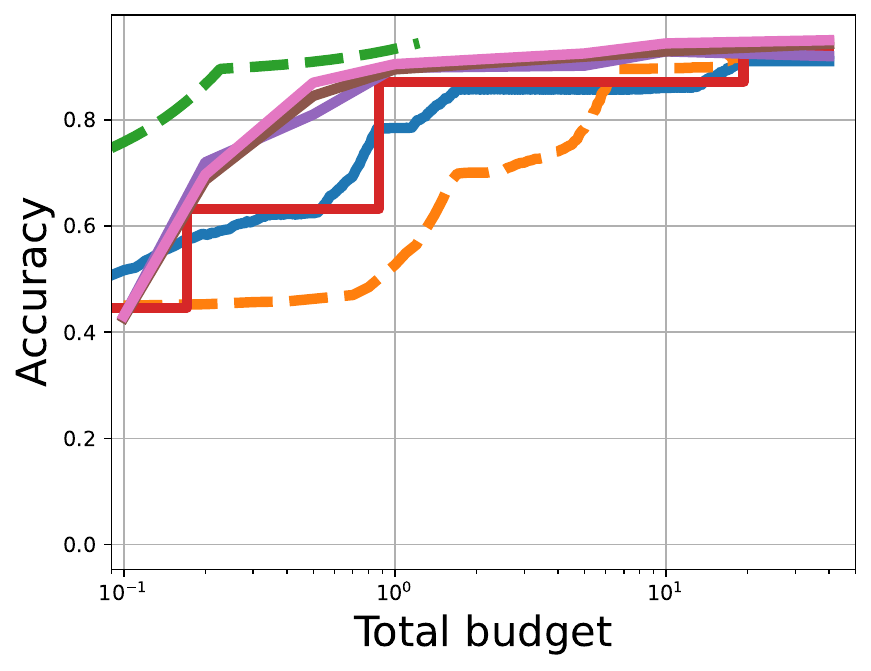}};
  \node at (0,-1.5) [scale=0.5] {LLC, $\beta =$3M};
\end{tikzpicture}
\end{subfigure}\\
\begin{subfigure}[b]{1\textwidth}
    \centering
	\begin{tikzpicture}
		\node at (0,0) [scale=0.4]{\includegraphics{legend.pdf}};
  % \node at (0,-2) [scale=0.8] {10 times};
\end{tikzpicture}
\end{subfigure}
\vspace{-10pt}
\caption{The performance of various methods for different cost functions and budget constraints. The dashed lines are methods that have ground knowledge, which is impractical but illustrates the best achievable performance.
}
\vspace{-10pt}
\label{fig:all_12_app}
\end{figure*}

\begin{table*}[ht]
    \centering
    \tiny
    \caption{Overview of average cost (\$) per query for different models and prompting strategies' combinations in different pricing strategies.}
    \label{tab:single model cost}
    \begin{subtable}[h]{0.95\textwidth}
    \centering
    \caption{GSM8K dataset}
    \resizebox{\linewidth}{!}{
    \begin{tabular}{c|c|c|c|c|c|c}
    \hline
        \multirow{2}{*}{Pricing Strategy} & Llama-2-7b & Llama-2-13b & GPT-3.5-turbo (old) & GPT-3.5-turbo (old) & GPT-4 & GPT-4\\
         & CoT & CoT & Domain Expert & CoT & Domain Expert & CoT \\\hline
        Pure monetary, $\alpha=10$ & $1.61\times10^{-5}$ & $1.67\times10^{-4}$ & $3.66\times10^{-4}$ & $1.38\times10^{-3}$ & $7.57\times10^{-3}$ & $2.94\times10^{-2}$\\
        Pure monetary, $\alpha=20$ & $4.01\times10^{-6}$ & $8.36\times10^{-5}$ & $3.66\times10^{-4}$ & $1.38\times10^{-3}$ & $7.57\times10^{-3}$ & $2.94\times10^{-2}$\\
        Pure monetary, $\alpha=50$ & $6.42\times10^{-7}$ & $3.35\times10^{-5}$ & $3.66\times10^{-4}$ & $1.38\times10^{-3}$ & $7.57\times10^{-3}$ & $2.94\times10^{-2}$\\\hline
        Price-latency combo, $\beta=50K$ & $1.60\times10^{-5}$ & $3.20\times10^{-4}$ & $4.86\times10^{-4}$ & $1.50\times10^{-3}$ & $7.69\times10^{-3}$ & $2.95\times10^{-2}$\\
        Price-latency combo, $\beta=500K$ & $1.60\times10^{-5}$ & $3.20\times10^{-5}$ & $3.78\times10^{-4}$ & $1.39\times10^{-3}$ & $7.58\times10^{-3}$ & $2.94\times10^{-2}$\\
        Price-latency combo, $\beta=1000K$ & $8.00\times10^{-6}$ & $1.60\times10^{-5}$ & $3.72\times10^{-4}$ & $1.38\times10^{-3}$ & $7.58\times10^{-3}$ & $2.94\times10^{-2}$ \\\hline\hline
        \multirow{2}{*}{Pricing Strategy} & MetaMath-7b & MetaMath-13b & GPT-3.5-turbo (new) & GPT-3.5-turbo (new) & GPT-4-turbo & GPT-4-turbo\\
         & Domain Expert & Domain Expert & Domain Expert & CoT & Domain Expert & CoT \\\hline
         Pure monetary, $\alpha=10$ & $4.97\times10^{-6}$ & $5.01\times10^{-5}$ & $3.20\times10^{-4}$ & $9.90\times10^{-4}$ & $5.91\times10^{-3}$ & $1.22\times10^{-2}$\\
         Pure monetary, $\alpha=20$ & $1.24\times10^{-6}$ & $2.50\times10^{-5}$ & $3.20\times10^{-4}$ & $9.90\times10^{-4}$ & $5.91\times10^{-3}$ & $1.22\times10^{-2}$\\
         Pure monetary, $\alpha=50$ & $1.99\times10^{-7}$ & $1.00\times10^{-5}$ &$3.20\times10^{-4}$ & $9.90\times10^{-4}$ & $5.91\times10^{-3}$ & $1.22\times10^{-2}$\\\hline
         Price-latency combo, $\beta=50K$ & $1.60\times10^{-4}$ & $3.20\times10^{-4}$ & $4.40\times10^{-4}$ & $1.11\times10^{-3}$ & $6.03\times10^{-3}$ & $1.23\times10^{-2}$\\
         Price-latency combo, $\beta=500K$ & $1.60\times10^{-5}$ & $3.20\times10^{-5}$ & $3.32\times10^{-4}$ & $1.00\times10^{-3}$ & $5.92\times10^{-3}$ & $1.22\times10^{-2}$\\
         Price-latency combo, $\beta=1000K$ & $8.00\times10^{-6}$ & $1.60\times10^{-5}$ & $3.26\times10^{-4}$ & $9.96\times10^{-4}$ & $5.92\times10^{-3}$ & $1.22\times10^{-2}$\\\hline
    \end{tabular}
        }
    \label{tab:GSM8K single model cost}
    \end{subtable}
    \begin{subtable}[h]{0.95\textwidth}
    \centering
    % \caption{\zijian{fill in caption}}
    % \resizebox{\linewidth}{!}{
    \vspace{0.05in}
    \caption{CSQA dataset}
    \begin{tabular}{c|c|c|c|c}
    \hline
        \multirow{2}{*}{Pricing Strategy} & Llama-2-7b & Llama-2-13b & GPT-3.5-turbo & GPT-4 \\
         & CoT & CoT & Standard & Standard \\\hline
        Pure monetary, $\alpha=10$ & $1.98\times10^{-5}$ & $1.98\times10^{-4}$ & $6.02\times10^{-4}$ & $1.20\times10^{-2}$ \\
        Pure monetary, $\alpha=20$ & $4.96\times10^{-6}$ & $9.92\times10^{-5}$ & $6.02\times10^{-4}$ & $1.20\times10^{-2}$\\
        Pure monetary, $\alpha=50$ & $7.94\times10^{-7}$ & $3.97\times10^{-5}$ & $6.02\times10^{-4}$ & $1.20\times10^{-2}$\\\hline
        Price-latency combo, $\beta=50K$ & $3.20\times10^{-4}$ & $6.20\times10^{-4}$ & $6.08\times10^{-4}$ & $1.20\times10^{-2}$\\
        Price-latency combo, $\beta=500K$ & $3.20\times10^{-5}$ & $6.20\times10^{-5}$ & $6.02\times10^{-4}$ & $1.20\times10^{-2}$\\
        Price-latency combo, $\beta=1000K$ & $1.60\times10^{-5}$ & $3.10\times10^{-5}$ & $6.02\times10^{-4}$ & $1.20\times10^{-2}$\\\hline
    \end{tabular}
        % }
    \label{tab:CSQA single model cost}
    \end{subtable}
    \begin{subtable}[h]{0.95\textwidth}
    \centering
    % \caption{\zijian{fill in caption}}
    % \resizebox{\linewidth}{!}{
    \vspace{0.05in}
    \caption{LLC dataset}
    \begin{tabular}{c|c|c|c|c}
    \hline
        \multirow{2}{*}{Pricing Strategy} & Llama-2-7b & GPT-3.5-turbo & GPT-3.5-turbo & GPT-4 \\
         & CoT & plaintext & CoT & CoT \\\hline
        Pure monetary, $\alpha=10$ & $6.60\times10^{-6}$ & $1.20\times10^{-4}$ & $5.83\times10^{-4}$ & $1.29\times10^{-2}$ \\
        Pure monetary, $\alpha=20$ & $1.65\times10^{-6}$ & $1.20\times10^{-4}$ & $5.83\times10^{-4}$ & $1.29\times10^{-2}$ \\
        Pure monetary, $\alpha=50$ & $2.64\times10^{-7}$ & $1.20\times10^{-4}$ & $5.83\times10^{-4}$ & $1.29\times10^{-2}$ \\\hline
        Price-latency combo, $\beta=150K$ & $1.07\times10^{-4}$ & $1.22\times10^{-4}$ & $5.85\times10^{-4}$ & $1.29\times10^{-2}$\\
        Price-latency combo, $\beta=1500K$ & $1.07\times10^{-5}$ & $1.20\times10^{-4}$ & $5.83\times10^{-4}$ & $1.29\times10^{-2}$\\
        Price-latency combo, $\beta=3000K$ & $5.33\times10^{-6}$ & $1.20\times10^{-4}$ & $5.83\times10^{-4}$ & $1.29\times10^{-2}$\\\hline
    \end{tabular}
        % }
    \label{tab:LLC single model cost}
    \end{subtable}
\end{table*}
\begin{table*}[ht]
    \centering
    \tiny
    \caption{Overview of average accuracy of different models and different prompting strategies' combinations. In GSM8K table, simple CoT few-shot and complex CoT few-shot mean CoT few-shot prompts with easy and hard examples.}
    \label{tab:single model accuracy}
    \begin{subtable}[h]{0.95\textwidth}
    \centering
    \caption{GSM8K dataset}
    % \caption{\zijian{fill in caption}}
    \resizebox{\linewidth}{!}{
    \begin{tabular}{c|c|c|c|c|c|c}
    \hline
        \multirow{2}{*}{Model} & \multirow{2}{*}{System Prompt} & \multirow{2}{*}{Content Prompt} & Training Set Accuracy & Testing Set Accuracy & avg input length & avg output length\\
         & & & (\%) & (\%) & (Training/Testing) & (Training/Testing) \\\hline
        Llama-2-7b & ``Follow example" & simple CoT few-shot & 23.36 & 23.65 & 909.81/911.43 & 120.49/119.14\\\hline
        \multirow{4}{*}{Llama-2-13b} & NA & simple CoT few-shot & 35.65 & 33.81 & 827.81/829.43  & 218.42/214.38 \\
                    & domain Eexpert & plain text & 4.47 & 25.70 & 90.15/83.43 & 28.83/130.51 \\
                    & ``Follow example" & simple CoT few-shot & 37.90 & 37.91 & 909.81/911.43 & 128.41/128.29 \\
                    & ``Follow example" & complex CoT few-shot & 42.77 & 44.05 & 2943.81/2945.43 & 328.99/326.11 \\\hline
        % Llama-2-70b & & & & & & \\\hline
        \multirow{3}{*}{GPT-3.5-turbo} & domain expert & plain text & 76.60 & 73.62 & 88.31/90.98 & 125.07/114.58 \\
                      & ``Follow example" & simple CoT few-shot & 82.00 & 79.15 & 772.05/773.70 & 107.23/108.31 \\
                      & ``Follow example" & complex CoT few-shot & 83.30 & 82.94 & 2419.00/2416.00 & 82.00/49.00 \\\hline
        \multirow{2}{*}{GPT-4-turbo} & domain expert & plain text & 88.18 & 88.48 & 87.31/88.98 & 166.97/167.41 \\
              & ``Follow example" & simple CoT few-shot & 92.61 & 92.34 & 770.05/771.70 & 146.83/149.67 \\\hline
        \multirow{2}{*}{GPT-4} & domain expert & plain text & 84.33 & 83.17 & 87.31/88.98 & 78.00/81.73 \\
              & ``Follow example" & simple CoT few-shot & 93.95 & 92.95 & 770.05/771.70 & 101.51/103.62 \\\hline
        MetaMath-7b & domain expert & plain text & 92.48 & 66.19 & 109.15/110.80 & 48.03/54.68 \\\hline
        MetaMath-13b & domain expert & plain text & 92.81 & 70.43 & 109.15/110.80 & 47.26/56.36\\\hline
        \multirow{2}{*}{PaLM} & ``Follow example" & simple CoT few-shot & 63.05 & 62.17 & 860.04/861.70 & 115.15/115.51 \\
             & ``Follow example" & complex CoT few-shot & 66.95 & 64.14 & 1918.04/1919.70 & 110.68/110.69 \\\hline
    \end{tabular}
        }
    \label{tab:GSM8K single model accuracy}
    \end{subtable}
    \begin{subtable}[h]{0.95\textwidth}
    \centering
    \vspace{0.05in}
    \caption{CSQA dataset}
    % \caption{\zijian{fill in caption}}
    \resizebox{\linewidth}{!}{
    \begin{tabular}{c|c|c|c|c|c|c}
    \hline
        \multirow{2}{*}{Model} & \multirow{2}{*}{System Prompting} & \multirow{2}{*}{Content Prompting} & Training Set Accuracy & Testing Set Accuracy & avg input length & avg output length\\
         & & & (\%) & (\%) & (Training/Testing) & (Training/Testing) \\\hline
        \multirow{3}{*}{Llama-2-7b} & NA & plain text & 33.76 & 34.23 & 52.36/52.06 & 86.44/81.03 \\
                   & ``Follow example" & standard few-shot & 58.86 & 63.06 & 446.36/446.06 & 512.00/512.00 \\
                  & ``Follow example" & CoT few-shot & 64.72 & 67.65 & 640.36/640.06 & 512.00/512.00 \\\hline
        \multirow{3}{*}{Llama-2-13b} & NA & plain text & 29.45 & 32.10 & 52.36/52.06 & 287.60/287.34 \\
                                     & ``Follow example" & standard few-shot & 65.29 & 66.83 & 446.36/446.06 & 512.00/511.80 \\
                                     & ``Follow example" & CoT few-shot & 68.19 & 71.17 & 640.30/640.06 & 512.00/512.00 \\\hline
        % \multirow{3}{*}{Llama-2-70b} & NA & plain text & & & & \\
        %                              & ``Follow example" & standard & & & & \\
        %                              & ``Follow example" & CoT & & & & \\\hline
        \multirow{3}{*}{GPT-3.5-turbo} & NA & plain text & 71.37 & 73.96 & 56.74/56.44 & 4.48/4.49 \\
                                     & ``Follow example" & standard few-shot & 74.09 & 76.82 & 396.74/396.44 & 3.50/3.55 \\
                                     & ``Follow example" & CoT few-shot & 68.31 & 68.55 & 575.74/575.44 & 16.51/16.74 \\\hline
        \multirow{3}{*}{GPT-4} & NA & plain text & 79.86 & 83.46 & 56.74/56.44 & 4.71/4.69 \\
                                     & ``Follow example" & standard few-shot & 84.29 & 87.14 & 396.74/396.44 & 2.00/2.00 \\
                                     & ``Follow example" & CoT few-shot & 82.12 & 85.83 & 575.74/575.44 & 5.05/5.20 \\\hline
    \end{tabular}
        }
    \label{tab:CSQA single model accuracy}
    \end{subtable}
    \begin{subtable}[h]{0.95\textwidth}
    \centering
    \vspace{0.05in}
    \caption{LLC dataset}
    % \caption{\zijian{fill in caption}}
    \resizebox{\linewidth}{!}{
    \begin{tabular}{c|c|c|c|c|c|c}
    \hline
        \multirow{2}{*}{Model} & \multirow{2}{*}{System Prompting} & \multirow{2}{*}{Content Prompting} & Training Set Accuracy & Testing Set Accuracy & avg input length & avg output length\\
         & & & (\%) & (\%) & (Training/Testing) & (Training/Testing) \\\hline
        \multirow{3}{*}{Llama-2-7b} & NA & plain text & 0.06 & 0.13 & 26.25/26.17 & 51.74/51.61 \\
                   & ``Follow example" & standard few-shot & 0.94 & 1.4 & 155.25/358.24 & 155.17/352.73 \\
                  & ``Follow example" & CoT few-shot & 44.23 & 44.60 & 344.25/344.17 & 71.77/71.11 \\\hline
        \multirow{3}{*}{Llama-2-13b} & NA & plain text & 9.01 & 9.73 & 26.25/26.17 & 55.01/54.55 \\
                                     & ``Follow example" & standard few-shot & 2.41 & 2.93 & 155.25/155.17 & 239.28/243.78 \\
                                     & ``Follow example" & CoT few-shot & 48.63 & 48.87 & 344.25/491.84 & 71.77/493.51 \\\hline
        % \multirow{3}{*}{Llama-2-70b} & NA & plain text & & & & \\
        %                              & ``Follow example" & standard & & & & \\
        %                              & ``Follow example" & CoT & & & & \\\hline
        \multirow{3}{*}{GPT-3.5-turbo} & NA & plain text & 62.71 & 63.20 & 30.51/30.45 & 36.93/34.08 \\
                                     & ``Follow example" & standard few-shot & 8.16 & 9.47 & 138.51/138.45 & 5.54/5.51 \\
                                     & ``Follow example" & CoT few-shot & 87.13 & 86.53 & 304.51/304.45 & 63.01/62.86 \\\hline
        \multirow{3}{*}{GPT-4} & NA & plain text & 80.54 & 81.73 & 30.51/29.92 & 36.93/30.24 \\
                                     & ``Follow example" & standard few-shot & 23.74 & 24.27 & 138.51/138.45 & 5.72/5.72 \\
                                     & ``Follow example" & CoT few-shot & 92.68 & 93.2 & 304.51/304.45 & 63.00/62.86 \\\hline
    \end{tabular}
        }
    \label{tab:LLC single model accuracy}
    \end{subtable}
\end{table*}

\subsection{Ablation experiments}\label{app:ablation}
We also run ablation experiments showing that prompt selection is useful, compared to using a fixed prompt (\eg CoT). The results are shown in \Cref{fig:ablation}, where \alg outperforms ``\alg (CoT only)'', indicating that the ability to choose the prompt helps.

\begin{figure}[t]
%\vspace{-5mm}
\begin{center}
\includegraphics[scale=0.45]{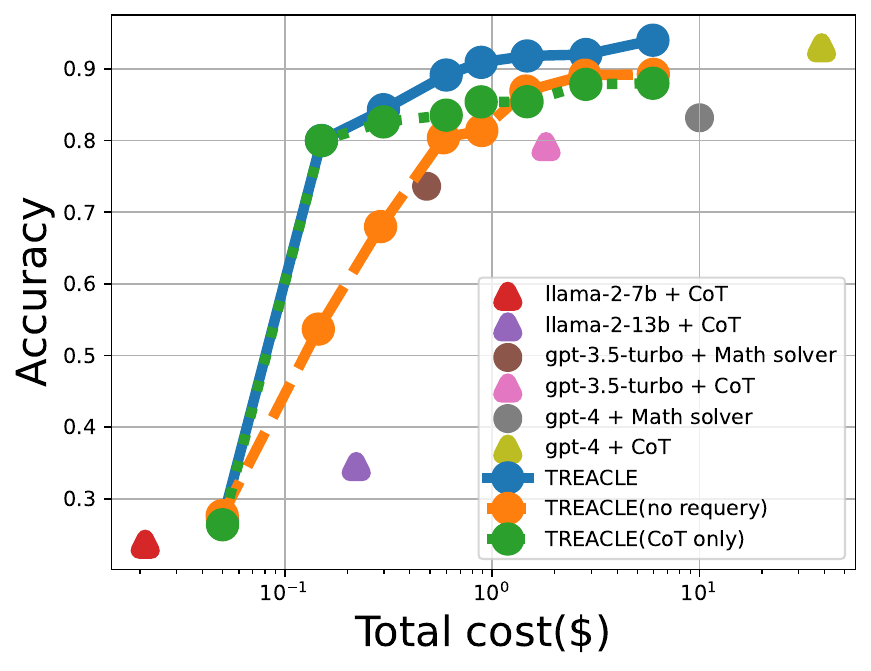}
\end{center}
\vspace{-3mm}
\caption{Ablation study. Without \alg's re-query and prompt selection, the performance decreases dramatically.
%GPT prices are from OpenAI's published API prices; Llama is open-source and thus synthetic prices are used.
}
\label{fig:ablation}
\vspace{-15pt}
\end{figure}
\subsection{Additional new LLM experiments}\label{app:Finetuning}

In this subsection, we report additional results relating to \Cref{sec:new_llm}.
The performance of the fine-tuned models with the API price  adjustments or the improved open-source LLM
%prices or Llama models
is shown in \Cref{fig:finetune_app} (``Finetuned:GPT'' and ``Finetuned:Llama'', respectively).
The results show that the fine-tuned model with both improvements (``Finetuned:all'', same as \Cref{fig:clean-full-new}) performs the best.
The sample efficiency results for fine-tuning these models  with both types of changes (corresponding to ``Finetuned:all'') are shown in \Cref{fig:finetune_app_sample}.

\begin{figure}[t]
%\vspace{-5mm}
\begin{center}
\includegraphics[width=0.40\textwidth]{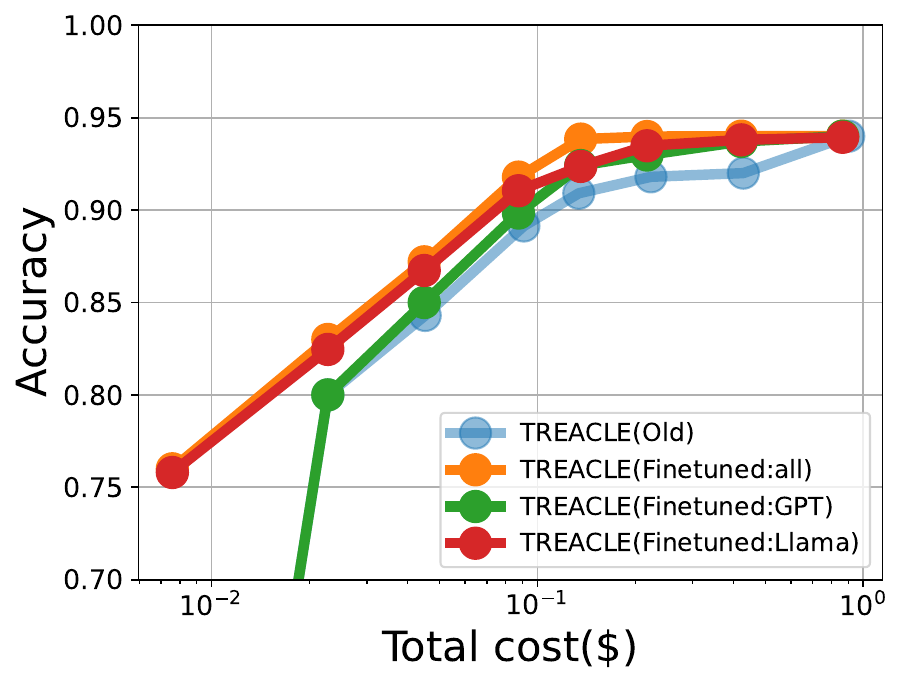}
\includegraphics[width=0.55\textwidth]{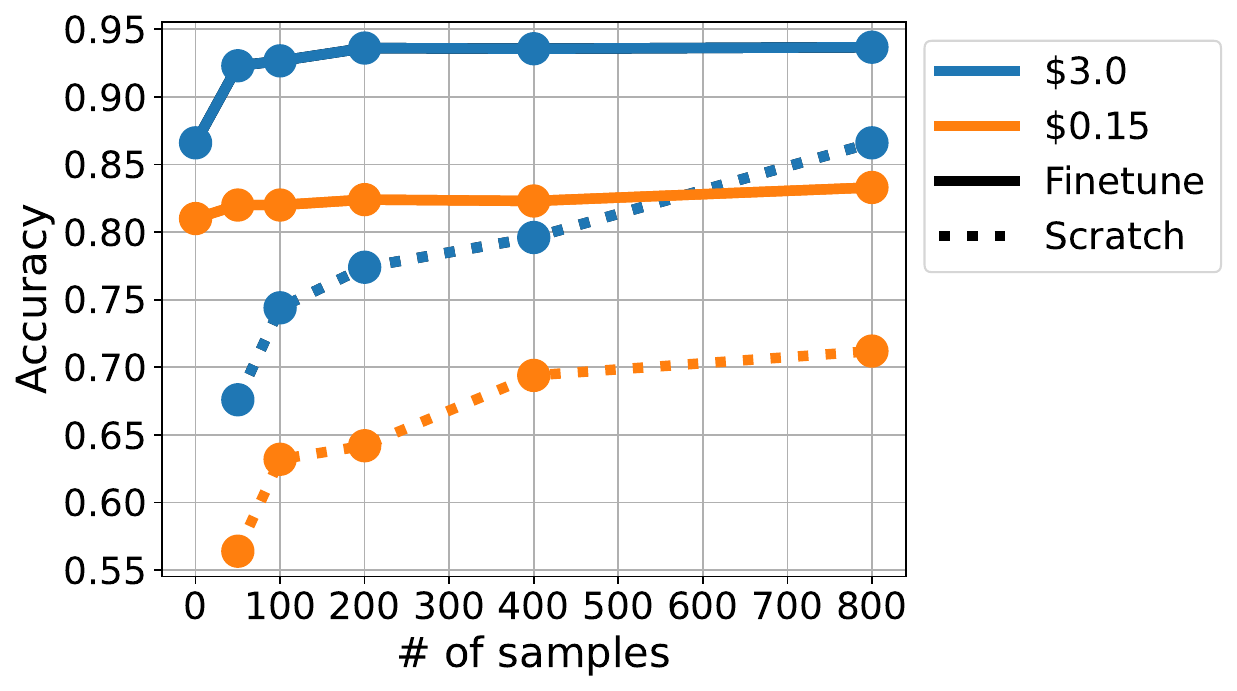}
\end{center}
\vspace{-3mm}
\caption{Additional new LLM results. Left: Zoomed in view of the accuracy with new GPT models, new Llama models, or both. Right: Sample efficiency with both new GPT models and new LLama models.
}
\label{fig:finetune_app}
\label{fig:finetune_app_sample}
%\vspace{-15pt}
\end{figure}
% \begin{figure}[t]
% %\vspace{-5mm}
% \begin{center}
% \includegraphics[scale=0.45]{finetune_sample_3.pdf}
% \end{center}
% \vspace{-3mm}
% \caption{Sample efficiency for modifying all models.
% \xuechen{subfigure with previous}
% }
% \label{fig:finetune_app_sample}
% \vspace{-15pt}
% \end{figure}

\section{API prices}
\label{app:api_prices}

\Cref{tab:single model cost} shows further details on the parameters used in the cost functions described in \Cref{sec:problem}.
\Cref{tab:API price} shows the change of the GPT API's monetary price in different API versions, relating to \Cref{sec:new_llm}.

\begin{table*}[ht]
    \centering
    % \scriptsize
    \caption{API price GPT API price of different versions, where NA means no corresponding model at that version. The price unit is \$/1K tokens. In our experiment setting, old version refers to the version of 0613 and the new version refers to the version 1106.}
    \label{tab:API price}
    \begin{subtable}[h]{0.95\textwidth}
    \centering
    % \caption{\zijian{fill in caption}}
    % \resizebox{\linewidth}{!}{
    \begin{tabular}{c|cc|cc|cc}
    \hline
        \multirow{2}{*}{} & \multicolumn{2}{c|}{0613} & \multicolumn{2}{c|}{1106} & \multicolumn{2}{c}{0125} \\
         & input & output & input & output & input & output \\\hline\hline
        GPT-3.5-turbo & 0.0015 & 0.002 & 0.001 & 0.002 & 0.0005 & 0.002 \\\hline
        GPT-4-turbo & \multicolumn{2}{c|}{NA} & 0.01 & 0.03 & 0.01 & 0.03 \\\hline
        GPT-4 & 0.06 & 0.12 & 0.06 & 0.12 & 0.06 & 0.12 \\\hline
    \end{tabular}
        % }
    
    \end{subtable}
\end{table*}
%\newpage
%\newpage
\section{Full Prompts}
\label{app:full_prompts}
In this section, we show our full prompts.
\begin{table*}[ht]
\small
    \centering
    % \scriptsize
    \caption{Domain expert prompting strategy (``Math solver" and ``Math assistant") in GSM8K dataset, where \{Question\} means that original question text.}
    \label{tab:GSM8K Math Solver prompt}
    \begin{subtable}[h]{0.95\textwidth}
    \centering
    % \caption{\zijian{fill in caption}}
    % \resizebox{\linewidth}{!}{
    \begin{tabular}{p{9cm}|p{3cm}}
    \hline
        \textbf{System Prompt} & \textbf{User Content Prompt} \\\hline\hline
        You are a math solver. Give the answer to the following question. & \{Question\}\\\hline
        \#\#\# Instruction: & \\
        You are a math assistant. Solve the following problem. & \\
        \#\#\# Problem: & \{Question\}\\
        \{User Content Prompt\} & \\
        \#\#\# Answer: & \\
        Let's think step by step.\\\hline
    \end{tabular}
        % }
    
    \end{subtable}
\end{table*}
\begin{table*}[ht]
    \centering
    \scriptsize
    \caption{Chain-of-Thought (CoT) few-shot prompting strategy in GSM8K dataset, where \{Question\} means that original question text.}
    \label{tab:GSM8K CoT prompt}
    % \begin{subtable}\centering
    % % \caption{\zijian{fill in caption}}
    % % \resizebox{\linewidth}{!}{
    % \begin{tabular}{c|c}
    % \hline
    %     System Prompt & User Content Prompt \\\hline\hline
    %     You are a math solver. Give the answer to the following question. & \{Question \}\\\hline
    % \end{tabular}
    %     % }
    % \label{tab:GSM8K Math Solver prompt}
    % \end{subtable}
    \begin{subtable}[h]{0.95\textwidth}
    \centering
    % \caption{\zijian{fill in caption}}
    % \resizebox{\linewidth}{!}{
    \begin{tabular}{p{3cm}|p{9.5cm}}
    \hline
        \textbf{System Prompt} & \textbf{User Content Prompt} \\\hline\hline
        Follow the given examples and answer the question. & Question: There are 15 trees in the grove. Grove workers will plant trees in the grove today. After they are done, there will be 21 trees. How many trees did the grove workers plant today? \\
        & Let's think step by step\\
        & There are 15 trees originally.\\
        & Then there were 21 trees after some more were planted.\\
        & So there must have been 21 - 15 = 6.\\
        & The answer is 6.\\
        & \\
        & Question: If there are 3 cars in the parking lot and 2 more cars arrive, how many cars are in the parking lot?\\
        & Let's think step by step\\
        & There are originally 3 cars.\\
        & 2 more cars arrive.\\
        & 3 + 2 = 5.\\
        & The answer is 5.\\
        & \\
        & Question: Leah had 32 chocolates and her sister had 42. If they ate 35, how many pieces do they have left in total?\\
        & Let's think step by step\\
        & Originally, Leah had 32 chocolates.\\
        & Her sister had 42.\\
        & So in total they had 32 + 42 = 74.\\
        & After eating 35, they had 74 - 35 = 39.\\
        & The answer is 39.\\
        & \\
        & Question: Jason had 20 lollipops. He gave Denny some lollipops. Now Jason has 12 lollipops. How many lollipops did Jason give to Denny?\\
        & Let's think step by step\\
        & Jason started with 20 lollipops.\\
        & Then he had 12 after giving some to Denny.\\
        & So he gave Denny 20 - 12 = 8.\\
        & The answer is 8.\\
        & \\
        & Question: Shawn has five toys. For Christmas, he got two toys each from his mom and dad. How many toys does he have now?\\
        & Let's think step by step\\
        & Shawn started with 5 toys.\\
        & If he got 2 toys each from his mom and dad, then that is 4 more toys.\\
        & 5 + 4 = 9.\\
        & The answer is 9.\\
        & \\
        & Question: There were nine computers in the server room. Five more computers were installed each day, from monday to thursday. How many computers are now in the server room?\\
        & Let's think step by step\\
        & There were originally 9 computers.\\
        & For each of 4 days, 5 more computers were added.\\
        & So 5 * 4 = 20 computers were added.\\
        & 9 + 20 is 29.\\
        & The answer is 29.\\
        & \\
        & Question: Michael had 58 golf balls. On tuesday, he lost 23 golf balls. On wednesday, he lost 2 more. How many golf balls did he have at the end of wednesday?\\
        & Let's think step by step\\
        & Michael started with 58 golf balls.\\
        & After losing 23 on tues- day, he had 58 - 23 = 35.\\
        & After losing 2 more, he had 35 - 2 = 33 golf balls.\\
        & The answer is 33.\\
        & \\
        & Question: Olivia has \$23. She bought five bagels for \$3 each. How much money does she have left?\\
        & Let's think step by step\\
        & Olivia had 23 dollars.\\
        & 5 bagels for 3 dollars each will be 5 x 3 = 15 dollars.\\
        & So she has 23 - 15 dollars left.\\
        & 23 - 15 is 8.\\
        & The answer is 8.\\
        & \\
        & Question: \{Question\}\\
        & \\
        \hline
    \end{tabular}
        % }
    
    \end{subtable}
\end{table*}
\begin{table*}[ht]
    \centering
    \scriptsize
    \caption{Standard few-shot prompting strategy in CSQA dataset, where \{Question\} means that original question text.}
    \label{tab:CSQA Standard prompt}
    \begin{subtable}[h]{0.95\textwidth}
    \centering
    % \caption{\zijian{fill in caption}}
    % \resizebox{\linewidth}{!}{
    \begin{tabular}{p{13cm}}
    \hline
        \textbf{User Content Prompt} \\\hline\hline
        Q: What do people use to absorb extra ink from a fountain pen? Answer Choices: (A) shirt pocket (B) calligrapher's hand (C) inkwell (D) desk drawer (E) blotter\\
        A: The answer is E.\\
        \\
        Q: What home entertainment equipment requires cable? Answer Choices: (A) radio shack (B) substation (C) television (D) cabinet\\
        A: The answer is C.\\
        \\
        Q: The fox walked from the city into the forest, what was it looking for? Answer Choices: (A) pretty flowers (B) hen house (C) natural habitat (D) storybook\\
        A: The answer is B.\\
        \\
        Q: Sammy wanted to go to where the people were. Where might he go? Answer Choices: (A) populated areas (B) race track (C) desert (D) apartment (E) roadblock\\
        A: The answer is A.\\
        \\
        Q: Where do you put your grapes just before checking out? Answer Choices: (A) mouth (B) grocery cart (C)supermarket (D) fruit basket (E) fruit market\\
        A: The answer is B.\\
        \\
        Q: Google Maps and other highway and street GPS services have replaced what? Answer Choices: (A) united states (B) mexico (C) countryside (D) atlas\\
        A: The answer is D.\\
        \\
        Q: Before getting a divorce, what did the wife feel who was doing all the work? Answer Choices: (A) harder (B) anguish (C) bitterness (D) tears (E) sadness\\
        A: The answer is C.\\
        \\
        Q: \{Question \}\\
        A: The answer is\\
        \hline
    \end{tabular}
        % }
    
    \end{subtable}
\end{table*}
\begin{table*}[ht]
    \centering
    \scriptsize
    \caption{Chain-of-Thought (CoT) few-shot prompting strategy in CSQA dataset, where \{Question\} means that original question text.}
    \label{tab:CSQA CoT prompt}
    \begin{subtable}[h]{0.95\textwidth}
    \centering
    % \caption{\zijian{fill in caption}}
    % \resizebox{\linewidth}{!}{
    \begin{tabular}{p{13cm}}
    \hline
        \textbf{User Content Prompt} \\\hline\hline
        Q: What do people use to absorb extra ink from a fountain pen? Answer Choices: (A) shirt pocket (B) calligrapher's hand (C) inkwell (D) desk drawer (E) blotter\\
        A: The answer must be an item that can absorb ink. Of the above choices, only blotters are used to absorb ink. The answer is E.\\
        \\
        Q: What home entertainment equipment requires cable? Answer Choices: (A) radio shack (B) substation (C) television (D) cabinet\\
        A: The answer must require cable. Of the above choices, only television requires cable. The answer is C.\\
        \\
        Q: The fox walked from the city into the forest, what was it looking for? Answer Choices: (A) pretty flowers (B) hen house (C) natural habitat (D) storybook\\
        A: The answer must be something in the forest. Of the above choices, only natural habitat is in the forest. The answer is B.\\
        \\
        Q: Sammy wanted to go to where the people were. Where might he go? Answer Choices: (A) populated areas (B) race track (C) desert (D) apartment (E) roadblock\\
        A: The answer must be a place with a lot of people. Of the above choices, only populated areas have a lot of people. The answer is A.\\
        \\
        Q: Where do you put your grapes just before checking out? Answer Choices: (A) mouth (B) grocery cart (C)supermarket (D) fruit basket (E) fruit market\\
        A: The answer should be the place where grocery items are placed before checking out. Of the above choices, grocery cart makes the most sense for holding grocery items. The answer is B.\\
        \\
        Q: Google Maps and other highway and street GPS services have replaced what? Answer Choices: (A) united states (B) mexico (C) countryside (D) atlas\\
        A: The answer must be something that used to do what Google Maps and GPS services do, which is to give directions. Of the above choices, only atlases are used to give directions. The answer is D.\\
        \\
        Q: Before getting a divorce, what did the wife feel who was doing all the work? Answer Choices: (A) harder (B) anguish (C) bitterness (D) tears (E) sadness\\
        A: The answer should be the feeling of someone getting divorced who was doing all the work. Of the above choices, the closest feeling is bitterness. The answer is C.\\
        \\
        Q: \{Question \}\\
        A: \\
        \hline
    \end{tabular}
        % }
    
    \end{subtable}
\end{table*}

\begin{table*}[ht]
    \centering
    \scriptsize
    \caption{Standard few-shot prompting strategy in LLC dataset, where \{Question\} means that original question text.}
    \label{tab:LLC standard prompt}
    \begin{subtable}[h]{0.95\textwidth}
    \centering
    % \caption{\zijian{fill in caption}}
    % \resizebox{\linewidth}{!}{
    \begin{tabular}{p{13cm}}
    \hline
        \textbf{User Content Prompt} \\\hline\hline
        Question: Take the last letters of the words in ``Elon Musk" and concatenate them.\\
        The answer is nk.\\
        \\
        Question: Take the last letters of the words in ``Larry Page" and concatenate them.\\
        The answer is ye.\\
        \\
        Question: Take the last letters of the words in ``Sergey Brin" and concatenate them.\\
        The answer is yn.\\
        \\
        Question: Take the last letters of the words in "Bill Gates" and concatenate them.\\
        The answer is ls.\\
        \\
        Question: \{Question\}\\
        \hline
    \end{tabular}
        % }
    
    \end{subtable}
\end{table*}

\begin{table*}[ht]
    \centering
    \scriptsize
    \caption{Chain-of-Thought (CoT) few-shot prompting strategy in LLC dataset, where \{Question\} means that original question text.}
    \label{tab:LLC CoT prompt}
    \begin{subtable}[h]{0.95\textwidth}
    \centering
    % \caption{\zijian{fill in caption}}
    % \resizebox{\linewidth}{!}{
    \begin{tabular}{p{13cm}}
    \hline
        \textbf{User Content Prompt} \\\hline\hline
        Question: Take the last letters of the words in ``Elon Musk" and concatenate them.\\
        Let's think step by step.\\
        The last letter of ``Elon" is ``n". \\
        The last letter of ``Musk" is ``k".\\
        Concatenating them is ``nk".\\
        The answer is nk.\\
        \\
        Question: Take the last letters of the words in ``Larry Page" and concatenate them.\\
        Let's think step by step.\\
        The last letter of "Larry" is ``y".\\
        The last letter of "Page" is ``e".\\
        Concatenating them is ``ye".\\
        The answer is ye.\\
        \\
        Question: Take the last letters of the words in ``Sergey Brin" and concatenate them.\\
        Let's think step by step.\\
        The last letter of ``Sergey" is ``y".\\
        The last letter of "Brin" is ``n".\\
        Concatenating them is ``yn".\\
        The answer is yn.\\
        \\
        Question: Take the last letters of the words in "Bill Gates" and concatenate them.\\
        Let's think step by step.\\
        The last letter of ``Bill" is ``l".\\
        The last letter of ``Gates" is ``s".\\
        Concatenating them is ``ls".\\
        The answer is ls.\\
        \\
        Question: \{Question\}\\
        Let's think step by step.\\
        \hline
    \end{tabular}
        % }
    
    \end{subtable}
\end{table*}

% \section{Uncertainty estimation:}
% Straightforwardly, performance estimation can rely on both the question and the answer of the current query $O_i$, $i$ is the number of queries, represented as $f(Q, O_i)$. We investigate that the performance can not be self-evaluated or predicted with text embedding.     

% \textbf{LLM self-evaluation}: Firstly, we investigate whether LLMs can self-evaluate. After the model provides an answer, we inquire about its confidence level regarding the previous response, using the prompt "{Answer} is your previous answer to question {Query}, please rate your confidence in this answer on a scale of 0 to 10, where 0 signifies no confidence at all, and 10 signifies full confidence." We observe that LLMs are consistently over-confident. Additionally, this approach is quite costly, as it requires inputting all answers and queries, which entails a substantial number of tokens.

% \textbf{Text embedding}: We use text embeddings to predict the correctness of answers. We used text-embedding-ada-002 model provided by OpenAI \cite{openai_embedding_model}, which is the most powerful commercial embedding model. Then, we train a two-layer calibrated neural network to use these embeddings to predict answer correctness. Although it still couldn't accurately predict answer correctness, we include it as a baseline. 
%%%%%%%%%%%%%%%%%%%%%%%%%%%%%%%%%%%%%%%%%%%%%%%%%%%%%%%%%%%%%%%%%%%%%%%%%%%%%%%
%%%%%%%%%%%%%%%%%%%%%%%%%%%%%%%%%%%%%%%%%%%%%%%%%%%%%%%%%%%%%%%%%%%%%%%%%%%%%%%

\newpage

\end{document}